\documentclass[letterpaper, 10 pt, conference]{ieeeconf}  % Comment this line out if you need a4paper

\IEEEoverridecommandlockouts                              % This command is only needed if 
\usepackage{graphics} % for pdf, bitmapped graphics files
\usepackage{epsfig} % for postscript graphics files
\usepackage{times} % assumes new font selection scheme installed
\usepackage{amsmath} % assumes amsmath package installed
\usepackage{amssymb}  % assumes amsmath package installed

\usepackage{bm}
\usepackage{graphicx}
\usepackage{color}
\usepackage{import}
\usepackage[font=footnotesize]{caption} %footnotesize=8pt under 10pt env
\usepackage{subcaption}
\usepackage{hyperref}
\usepackage{cite}
\usepackage{float}
\usepackage[percent]{overpic} % percentage to put text in figure
\usepackage{algorithm2e}
\usepackage{gensymb} % for the degree symbol
\usepackage[para,online,flushleft]{threeparttable} % for table footnotes
\usepackage{array, makecell}
\usepackage{tabularx} % to adjust table width
\usepackage{xcolor} % to change text color
\usepackage[percent]{overpic} % percentage to overlay text in a figure
\usepackage{moresize} % more font sizes
\usepackage{tikz} % tikz visuals

\newcommand{\probName}{DAF-SLAM}
\newcommand{\algName}{kSLAM}
\newcommand{\algNameS}{skSLAM}
\newcommand{\algNameI}{ikSLAM}
\newcommand{\algNameIS}{iskSLAM}
\newcommand{\algNameDNC}{skid-SLAM}
\newcommand{\algNameLL}{skill-SLAM}
\newcommand{\inlinetitle}{\bf}
\newcommand{\boldx}{\boldsymbol{x}} % shortcut 
\newcommand{\boldy}{\boldsymbol{y}} % shortcut
\newcommand{\boldzbar}{\boldsymbol{\bar{z}}}
\newcommand{\zbar}{\bar{z}}
\newcommand{\SE}{\mathrm{SE}}
\newcommand{\realspace}{\mathbb{R}}
\newcommand{\gauss}{\mathsf{N}}
\newcommand{\tran}{^{\mathsf{T}}}
\newcommand{\fslamstar}[1]{f_{\text{slam}}^\star(#1)}
\newcommand{\fslam}{f_{\text{slam}}}

\newcommand{\zsbar}[1]{\bar{z}_{s,#1}} % for easy notation switch
\newcommand{\ys}[1]{y_{s,#1}} % for easy notation switch
\newcommand{\epss}[1]{\epsilon_{s,#1}}
\newcommand{\boldys}{\boldsymbol{y_s}}

\newtheorem{lemma}{Lemma}

\newcommand{\R}{\mathbb{R}}
\newcommand{\E}{\mathbb{E}}
\newcommand{\Prob}{\mathbb{P}}
\newcommand{\dd}{\mathrm{d}}

\newif\ifarxiv
\arxivtrue      % switch to arxiv version
\newcommand{\alt}[2]{%
  \ifarxiv
    #1%
  \else
    #2%
  \fi
}
\newcommand{\supplementary}{the Appendix} % for arxiv submission

\title{\LARGE \bf
Semantic Semi-Incremental Data-Association-Free Object SLAM
}

\alt{
\author{Yihao Zhang$^{1*}$, Jungseok Hong$^{1*}$, John J. Leonard$^{1}$% <-this % stops a space
\thanks{$^{1}$Yihao Zhang, Jungseok Hong, and John Leonard are with the Computer Science and Artificial Intelligence Laboratory, Massachusetts Institute of Technology ({\tt\small \{yihaozh, jungseok, jleonard\}@mit.edu}). * indicates equal contribution.}%
}}
{
\author{Anonymous Authors
\thanks{*Code will be released upon acceptance to preserve anonymity.}% <-this % stops a space
}}

\begin{document}

\maketitle
\thispagestyle{empty}
\pagestyle{empty}

%%%%%%%%%%%%%%%%%%%%%%%%%%%%%%%%%%%%%%%%%%%%%%%%%%%%%%%%%%%%%%%%%%%%%%%%%%%%%%%%
\begin{abstract}
Data association between landmark measurements and landmark variables has long been a central challenge in SLAM, as estimation accuracy depends critically on associating measurements with the correct landmark variables. Recent advances in deep learning have created new opportunities for the problem; data association can now leverage not only positional measurements but also semantic information about object landmarks, such as class labels from neural object detectors and feature vectors from visual foundation models. In this paper, we present a generalized data-association-free SLAM framework that jointly estimates data associations, robot poses, landmark positions, and landmark semantics from odometry, and positional and semantic measurements of landmarks. The proposed framework (i) creates a synergy between data association and landmark semantics estimation; (ii) adopts a semi-incremental estimation scheme for improved accuracy and computational efficiency; and (iii) provides a principled justification, guidelines, and heuristics for landmark-number estimation, improving the interpretability and practical usability of the framework. The proposed framework and algorithms are evaluated on synthetic and real-world datasets with two types of semantic information, class labels and real-valued feature vectors, and demonstrate superior performance compared to strong baselines.
\end{abstract}

%%%%%%%%%%%%%%%%%%%%%%%%%%%%%%%%%%%%%%%%%%%%%%%%%%%%%%%%%%%%%%%%%%%%%%%%%%%%%%%%
\section{Introduction}
Object SLAM builds a map with objects as landmarks. It is a step beyond traditional geometry-only SLAM toward semantic scene understanding. Objects are typically detected by a neural module, which also provides semantic information about them. Recent work \cite{doherty2020probabilistic, singh2024opti, hong2025semantic, michael2022probabilistic} explores the use of such semantic information to facilitate data association of landmark measurements. Our work advances this line of research. In particular, building on the data-association-free landmark-based SLAM (DAF-SLAM) framework, we develop a semantic data-association-free object SLAM framework that jointly estimates robot poses, landmark positions, landmark semantics, data associations, and the number of landmarks given odometry measurements, landmark position measurements, and landmark semantic measurements.

DAF-SLAM \cite{zhang2023data} studies a minimal problem in which semantic information is not available; therefore, data association has to be inferred solely from odometry and landmark position measurements. The authors recognize that when data association is unknown, the number of landmarks in the environment is also unknown and may be as large as the number of landmark measurements received. They split the problem into an inner problem that estimates the data associations given the number of landmarks and an outer problem that estimates the number of landmarks, assuming the inner problem can be solved. Various algorithms have been developed to solve these sub-problems.

While DAF-SLAM establishes a promising foundation, extending it to an efficient and accurate object SLAM framework introduces several important challenges. First, the DAF-SLAM framework does not exploit semantic information that is often available from modern object detectors and visual foundation models. Second, its batch optimization strategy can suffer from accumulated odometry drift when pose estimates are initialized from odometry over long trajectories. Third, a deeper understanding of the landmark-number estimation mechanism is desirable for formulation interpretation and parameter selection.

In this paper, we extend the DAF-SLAM framework and address the aforementioned challenges through four main contributions. (i) We provide a principled analysis of the landmark-number estimation mechanism based on local optimality conditions. (ii) We develop a semi-incremental (or block-incremental) processing scheme inspired by the classical divide-and-conquer strategy to reduce drift and improve computational efficiency. (iii) We integrate semantic information into the estimation process, enabling joint inference of data associations and landmark semantics. The semantic measurements may take the form of either real-valued feature vectors or one-hot vectors, which are compatible with the outputs of visual foundation models and object detectors. (iv) We evaluate the proposed algorithms extensively on synthetic datasets, a real-world dataset using an object detector to extract semantic measurements, and a real-world underwater dataset using a visual foundation model to extract semantic measurements. The proposed algorithms outperform strong baselines across these diverse datasets. Together, these contributions advance the development of a more general and efficient data-association-free object SLAM framework.

\section{Related Work}
Research on data association in SLAM has a long history. Neira and Tard{\'o}s \cite{Neira2001JCBB} and Montemerlo et al. \cite{montemerlo2002fastslam} provide an early review of data association methods in SLAM, including maximum likelihood, joint compatibility branch and bound, combined constraint data association, iterative closest point, and multi-hypothesis tracking. Traces of these methods can still be seen in modern geometry-only and object SLAM methods. For example, object SLAM pipelines \cite{hong2025semantic, singh2024opti} belong to the maximum likelihood category. They compute the Mahalanobis distance, derived from maximum likelihood estimation, to handle the positional component, and the cosine similarity (or inner product) to handle the semantic feature-vector component of data association. Iterative-style methods such as Max-Mixture \cite{doherty2020probabilistic}, DC-SAM \cite{doherty2022discrete}, and DAF-SLAM \cite{zhang2023data} are, in principle, similar to iterative closest point or coordinate descent. However, how variables are split can affect the algorithm performance. Both Max-Mixture and DC-SAM split the variables into discrete (i.e., data associations) and continuous (i.e., landmark positions and robot poses) variables, and solve them separately in an alternating fashion. Whereas DAF-SLAM considers a smarter split, where one step solves data associations and landmark positions, and another step solves robot poses and again landmark positions. Mature algorithms exist to solve both sub-problems, yielding improved performance. A recent approach \cite{korotkine2025globally} solves the discrete-continuous problem with semi-definite relaxation but has limited scalability.
% Multi-hypothesis methods include \cite{kaltiokallio2023multihypotheses, bernreiter2019multiple}. % optional

Although still running maximum likelihood or maximum a posteriori estimation, a distinct line of object SLAM work \cite{mu2016slam, zhang2021bayesian, ran2021not} models data association as a Dirichlet process so that landmarks with more associated measurements are more likely to be associated again. Another line of object SLAM work \cite{bowman2017probabilistic, michael2022probabilistic} takes an expectation over data associations in an expectation-maximization style to eliminate the discrete data association variables in the inference. 

When building object SLAM systems, researchers often employ sensor-dependent front-end association methods. For example, SLAM++ \cite{salas2013slam++} uses image-based projective data association; room numbers on door signs are used for association in \cite{rogers2011simultaneous}; ODAM \cite{li2021odam} (mapping-only) makes associations with a graph neural network that takes in different image-based attributes and features. These front-end methods can be easily incorporated as a module in a SLAM system. However, mistakes made in these front-end associations are not recoverable in the back end. In addition, a redesign of the data association module is required for these sensor-dependent methods when the sensor is changed. In this paper, we focus on estimating associations in the sensor-agnostic back-end solver. In particular, we extend and enhance the DAF-SLAM framework \cite{zhang2023data} to develop a more general and efficient object SLAM framework.

\section{Background and \probName\ Recap}
The data-association-free SLAM (\probName) problem is named in \cite{zhang2023data} to indicate a setup where only odometry and landmark measurements are available to a SLAM solver, but not data associations of the landmark measurements. Mathematically, we denote the robot trajectory as $\boldx = \{x_i\}_{i=1}^N$, where $x_i \in \SE(d)$ is the unknown $i$-th pose of the robot in dimension $d=2$ or $3$, and $x_i = (R_i, t_i)$ are the rotation and translation components. The unknown landmark positions are denoted by $\boldy = \{y_j\}_{j=1}^K$ where $y_j \in \realspace^d$. Finally, the given landmark measurements are $\boldzbar = \{\zbar_k\}_{k=1}^M$. $\zbar_k \in \realspace^d$ is assumed to be generated from the standard model:
\begin{equation}
    \zbar_k = R_{i_k}\tran(y_{j_k} - t_{i_k}) + \epsilon_k
\label{eq: posgen}
\end{equation}
where a measurement $\zbar_k$ subject to zero-mean Gaussian noise $\epsilon_k \sim \gauss(0, \Sigma)$ is made in the local frame of robot pose $i_k$. This $k$-th measurement can be associated with the corresponding pose $x_{i_k}$ via timestamp synchronization. However, the data association, i.e., the index of the landmark being measured ($j_k$), is unknown. Furthermore, since data associations are unknown, the number of landmarks ($K$) in the environment is also unknown.

It is proposed in \cite{zhang2023data} to split the \probName\ problem into an inner problem and an outer problem. The inner problem takes the form:
\begin{equation}
    \min_{\substack{\boldx \in \SE(d)^N \\ \boldy \in \realspace^{d \times K}}}{
    f_{\text{odom}}(\boldx) + \sum_{k=1}^M{
    \min_{j_k \in [K]}{\|R_{i_k}\tran(y_{j_k} - t_{i_k}) - \zbar_k\|_{\Sigma}^2}}}
\label{eq: inner}
\end{equation}
where $[K] = \{1, \ldots, K\}$. \eqref{eq: inner} is a function of the number of landmarks ($K$). Given $K$, it minimizes the residuals of landmark measurements over data associations ($j_k$), which can be viewed as a max-mixture model \cite{olson2013inference}. \eqref{eq: inner} is solved by alternating minimization between solving $\{j_k, \boldy\}$ through k-means and k-means++ \cite{arthur2006k} and solving $\{\boldx, \boldy\}$ through a SLAM solver \cite{dellaert2019georgia}. This alternating scheme is shown to be advantageous over alternating between $j_k$ alone and $\{\boldx, \boldy\}$. In addition, k-means++ eliminates the need of an initial guess for $\boldy$ so only odometry is needed to initialize $\boldx$.

Let the optimal objective value of \eqref{eq: inner} be $\fslamstar{K}$. The outer problem of estimating the number of landmarks ($K$) is
\begin{equation}
    \min_{K \in \{1,\ldots,M\}}{\fslamstar{K} + \beta K}
\label{eq: outer}
\end{equation}
where $\beta$ is a hyperparameter. \eqref{eq: outer} is solved with multi-resolution gridding, which is a zeroth-order method that uses the inner problem \eqref{eq: inner} solver to evaluate $\fslamstar{K}$. The full algorithm is named \algName\ for its use of k-means.

\section{Method}
\subsection{Local Optimality Condition} \label{sec: optimcond}
One important but missing analysis in \cite{zhang2023data} is the local optimality condition for \eqref{eq: outer}. This optimality condition not only plays a role in finding heuristics for setting $\beta$ but also helps explain the working principle of the outer problem formulation. Suppose the actual number of landmarks is $K^\star$. In order for the minimizer of \eqref{eq: outer} to be $K^\star$, at least the following necessary local conditions have to be met:
\begin{gather}
    \fslamstar{K^\star + 1} + \beta \ge \fslamstar{K^\star} \label{eq: cond_kplus1} \\
    \fslamstar{K^\star - 1} - \beta \ge \fslamstar{K^\star} \label{eq: cond_kminus1}
\end{gather}
Combining \eqref{eq: cond_kplus1} and \eqref{eq: cond_kminus1}, we can derive a condition on $\beta$:
\begin{equation}
    \fslamstar{K^\star - 1} - \fslamstar{K^\star} \ge \beta \ge \fslamstar{K^\star} - \fslamstar{K^\star + 1}
\label{eq: betacond}
\end{equation}
$\fslamstar{K^\star - 1}$ is the optimal SLAM objective value with one landmark variable fewer than the actual number so the measurements for this missing landmark have to be associated incorrectly to other landmarks (if we assume that $\fslamstar{K^\star}$ contains all correct associations). The residuals $R_{i_k}\tran(y_{j_k} - t_{i_k}) - \zbar_k$ for these measurements do not follow the Gaussian model since they are caused by association errors rather than sensor noise. When the inter-landmark spacing is large relative to the landmark measurement noise (Fig. \ref{fig: asserr_distant}), mis-associating measurements with a distant wrong landmark can cause large residuals in the objective function. This is where formulation \eqref{eq: outer} works robustly\alt{ (see the visualization of $\fslamstar{K}$ in \supplementary)}{}. In this case, the large upper bound $\beta_\text{ub} = \fslamstar{K^\star - 1} - \fslamstar{K^\star}$ creates a favorable gap above the lower bound $\beta_\text{lb} = \fslamstar{K^\star} - \fslamstar{K^\star + 1}$ for $\beta$ to satisfy \eqref{eq: betacond}. Since any mis-association is unlikely to be in the optimal solution, the one extra landmark in $\fslamstar{K^\star + 1}$ would only split the measurements of a landmark. Under additional assumptions\alt{ (see details in \supplementary)}{}, such as moderate odometry noise and negligible residual changes from the other measurement terms, the reduction in the objective value (i.e., $\beta_\text{lb}$) due to the split is therefore capped by the maximum residuals of all measurements associated with a landmark:
\begin{equation}
    \beta_\text{lb} \le \max_{j^\star_k}{\sum_{k \in \mathcal{M}(j^\star_k)}{\|{R^\star_{i_k}}\tran(y^\star_{j^\star_k} - t^\star_{i_k}) - \zbar_k\|_{\Sigma}^2}} \label{eq: betalbub}
\end{equation}
where $\mathcal{M}(j^\star_k)$ is the set of measurements associated with landmark $j^\star_k$ and all the optimal values are computed at $\fslamstar{K^\star}$.
\begin{figure}[h]
    \centering
    \begin{subfigure}{1.0\linewidth}
        \centering
        \scalebox{1.0}{\tikzset{every picture/.style={line width=0.75pt}} %set default line width to 0.75pt        

\begin{tikzpicture}[x=0.75pt,y=0.75pt,yscale=-1,xscale=1]
%uncomment if require: \path (0,67); %set diagram left start at 0, and has height of 67

%Shape: Triangle [id:dp9403057413381728] 
\draw  [draw opacity=0][fill={rgb, 255:red, 245; green, 166; blue, 35 }  ,fill opacity=1 ] (79.76,28.67) -- (87.19,40.67) -- (72.33,40.67) -- cycle ;
\draw  [color={rgb, 255:red, 245; green, 166; blue, 35 }  ,draw opacity=1 ] (84.7,24.09) -- (95.3,34.7)(95.3,24.09) -- (84.7,34.7) ;
\draw  [color={rgb, 255:red, 245; green, 166; blue, 35 }  ,draw opacity=1 ] (88.94,34.73) -- (99.55,45.33)(99.55,34.73) -- (88.94,45.33) ;
\draw  [color={rgb, 255:red, 245; green, 166; blue, 35 }  ,draw opacity=1 ] (77.27,39.73) -- (87.88,50.33)(87.88,39.73) -- (77.27,50.33) ;
\draw  [color={rgb, 255:red, 245; green, 166; blue, 35 }  ,draw opacity=1 ] (61,33.79) -- (71.61,44.39)(71.61,33.79) -- (61,44.39) ;
\draw  [color={rgb, 255:red, 245; green, 166; blue, 35 }  ,draw opacity=1 ] (70,20) -- (80.61,30.61)(80.61,20) -- (70,30.61) ;
\draw  [color={rgb, 255:red, 74; green, 144; blue, 226 }  ,draw opacity=1 ] (288.09,20.36) -- (298.7,30.97)(298.7,20.36) -- (288.09,30.97) ;
\draw  [color={rgb, 255:red, 74; green, 144; blue, 226 }  ,draw opacity=1 ] (274.7,34.09) -- (285.3,44.7)(285.3,34.09) -- (274.7,44.7) ;
\draw  [color={rgb, 255:red, 74; green, 144; blue, 226 }  ,draw opacity=1 ][fill={rgb, 255:red, 74; green, 144; blue, 226 }  ,fill opacity=1 ] (275.3,24.7) -- (285.91,35.3)(285.91,24.7) -- (275.3,35.3) ;
\draw  [color={rgb, 255:red, 74; green, 144; blue, 226 }  ,draw opacity=1 ][fill={rgb, 255:red, 74; green, 144; blue, 226 }  ,fill opacity=1 ] (284.7,44.09) -- (295.3,54.7)(295.3,44.09) -- (284.7,54.7) ;
\draw  [color={rgb, 255:red, 74; green, 144; blue, 226 }  ,draw opacity=1 ][fill={rgb, 255:red, 74; green, 144; blue, 226 }  ,fill opacity=1 ] (295.03,28.97) -- (305.64,39.58)(305.64,28.97) -- (295.03,39.58) ;
%Straight Lines [id:da7144507493821037] 
\draw [color={rgb, 255:red, 245; green, 33; blue, 223 }  ,draw opacity=1 ][line width=0.75]  [dash pattern={on 0.84pt off 2.51pt}]  (79.39,36.7) -- (293.89,25.78) ;
%Straight Lines [id:da5674390852908096] 
\draw [color={rgb, 255:red, 245; green, 33; blue, 223 }  ,draw opacity=1 ][line width=0.75]  [dash pattern={on 0.84pt off 2.51pt}]  (79.39,36.7) -- (290,50) ;
%Straight Lines [id:da42712987313213713] 
\draw [color={rgb, 255:red, 245; green, 33; blue, 223 }  ,draw opacity=1 ][line width=0.75]  [dash pattern={on 0.84pt off 2.51pt}]  (79.39,36.7) -- (280,40) ;
%Straight Lines [id:da23382593091700699] 
\draw [color={rgb, 255:red, 245; green, 33; blue, 223 }  ,draw opacity=1 ][line width=0.75]  [dash pattern={on 0.84pt off 2.51pt}]  (79.39,36.7) -- (280,30) ;
%Straight Lines [id:da3879272573000001] 
\draw [color={rgb, 255:red, 245; green, 33; blue, 223 }  ,draw opacity=1 ][line width=0.75]  [dash pattern={on 0.84pt off 2.51pt}]  (79.39,36.7) -- (300.47,34.4) ;
%Shape: Triangle [id:dp3208288449841161] 
\draw  [draw opacity=0][fill={rgb, 255:red, 74; green, 144; blue, 226 }  ,fill opacity=1 ] (290.1,28.67) -- (297.53,40.67) -- (282.67,40.67) -- cycle ;

\end{tikzpicture}}
        \caption{Large inter-landmark spacing and small measurement noise.}
        \label{fig: asserr_distant}
    \end{subfigure}
    % \hfill
    \par\smallskip
    \begin{subfigure}{1.0\linewidth}
        \centering
        \scalebox{1.0}{\tikzset{every picture/.style={line width=0.75pt}} %set default line width to 0.75pt        

\begin{tikzpicture}[x=0.75pt,y=0.75pt,yscale=-1,xscale=1]
%uncomment if require: \path (0,93); %set diagram left start at 0, and has height of 93

%Shape: Triangle [id:dp07085579074197024] 
\draw  [draw opacity=0][fill={rgb, 255:red, 74; green, 144; blue, 226 }  ,fill opacity=1 ] (182.1,35.92) -- (189.53,47.92) -- (174.67,47.92) -- cycle ;
\draw  [color={rgb, 255:red, 245; green, 166; blue, 35 }  ,draw opacity=1 ] (178.18,16.5) -- (188.79,27.11)(188.79,16.5) -- (178.18,27.11) ;
\draw  [color={rgb, 255:red, 245; green, 166; blue, 35 }  ,draw opacity=1 ] (178.79,45.89) -- (189.39,56.5)(189.39,45.89) -- (178.79,56.5) ;
\draw  [color={rgb, 255:red, 245; green, 166; blue, 35 }  ,draw opacity=1 ] (149.39,66.5) -- (160,77.11)(160,66.5) -- (149.39,77.11) ;
\draw  [color={rgb, 255:red, 245; green, 166; blue, 35 }  ,draw opacity=1 ] (116.39,45.5) -- (127,56.11)(127,45.5) -- (116.39,56.11) ;
\draw  [color={rgb, 255:red, 245; green, 166; blue, 35 }  ,draw opacity=1 ] (138.79,16.11) -- (149.39,26.71)(149.39,16.11) -- (138.79,26.71) ;
\draw  [color={rgb, 255:red, 74; green, 144; blue, 226 }  ,draw opacity=1 ] (199.39,16.5) -- (210,27.11)(210,16.5) -- (199.39,27.11) ;
\draw  [color={rgb, 255:red, 74; green, 144; blue, 226 }  ,draw opacity=1 ] (209.39,56.5) -- (220,67.11)(220,56.5) -- (209.39,67.11) ;
\draw  [color={rgb, 255:red, 74; green, 144; blue, 226 }  ,draw opacity=1 ][fill={rgb, 255:red, 74; green, 144; blue, 226 }  ,fill opacity=1 ] (158.79,17.11) -- (169.39,27.71)(169.39,17.11) -- (158.79,27.71) ;
\draw  [color={rgb, 255:red, 74; green, 144; blue, 226 }  ,draw opacity=1 ][fill={rgb, 255:red, 74; green, 144; blue, 226 }  ,fill opacity=1 ] (149.39,56.5) -- (160,67.11)(160,56.5) -- (149.39,67.11) ;
\draw  [color={rgb, 255:red, 74; green, 144; blue, 226 }  ,draw opacity=1 ][fill={rgb, 255:red, 74; green, 144; blue, 226 }  ,fill opacity=1 ] (179.39,66.5) -- (190,77.11)(190,66.5) -- (179.39,77.11) ;
%Shape: Triangle [id:dp5205076504549448] 
\draw  [draw opacity=0][fill={rgb, 255:red, 245; green, 166; blue, 35 }  ,fill opacity=1 ] (153.76,35.92) -- (161.19,47.92) -- (146.33,47.92) -- cycle ;
%Straight Lines [id:da35689951029796596] 
\draw [color={rgb, 255:red, 245; green, 33; blue, 223 }  ,draw opacity=1 ][line width=0.75]  [dash pattern={on 0.84pt off 2.51pt}]  (153.39,43.95) -- (163.85,22.5) ;
%Straight Lines [id:da0494290379669442] 
\draw [color={rgb, 255:red, 245; green, 33; blue, 223 }  ,draw opacity=1 ][line width=0.75]  [dash pattern={on 0.84pt off 2.51pt}]  (153.39,43.95) -- (154.35,61.5) ;
%Straight Lines [id:da30260266110019707] 
\draw [color={rgb, 255:red, 245; green, 33; blue, 223 }  ,draw opacity=1 ][line width=0.75]  [dash pattern={on 0.84pt off 2.51pt}]  (153.39,43.95) -- (184.35,72) ;
%Straight Lines [id:da7395559533098788] 
\draw [color={rgb, 255:red, 245; green, 33; blue, 223 }  ,draw opacity=1 ][line width=0.75]  [dash pattern={on 0.84pt off 2.51pt}]  (153.39,43.95) -- (204.6,22.25) ;
%Straight Lines [id:da8615133970316411] 
\draw [color={rgb, 255:red, 245; green, 33; blue, 223 }  ,draw opacity=1 ][line width=0.75]  [dash pattern={on 0.84pt off 2.51pt}]  (153.39,43.95) -- (214.85,61.75) ;

\end{tikzpicture}}
        \caption{Small inter-landmark spacing and large measurement noise.}
        \label{fig: asserr_close}
    \end{subfigure}
\caption{Pre-optimization errors induced by mis-associations after removing the blue landmark. The crosses are measurements of landmarks. The length of each dashed magenta line is the magnitude of a mis-association error.}
\label{fig: asserr}
\end{figure}
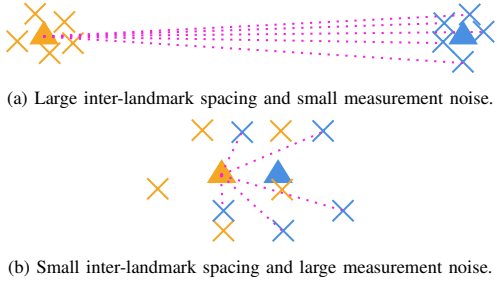

While the upper bound $\beta_\text{ub}$ is high, to ensure that $\beta$ is greater than $\beta_\text{lb}$, the right-hand side of \eqref{eq: betalbub} can be used as a heuristic to set $\beta$ so as to satisfy \eqref{eq: betacond}. Although $\arg\!\max_{j^\star_k}$ and $\mathcal{M}(j^\star_k)$ are not known beforehand, we can estimate the number of measurements for a landmark ($|\hat{\mathcal{M}}(j^\star_k)|$) from prior information such as average robot speed, field of view, sensor range, and sensor frame rate\alt{ (see \supplementary)}{}. Since $\sum_{k  \in \mathcal{M}(j^\star_k)}\|{R^\star_{i_k}}\tran(y^\star_{j^\star_k} - t^\star_{i_k}) - \zbar_k\|_{\Sigma}^2$ follows a $\chi^2$ distribution due to \eqref{eq: posgen}, $\beta$ can be heuristically set by the inverse $\chi^2$ distribution with $d \times |\hat{\mathcal{M}}(j^\star_k)|$ degrees of freedom and a sufficiently large probability threshold to be robust to violations of the assumptions underlying \eqref{eq: betalbub}. The preceding analysis provides a more principled explanation of the heuristic in \cite{zhang2023data}.

When the inter-landmark spacing is small relative to the landmark measurement noise, measurements from different landmarks may mix together (Fig. \ref{fig: asserr_close}). Mis-associating measurements with a nearby landmark incurs a cost similar to the original sensor noise error. Thus, $\fslamstar{K^\star - 1}$ is close to $\fslamstar{K^\star}$, and the upper bound $\beta_\text{ub}$ is lower. In the extreme case where two landmarks are on top of each other, removing one does not increase the objective value. Therefore, $\beta$ is set empirically in this regime because setting $\beta$ using the previous heuristic \eqref{eq: betalbub} may cause it to exceed $\beta_\text{ub}$. However, a starting value for $\beta$ ($\frac{2|\hat{\mathcal{M}}(j^\star_k)|}{\pi}$) based on quantitative modeling of the effect of adding a landmark ($\fslamstar{K^\star + 1}$) by splitting Gaussian samples\alt{, along with additional guidelines provided in \supplementary,}{ \supplementary} is used in an experiment (Section \ref{sec: kitchen}) where landmarks are considered clustered relative to the measurement noise.

\subsection{Incorporating Semantic Feature Vectors} \label{sec: feat}
\probName\ studies the setting where only spatial measurements of landmarks are available. However, for applications such as object SLAM, we may also have landmark feature vectors predicted by an object detector (e.g., one-hot classification vectors from Mask R-CNN \cite{he2017mask}) or a foundation model (e.g., real-valued vectors from DINO \cite{caron2021emerging}). To enable synergy between data association and feature vector estimation, we incorporate these vectors by assuming a Gaussian noise model (which may be relaxed in practice in Section \ref{sec: prac}):
\begin{equation}
    \zsbar{k} = \ys{j_k} + \epss{k}
\label{eq: semgen}
\end{equation}
where $\zsbar{k} \in \realspace^s$ is the semantic measurement (i.e., a feature vector of length $s$) of the landmark semantics ($\ys{j_k} \in \realspace^s$), corrupted by zero-mean Gaussian noise $\epss{k} \sim \gauss(0, \Sigma_s)$. The inner problem \eqref{eq: inner} is modified to include a semantic residual:
\begin{multline}  % cannot split within {}
    \min_{\substack{\boldx \in \SE(d)^N \\ \boldy \in \realspace^{d \times K} \\ \boldys \in \realspace^{s \times K}}}
    f_{\text{odom}}(\boldx) + \sum_{k=1}^M\min_{j_k \in [K]} \\
    \Bigg\{\|R_{i_k}\tran(y_{j_k} - t_{i_k}) - \zbar_k\|_{\Sigma}^2
    + w_s^2\|\ys{j_k} - \zsbar{k}\|_{\Sigma_s}^2\Bigg\}
\label{eq: seminner}
\end{multline}
where $w_s^2$ is a weighting term used to normalize the feature-vector residual by length (if the residual scales with length) and to have a tunable balance between the spatial residual and the semantic residual. In the alternating minimization, the association step now solves the new objective function:
\begin{equation}
    \min_{\substack{\boldy \in \realspace^{d \times K} \\ \boldys \in \realspace^{s \times K}}}
    \sum_{k=1}^M\min_{j_k \in [K]}
    \left\|
    \begin{bmatrix}
        R_{i_k}\zbar_k + t_{i_k} \\
        \frac{\sigma}{\sigma_s}w_s\zsbar{k}
    \end{bmatrix}
    - 
    \begin{bmatrix}
        y_{j_k} \\
        \frac{\sigma}{\sigma_s}w_s\ys{j_k}
    \end{bmatrix}
    \right\|^2
\label{eq: semkmeans}
\end{equation}
where we assume isotropic noise covariances ($\Sigma = \mathsf{diag}(\sigma)$ and $\Sigma_s = \mathsf{diag}(\sigma_s)$) to simplify the exposition, although the anisotropic case can be derived in a similar fashion. In practice, $\sigma_s$ is often unknown so $\frac{\sigma}{\sigma_s}w_s$ can be lumped into a single hyperparameter. Equation \eqref{eq: semkmeans} stacks the spatial and semantic residuals so that they can be solved jointly in a single pass of k-means (with k-means++ \cite{arthur2006k}). The data association ($j_k$) becomes the cluster assignment, and the cluster centers are the landmark positions ($y_{j_k}$) concatenated with the weighted feature vectors ($\frac{\sigma}{\sigma_s}w_s\ys{j_k}$).

In the estimation step, we take the data association ($j_k$) results from the association step and solve $\{\boldx, \boldy\}$ using a SLAM solver \cite{dellaert2019georgia}. For the semantic part, once the data associations are given, it is decoupled from $\{\boldx, \boldy\}$ in \eqref{eq: seminner} so it can be solved independently. The solution to $\min_{\boldys}\sum_{k=1}^M{\|\ys{j_k} - \zsbar{k}\|}^2$ is simply the average of $\zsbar{k}$'s associated with each $\ys{j_k}$, which is already computed as the cluster center in the k-means step.

For the outer problem, we retain the original objective function \eqref{eq: outer} because $\beta$ is set according to the Gaussian assumption of the generative process (Section \ref{sec: optimcond}), and the Gaussian assumption for the semantic part \eqref{eq: semgen} may be less reliable than that for the positional part \eqref{eq: posgen}. Therefore, the semantic part does not enter the outer problem.

{\inlinetitle Remark.} The feature vector estimation can be a simple add-on to the original algorithm. Instead of performing k-means in the $d$-dimensional space, we now run it in the ($d+s$)-dimensional space. This allows both spatial information and semantic information to facilitate data association during the association step, while the association results can in turn help better estimate the feature vectors of landmarks. For example, a semantically erroneous feature vector may be corrected if the measurement is correctly associated based on its positional component. The updated algorithm is named \algNameS\ for its additional semantic component.

\subsection{Semi-Incremental Processing} \label{sec: inc}
The original \probName\ work \cite{zhang2023data} adopts a batch processing scheme so that the entire trajectory is processed all at once and all the data associations are jointly determined, thereby avoiding the difficulty of correcting wrong past data associations in a pure incremental scheme. However, drift accumulates along a trajectory and leads to large errors in the initial guess for $\boldx$ (i.e., the odometry chain) toward the end of the trajectory when the trajectory is long or the odometry noise is high, causing lower accuracy for the batch scheme. 

To reduce drift, we propose a semi-incremental scheme following the divide-and-conquer strategy. Specifically, we define intervals of robot poses indexed by $n$, $\boldx_n = \{x_i\}_{i=(n-1)N_n+1}^{\min(nN_n, N)}$, where $n \in \{1,\ldots,\lceil\frac{N}{N_n}\rceil\}$ and $N_n \in \{1,\ldots,N\}$ is the interval length. These intervals divide the trajectory into $\lceil\frac{N}{N_n}\rceil$ segments. We first run \algNameS\ on each segment independently (either sequentially if the measurements arrive sequentially or in parallel if they are available at once). In the heuristic for $\beta$ (Section \ref{sec: optimcond}), the number of measurements is taken as $N_n$ for the segment processing if $|\hat{\mathcal{M}}(j^\star_k)| > N_n$, because $N_n$ poses can at most produce $N_n$ observations of a landmark. Once all segments have been processed, we chain them together by connecting the last pose of a segment to the first pose of the next segment using odometry and propagating all pose updates along the chain. We then perform global optimization with \algNameS\ on the chained segments, which now represent the entire trajectory (Fig. \ref{fig: semiinc}). The key saving is that the outer problem solution space over the entire trajectory is reduced from $K \in \{1,\ldots,M\}$ to $\{1,\ldots,\sum_n\hat{K}_n\}$, where $\hat{K}_n$ is the estimated number of landmarks for segment $n$. We name the updated algorithm \algNameDNC\, with ``i" standing for (semi-)incremental and ``d" standing for divide-and-conquer. It is semi-incremental because poses are processed block-wise in segments.

{\inlinetitle Remark.} The segment processing first reduces drift within each segment. Therefore, the initial guess for the trajectory in the global optimization contains less accumulated error than the raw odometry chain. This scheme is analogous to the local and global bundle-adjustment in ORB-SLAM \cite{mur2015orb}. In terms of computation cost, the outer problem solver originally requires $\mathcal{O}(\log(M))$ evaluations of $\fslamstar{K}$, each of which requires solving a nonlinear SLAM problem multiple times in the alternating minimization steps. The SLAM solver complexity can be as high as cubic in the problem size. Overall, this is  computationally expensive and becomes especially worse as the number of poses ($N$) and the number of measurements ($M$) increase. By performing the proposed semi-incremental scheme, we have $\lceil\frac{N}{N_n}\rceil$ small problems that keep the $\mathcal{O}(\log(M_n))$ and cubic scaling low. Additionally, the global optimization can take advantage of the already estimated numbers of landmarks for the segments $\sum_n\hat{K}_n$, which is typically much smaller than $M$ because each landmark usually has tens or hundreds of measurements for a high-frame-rate sensor such as a camera. The global optimization does not inherit the association results from the segment processing to avoid convergence to local minima.

\subsection{Practical Considerations} \label{sec: prac}
% discuss: 1. one hot feature vector 2. parallel 3. fp, fn, wrong label
{\inlinetitle One-hot feature vectors.} In practice, researchers may want to use one-hot classification vectors from an object detector (e.g., Mask R-CNN \cite{he2017mask}) to describe object landmarks. Although our formulation and algorithm are derived from the real-valued feature vector assumption, our algorithm also works functionally for one-hot vectors. During the k-means association step (Section \ref{sec: feat}), the use of one-hot vectors would violate the Gaussian residual assumption \eqref{eq: semgen} and the physical meaning of $\sigma_s$, but $\frac{\sigma}{\sigma_s}w_s$ can be treated as a hyperparameter to balance the spatial and semantic terms. In the estimation step, taking the average of the one-hot vectors associated with a landmark and selecting the $\arg\!\max$ element in the average vector is equivalent to majority voting for the landmark class. Therefore, the same algorithm can be applied to one-hot feature vectors without any modification.

{\inlinetitle Parallel global optimization.} The global optimization (Section \ref{sec: inc}) can run in parallel with the segment processing. This parallel scheme runs the global optimization on a different thread whenever there are more than two processed segments available. The output of the parallel global optimization is stored as a segment so that it will be combined with the subsequent segments and jointly optimized together with them in the next global optimization. Fig. \ref{fig: semiinc} shows a schematic of the different processing schemes. The parallel scheme is advantageous if the up-to-date pose estimates, incorporating the latest global optimization results, are required in real time. In principle, it should also be more accurate than the series scheme (Section \ref{sec: inc}) because drift is reduced further through the parallel global optimization before the final whole-trajectory global optimization. This parallel implementation is named \algNameLL, with ``ll" standing for parallel.
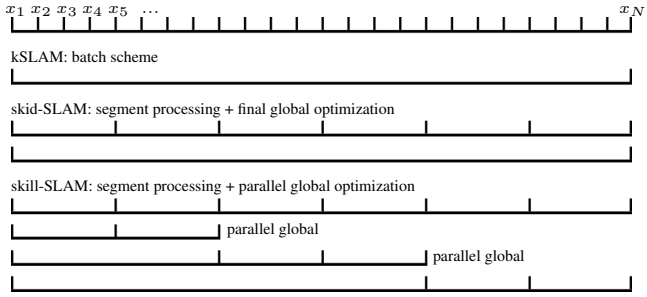
\begin{figure}[h]
\vspace{1.2em} % adjust margin
\centering
\begin{overpic}[width=0.95\columnwidth, tics=5]{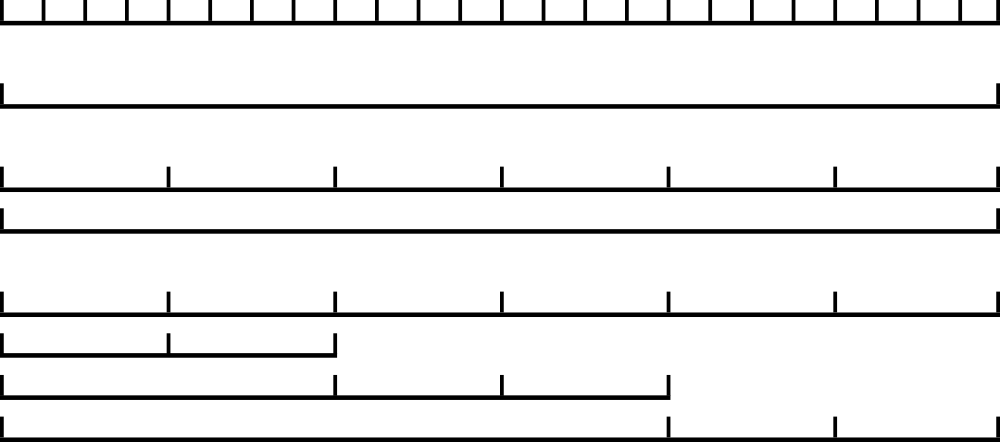} % [tics=5, grid] to put grid
    \put(-1,45){\ssmall{$x_1$}}
    \put(3.1,45){\ssmall{$x_2$}}
    \put(7.3,45){\ssmall{$x_3$}}
    \put(11.4,45){\ssmall{$x_4$}}
    \put(15.6,45){\ssmall{$x_5$}}
    \put(21,45){\ssmall{$...$}}
    \put(98,45){\ssmall{$x_N$}}
    \put(0,37){\ssmall{\algName: batch scheme}}
    \put(0,28.7){\ssmall{\algNameDNC: segment processing + final global optimization}}
    \put(0,16.2){\ssmall{\algNameLL: segment processing + parallel global optimization}}
    \put(34.8,9.3){\ssmall{parallel global}}
    \put(68,5){\ssmall{parallel global}}
\end{overpic}
\caption{Different processing schemes.}
\label{fig: semiinc}
\vspace{-0.9em} % adjust margin
\end{figure}

{\inlinetitle Detector mistakes.} Object detectors, and similarly visual foundation models, may make three types of errors: detecting an irrelevant object as one of the target object classes (false positives), missing detections, and mislabeling objects as one of the other target classes (wrong labels). We discuss their effects and possible treatments. Sporadic false positives that are spatially close to true positives tend to be merged with the true positives (Section \ref{sec: pool}) because the reduction in $\fslamstar{K}$ by forming a new landmark for them usually does not exceed the penalty $\beta$ in \eqref{eq: outer}. Handling false positives is an interesting direction for future work. For occasional missed detections, if their corresponding landmark can still be estimated from detections at other viewing angles, we can render the estimated object in the frames where it is missed to correct the missing detections in post-processing. For mislabeling, a similar assumption is that the object is correctly labeled in the majority of frames from other viewing angles and correctly associated using the spatial information. In such a case, our algorithm can natively correct incorrect labels by selecting the most-voted class as the final semantic class estimate.

\section{Experiments} \label{sec: exp}
\subsection{Synthetic Datasets and Ablation Study} \label{sec: synth}
Benchmarking and ablation experiments are conducted on the same 2D and 3D synthetic grid datasets as those in \cite{zhang2023data}, but with the addition of one-hot and real-valued semantic feature vectors. We compare against three baselines, (i) raw odometry (ii) an alternating minimization algorithm (named Oracle) that is given the ground-truth number of landmarks and an initial guess for each landmark \cite{doherty2022discrete, zhang2023data}, and (iii) the original batch processing \algName\ algorithm \cite{zhang2023data}. We also evaluate two ablated versions of our proposed algorithms, (i) running \algName\ repetitively every $N_n$ poses on all poses seen so far to (semi-)incrementally process the trajectory (\algNameI), and (ii) enabling the semantic feature vector processing (Section \ref{sec: feat}) on top of \algNameI\ (\algNameIS). The expectation is that \algNameI\ should achieve higher accuracy than \algName\ because odometry drift is reduced by incrementally rerunning the algorithm instead of performing one-time batch processing. \algNameIS\ may achieve the highest accuracy among all since it is the most redundant version. However, both \algNameI\ and \algNameIS\ are computationally expensive and even prohibitive on large datasets due to the naive incremental rerunning scheme. We increase the number of inner solver iterations for all \algName\ variants on the synthetic datasets to improve convergence and obtain a more accurate characterization of the trend. The performance of all methods is measured in terms of absolute trajectory error (ATE) after $\mathrm{SIM}(3)$ alignment between the estimated trajectory and a reference trajectory, which is obtained by optimizing the same measurements given the ground-truth data associations. We typically find that ATE is positively correlated with other metrics such as landmark position error. We sweep various dataset parameters in Fig. \ref{fig: odomnoise} -- \ref{fig: semnoise_realval}. Each data point in the plots is the median result of 20 simulations (different realizations of the noise) and the shaded region shows the 25\% -- 75\% quartile range. We use the heuristic in Section \ref{sec: optimcond} to set $\beta$. Detailed dataset settings, algorithm parameter settings, and evaluation protocols are in \supplementary.

{\inlinetitle Odometry noise.} The proposed semi-incremental scheme (Section \ref{sec: inc}) is designed to reduce drift while keeping the computational load low. In Fig. \ref{fig: odomnoise}, we show that all the semi-incremental methods (\algNameI, \algNameIS, \algNameDNC, and \algNameLL) generally outperform the batch \algName\ algorithm. Further, the incorporation of semantic feature vectors (\algNameIS, \algNameDNC, and \algNameLL) results in even lower ATE, particularly when the odometry noise is high. The baseline alternating minimization algorithm (Oracle) is almost always inferior, which is consistent with the results reported in \cite{zhang2023data}.
\begin{figure}[h]
    \centering
        \includegraphics[width=1.0\linewidth]{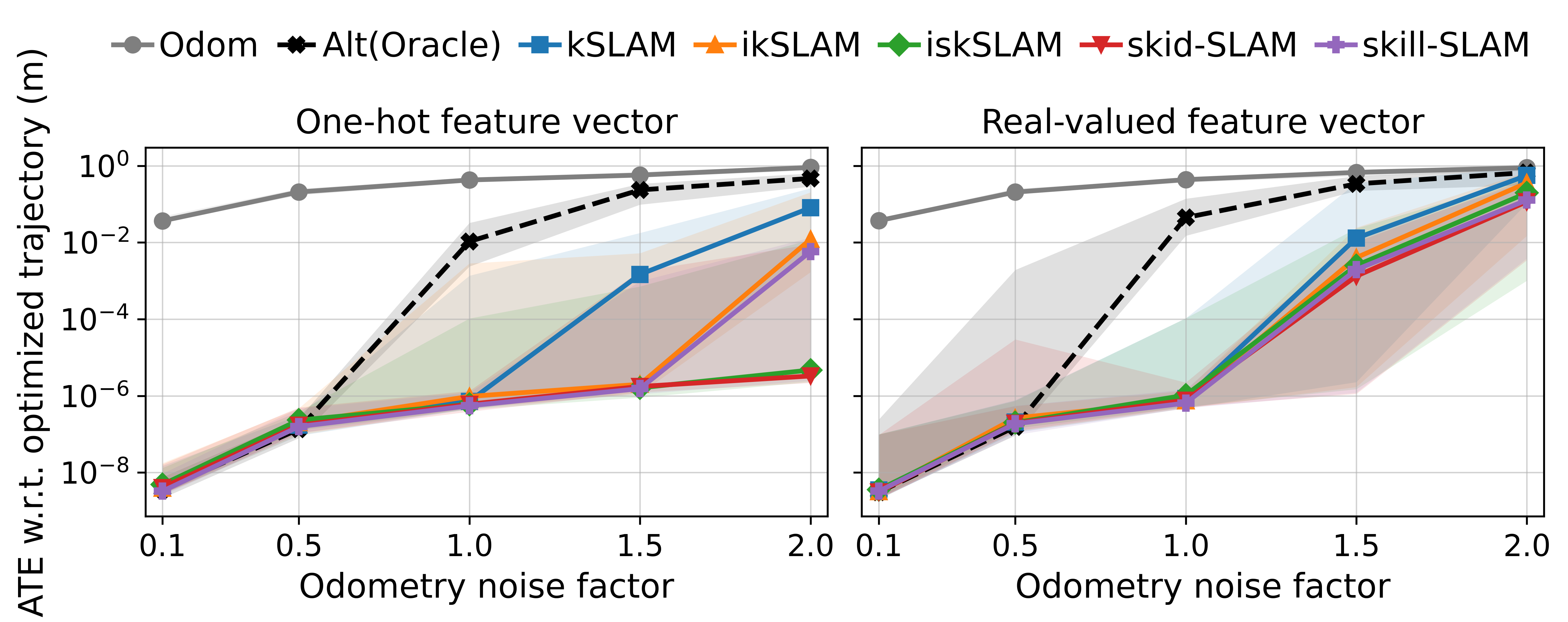}
\caption{Absolute trajectory error (ATE) in log scale for different odometry noise levels on the 3D grid dataset. The noise factor is multiplied to the default noise level (translational std = 0.05 (m) and rotational std = 0.005).}
\label{fig: odomnoise}
\end{figure}

{\inlinetitle Runtime.} Although the naive incremental rerunning scheme (\algNameI\ and \algNameIS) also reduces ATE, it has worse runtime (Fig. \ref{fig: runtime}). Among the algorithms that incorporate semantic information (\algNameIS, \algNameDNC, and \algNameLL), \algNameIS\ is clearly the slowest. \algNameLL\ is slower than \algNameDNC\ because the parallel global-optimization thread has to wait for the last global optimization to finish before starting the final whole-trajectory optimization, whereas \algNameDNC\ starts the whole-trajectory optimization immediately after the segment processing is completed. Among all the algorithms, Oracle is the fastest because it is given the ground-truth number of landmarks and does not spend time estimating it. On the smaller 3D grid dataset (216 poses), \algName\ is faster than our proposed algorithms. However, on the larger 2D grid dataset (500 poses), the savings from our divide-and-conquer strategy (Section \ref{sec: inc}) become apparent and enable the proposed algorithms to match or surpass the other baselines, particularly when more landmarks are present. The runtime advantage becomes even more significant when the number of measurements per landmark is higher ($|\mathcal{M}(j^\star_k)| = 10$ here), resulting in a large total number of measurements ($M$) (Section \ref{sec: kitchen}).
\begin{figure}[h]
\vspace{0.5em} % adjust margin
    \centering
        \includegraphics[width=1.0\columnwidth]{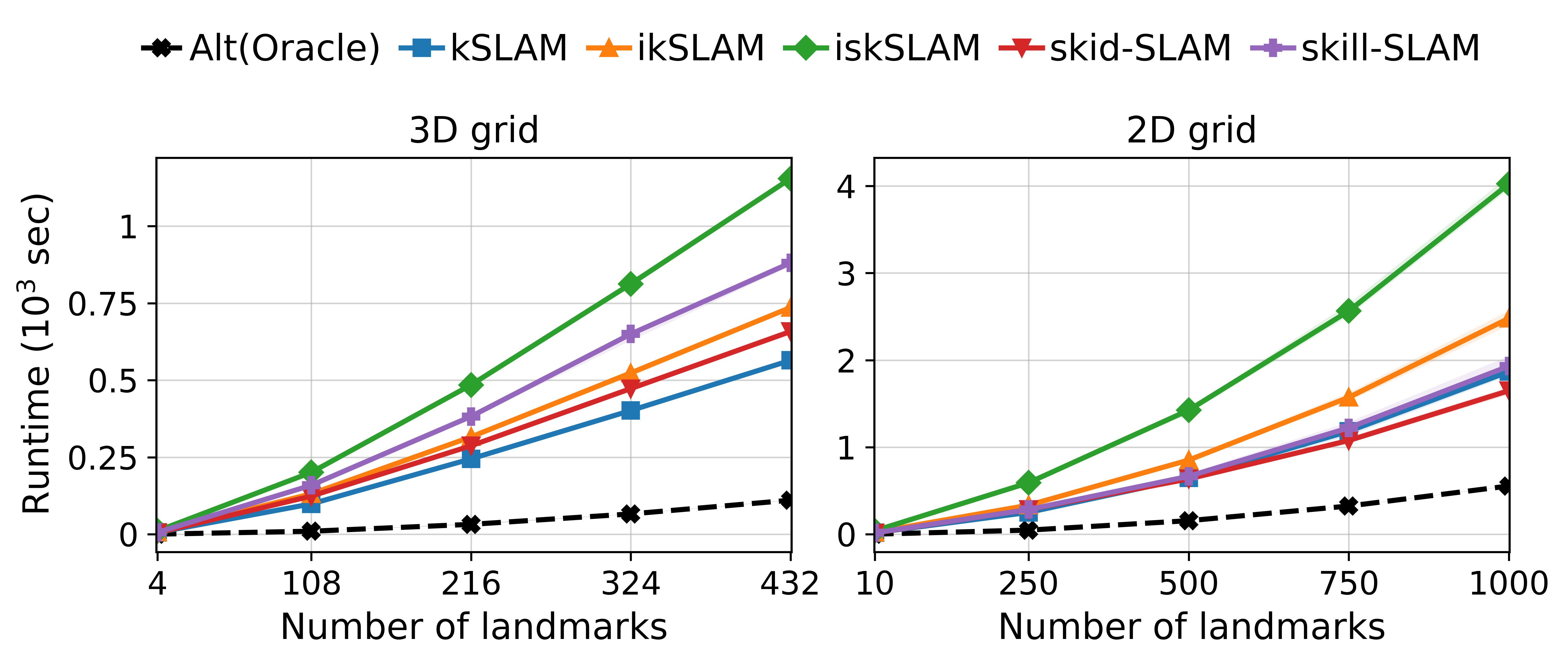}
\caption{Runtime for different numbers of landmarks.}
\label{fig: runtime}
\vspace{-0.5em} % adjust margin
\end{figure}

{\inlinetitle Semantic noise.} We vary the nominal detector mislabeling probability for one-hot semantic vectors (Fig.~\ref{fig: semnoise_onehot}) and the Gaussian noise standard deviation for real-valued semantic feature vectors (Fig.~\ref{fig: semnoise_realval}). The nominal detector mislabeling probability specifies the probability with which each semantic measurement is intentionally corrupted. The ``Raw" curve in each figure denotes the actual semantic error in the measurements after the noise is realized. The remaining curves show the errors of the estimated landmark semantics. The algorithms that process semantic information (\algNameIS, \algNameDNC, and \algNameLL) generally achieve lower ATE than those do not (\algName\ and \algNameI), but the benefits diminish as the semantic information becomes less reliable. On the other hand, joint estimation of semantics yields substantial reductions in landmark semantic estimation errors in both the one-hot and real-valued cases, as evidenced by the large differences between the raw semantic errors and the estimation errors.
\begin{figure}[h]
    \centering
        \includegraphics[width=1.0\columnwidth]{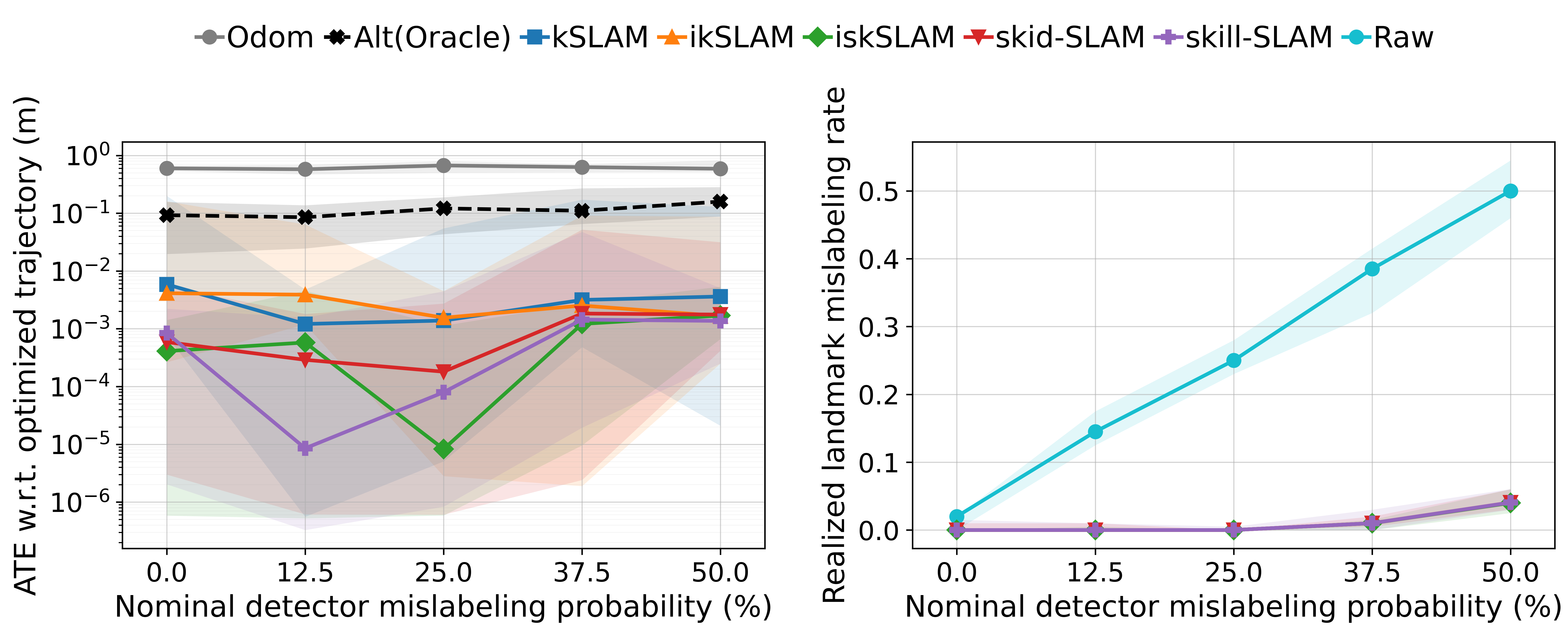}
\caption{Absolute trajectory error (ATE) in log scale and mislabeling rate for different nominal detector mislabeling probabilities on the 2D grid dataset with one-hot feature vectors.}
\label{fig: semnoise_onehot}
\end{figure}

\begin{figure}[h]
\vspace{0.2em}
    \centering
        \includegraphics[width=1.0\columnwidth]{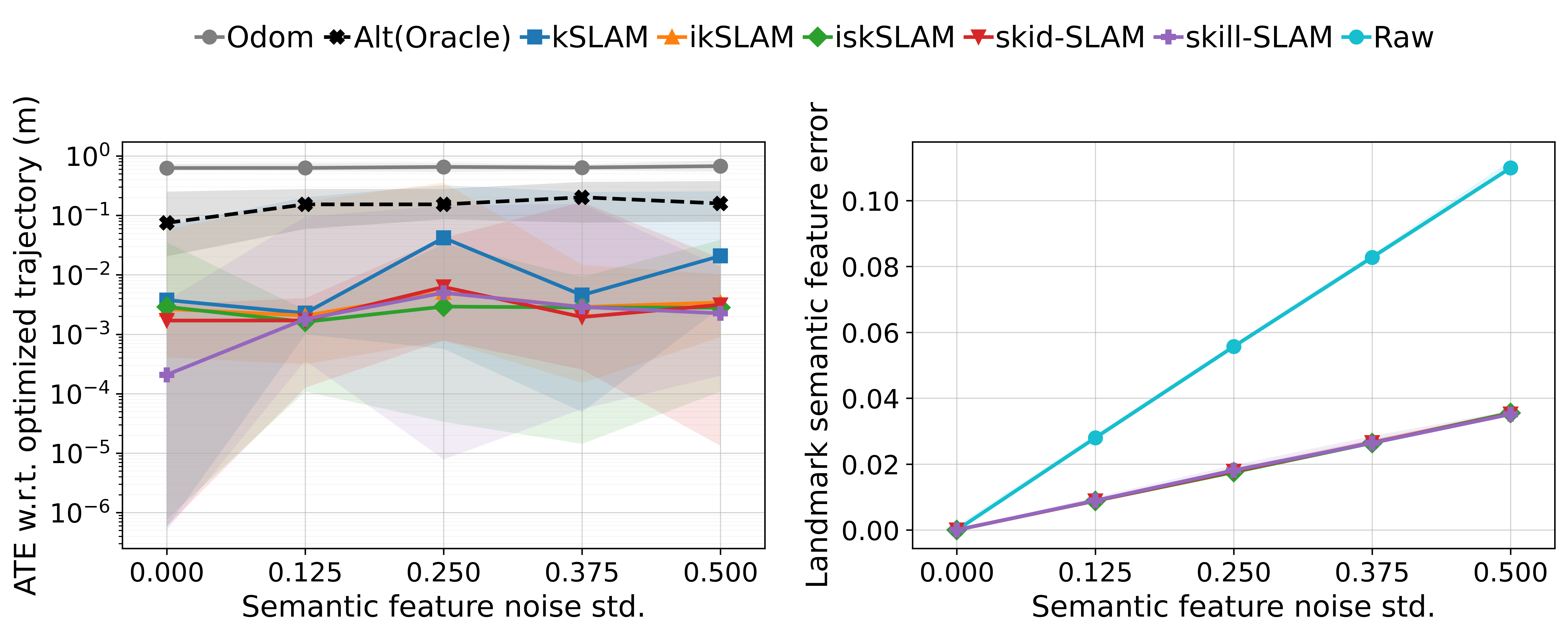}
\caption{Absolute trajectory error (ATE) in log scale and feature vector estimation error (Euclidean norm) for different semantic noise standard deviations on the 2D grid dataset with real-valued feature vectors.}
\label{fig: semnoise_realval}
\vspace{-1.0em}
\end{figure}

\subsection{Kitchen Datasets with One-Hot Feature Vectors} \label{sec: kitchen}
We evaluate our algorithms on in-house indoor kitchen datasets and designed to test their ability to incorporate object detections and one-hot semantic vectors. We place various objects in a kitchen environment and remotely operate an RB-Y1 humanoid robot equipped with a ZED stereo camera, guiding it through the kitchen over multiple passes to mimic a household robot performing daily tasks. Two baselines are included: (i) the batch processing \algName\ algorithm \cite{zhang2023data}, and (ii) an object-centric SLAM method \cite{singh2024opti}. The input odometry is computed using RGB-D visual odometry \cite{zhou2018open3d, steinbrucker2011real}. The landmark measurements are taken as the centroids of segmented objects in point clouds back-projected from depth images, where the object masks are obtained from Mask R-CNN \cite{he2017mask} for \algName\ and our proposed algorithms, and from YOLOv8 for the object-centric SLAM baseline \cite{singh2024opti}. To provide ground truth for evaluation, we place a unique AprilTag \cite{wang2016apriltag} close to each object (rather than on the object itself, to avoid interfering with the object detector). These tags serve as pseudo-landmarks whose measurements are extracted in a similar manner from the back-projected depth images. The identities of the tags allow us to associate the tag measurements and run standard landmark-based SLAM. The resulting estimated trajectory is used as the pseudo-ground-truth trajectory for evaluation, and the estimated tag positions serve as qualitative references for the object positions because they are slightly offset from the objects, making them unsuitable for quantitative evaluation).

The objects are strategically arranged in two layouts: (i) clusters of objects from different classes, and (ii) clusters of objects from the same class. As discussed in Section \ref{sec: optimcond}, for environments of this type, where landmarks are close to one another relative to the landmark measurement noise, $\beta$ is set empirically. However, \alt{following the modeling results and guidelines presented in \supplementary}{based on quantitative modeling of the effect of adding a landmark by splitting Gaussian samples \supplementary}, a reasonable starting value for $\beta$ is $\frac{2|\hat{\mathcal{M}}(j^\star_k)|}{\pi}$, which is used for the kitchen datasets.

As shown in Fig. \ref{fig: kitchen} and Table \ref{tab: kitchen}, our proposed algorithms (\algNameDNC\ and \algNameLL) are able to recover from several errors, such as missing landmarks and extra landmarks, made by \algName\, and reduce the trajectory error in both cases (i) and (ii). In addition, our proposed algorithms exhibit a clear speed advantage on these larger datasets due to the savings discussed in Section \ref{sec: inc}, continuing the trend observed in Section \ref{sec: synth}. The more challenging layout in case (ii) does degrade performance, causing landmarks of the same class to be merged rather than correctly distinguished. However, the estimated landmark labels are always correct. The object-centric SLAM baseline \cite{singh2024opti} clearly performs worse. We find that it relies heavily on the manually set covariance matrices and thresholds.
\begin{table*}[th]
\vspace{1.0em} % adjust margin
\begin{center}
% \small % caption font size
\caption{Absolute trajectory error (m) and runtime (s) on the kitchen datasets.}
\label{tab: kitchen}
\begin{tabular}{|l| c c c| c c c c c|}
    \hline
    Dataset & $N$ & $K_\text{gt}$ & $M$ &  Odometry & Object-Centric \cite{singh2024opti} & \algName\ & \algNameDNC & \algNameLL \\
    \hline
    \hline
    Diff Class & 3719 & 21 & 6552 & 0.264 & 0.361 & 0.034 / 429 & 0.033 / 325 & 0.033 / 364 \\
    \hline
    Same Class & 4009 & 21 & 6872 & 0.200 & 0.147 & 0.039 / 503 & 0.039 / 321 & 0.037 / 342 \\
    \hline
\end{tabular}
\end{center}
\vspace{-0.7em} % adjust margin
\end{table*}
\begin{figure*}[th]
    \centering
    \begin{subfigure}{0.19\linewidth}
        \includegraphics[width=1.0\columnwidth]{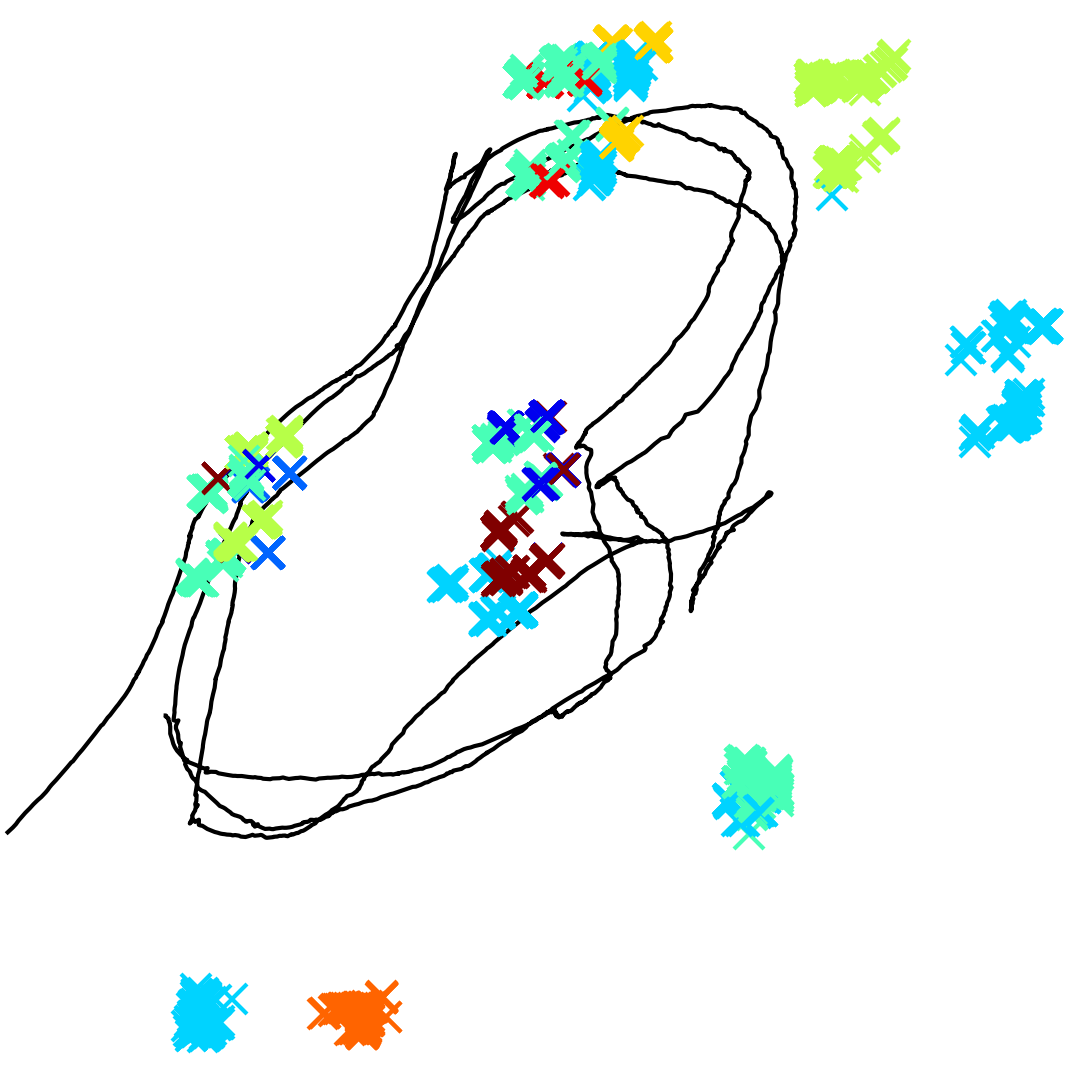}
    \end{subfigure}
    \hspace{0.5em}
    \begin{subfigure}{0.19\linewidth}
        \includegraphics[width=1.0\columnwidth]{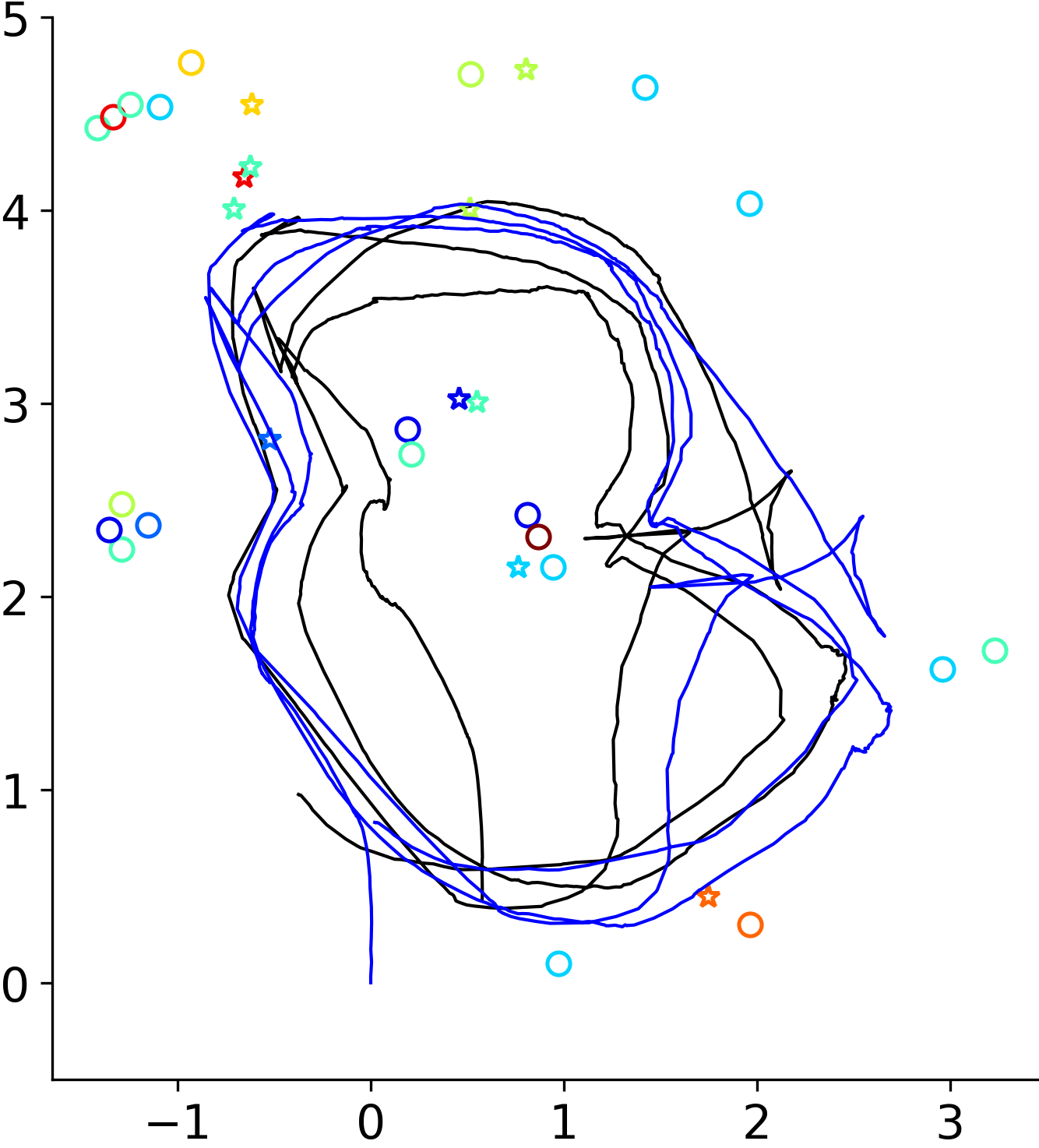}
    \end{subfigure}
    \begin{subfigure}{0.19\linewidth}
        \includegraphics[width=1.0\columnwidth]{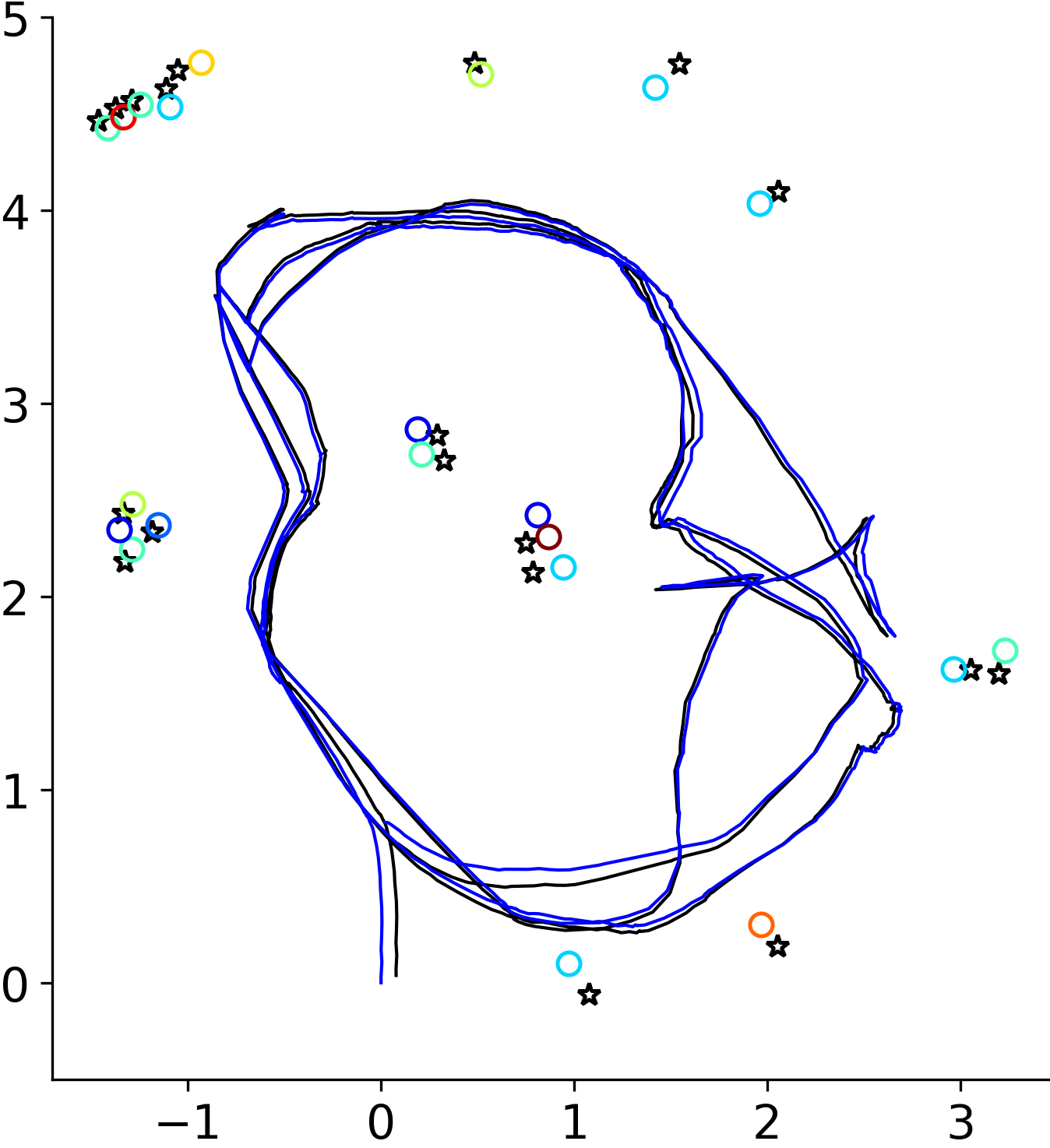}
    \end{subfigure}
    \begin{subfigure}{0.19\linewidth}
        \includegraphics[width=1.0\columnwidth]{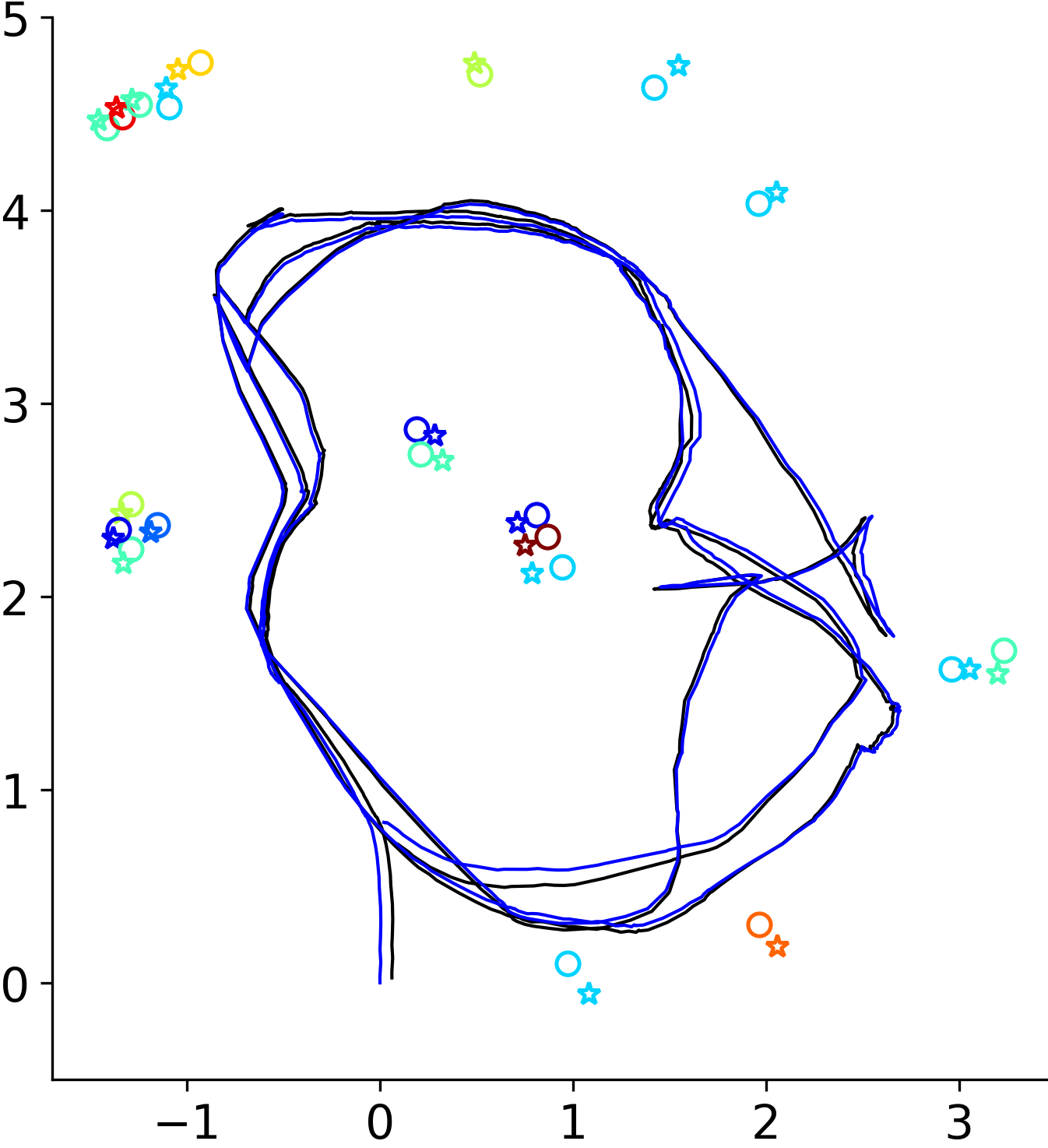}
    \end{subfigure}
    \begin{subfigure}{0.19\linewidth}
        \includegraphics[width=1.0\columnwidth]{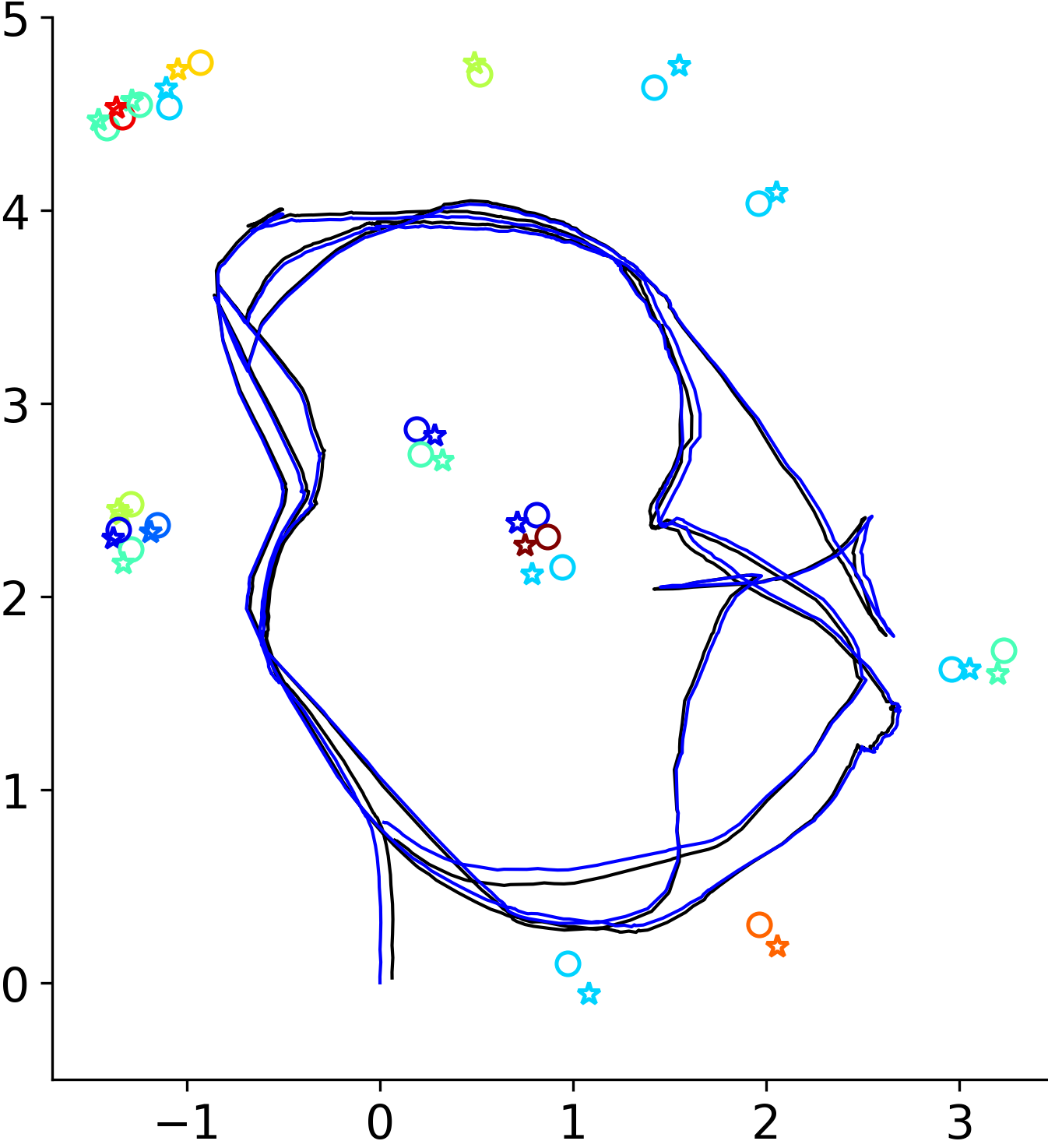}
    \end{subfigure}
    \par\smallskip
    \begin{subfigure}{0.19\linewidth}
        \includegraphics[width=1.0\columnwidth]{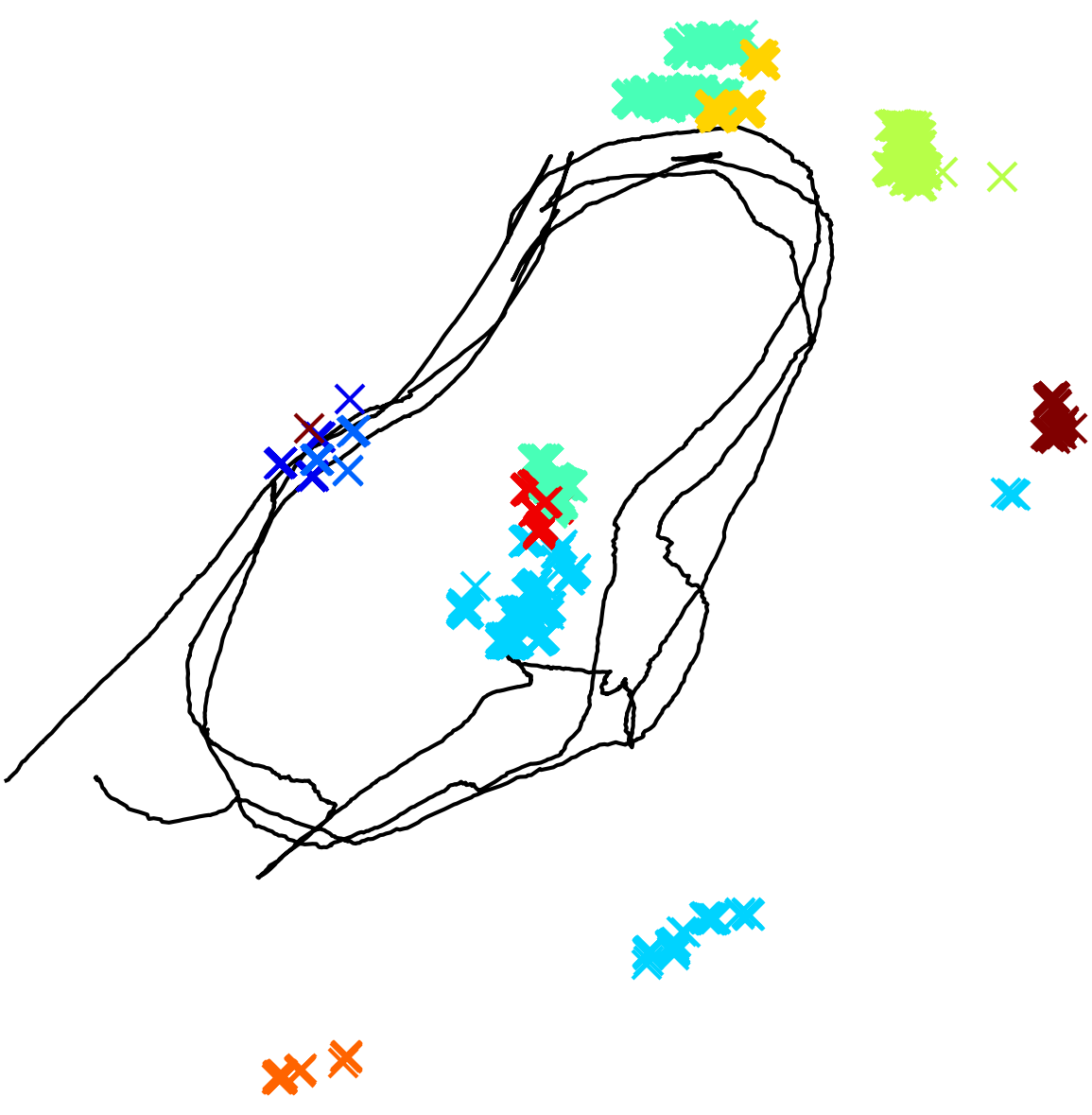}
        \caption{Raw data}
    \end{subfigure}
    \hspace{0.5em}
    \begin{subfigure}{0.19\linewidth}
        \includegraphics[width=1.0\columnwidth]{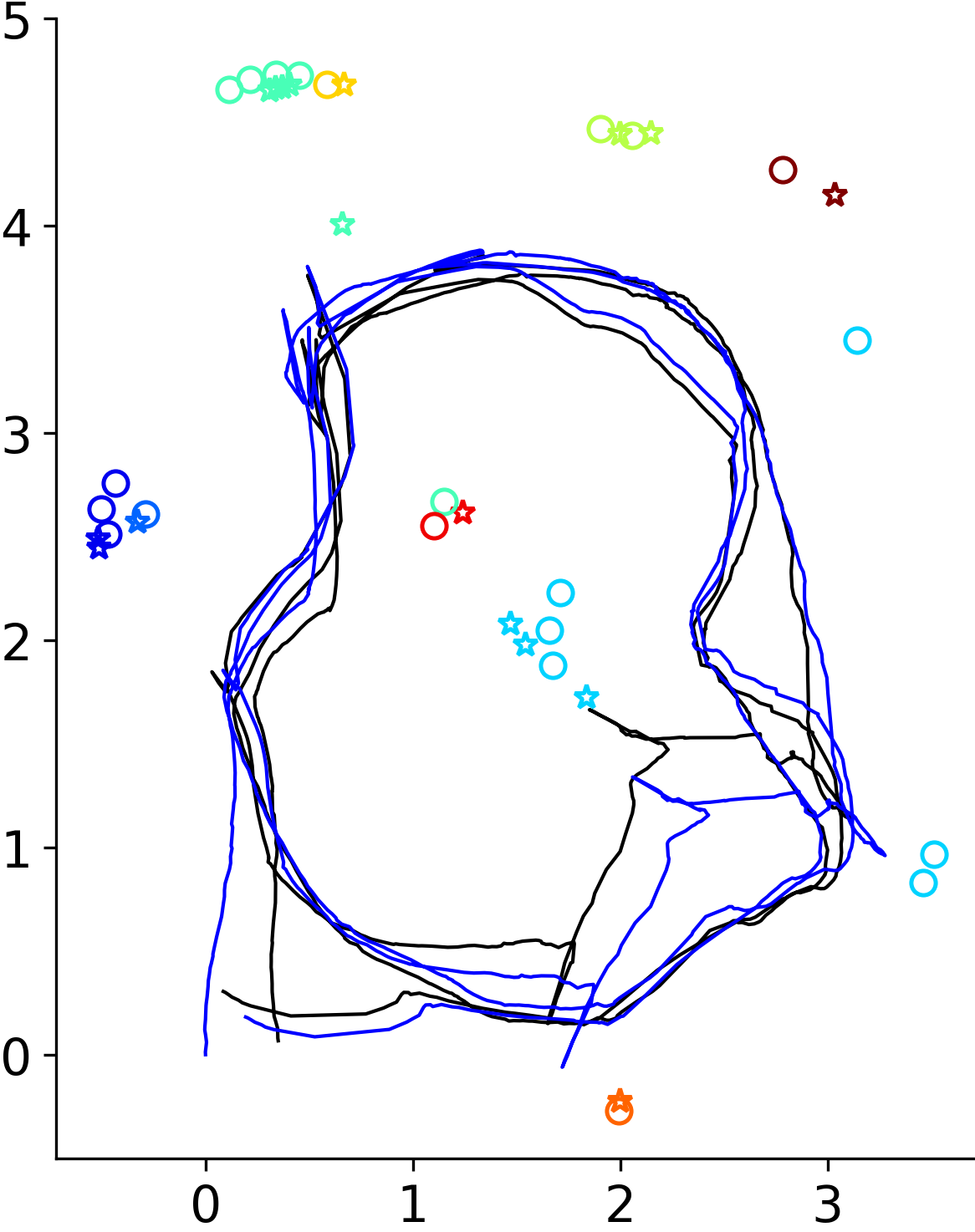}
        \caption{Object-Centric SLAM \cite{singh2024opti}}
    \end{subfigure}
    \begin{subfigure}{0.19\linewidth}
        \includegraphics[width=1.0\columnwidth]{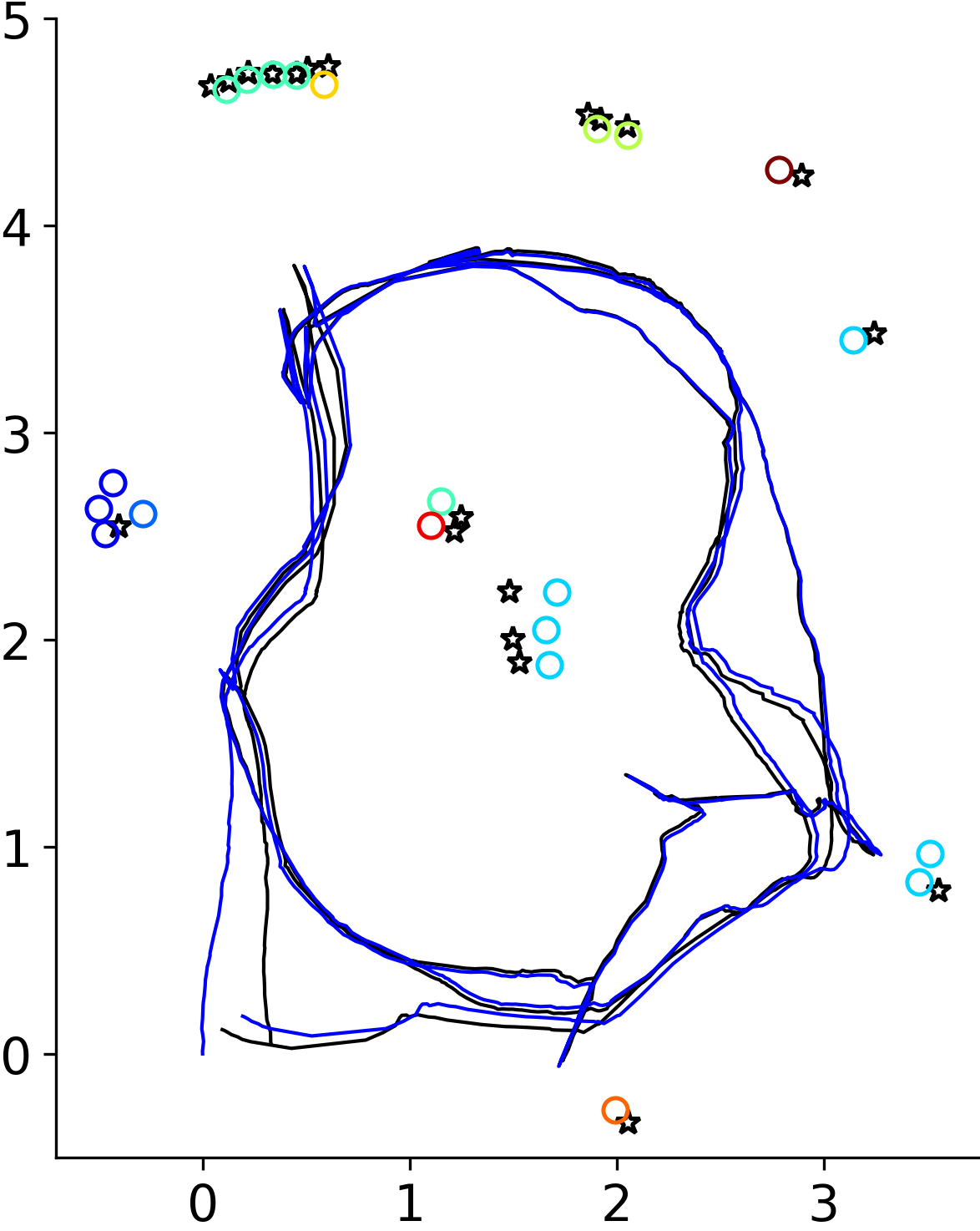}
        \caption{\algName\ \cite{zhang2023data}}
    \end{subfigure}
    \begin{subfigure}{0.19\linewidth}
        \includegraphics[width=1.0\columnwidth]{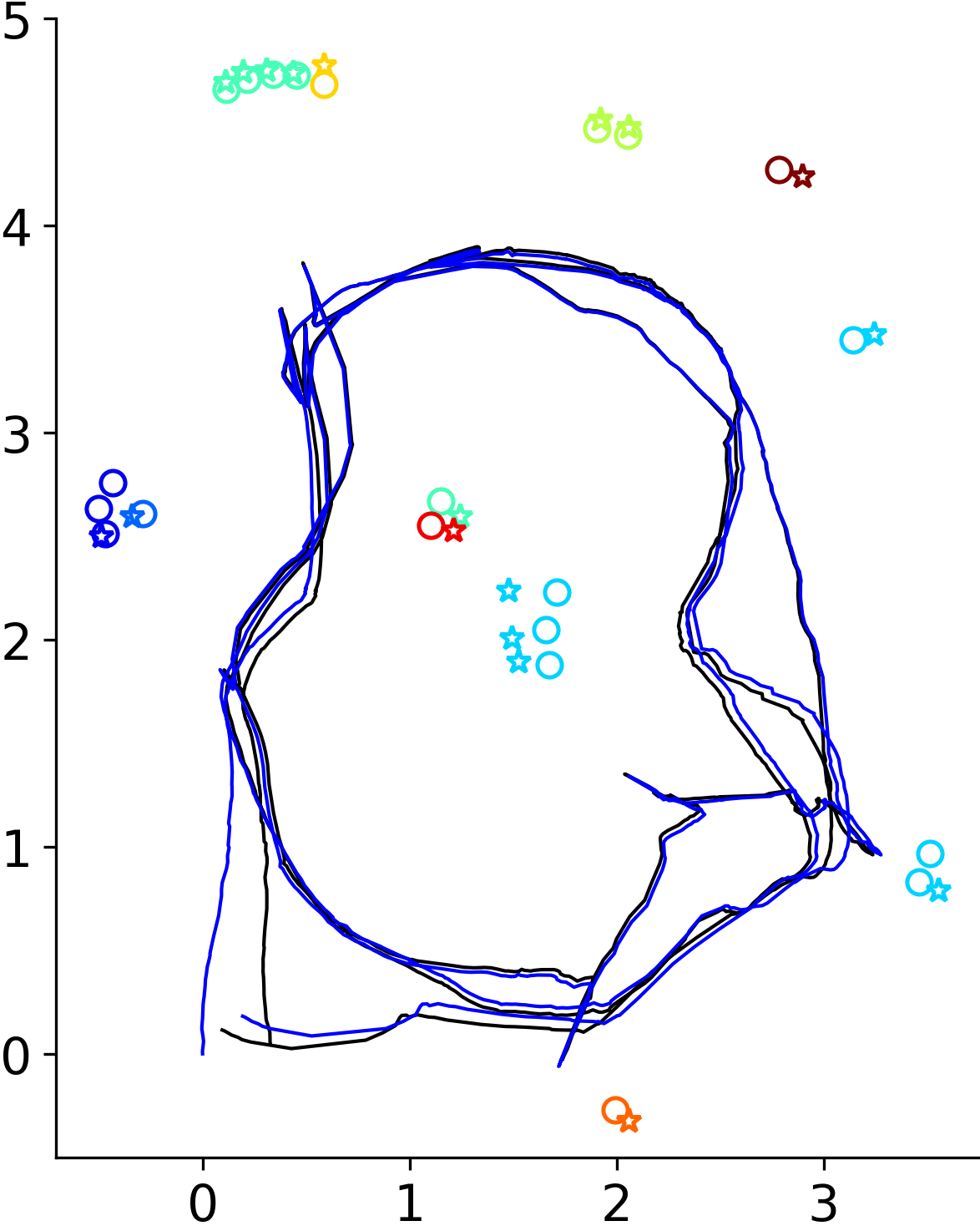}
        \caption{\algNameDNC}
    \end{subfigure}
    \begin{subfigure}{0.19\linewidth}
        \includegraphics[width=1.0\columnwidth]{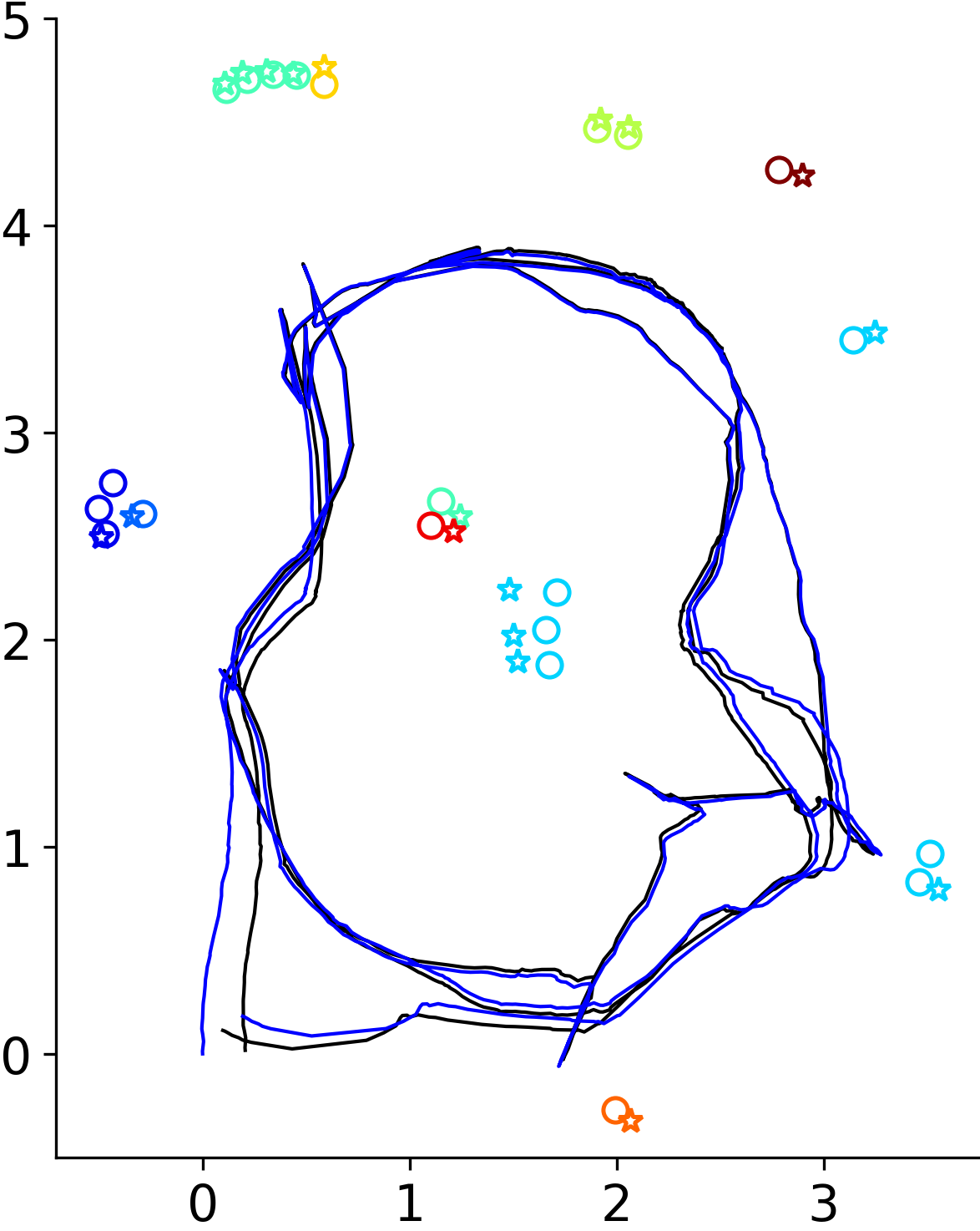}
        \caption{\algNameLL}
    \end{subfigure}
\caption{Visualization of the kitchen results. Circles indicate pseudo-ground-truth landmark positions (which are slightly offset from the actual landmark positions). Crosses represent landmark measurements. Stars indicate estimated landmark positions. The colors of landmarks denote their class labels (if applicable). The blue and black trajectories are the pseudo-ground-truth and estimated trajectories, respectively. Top: Diff Class. Bottom: Same Class.}
\label{fig: kitchen}
\end{figure*}

\subsection{Pool Dataset with Real-Valued Feature Vectors} \label{sec: pool}
The kitchen dataset evaluates our algorithm with one-hot feature vectors from a closed-set object detector, where landmark classes are predefined and semantically distinct by construction. To evaluate our algorithms under the more challenging open-set setting, where object classes are unknown and real-valued feature vectors are required, we use the indoor pool dataset from~\cite{singh2024opti}. This dataset was collected with a BlueROV2 equipped with a monocular camera and an Oculus m1200d multibeam sonar while navigating an indoor pool that contains diverse objects (e.g., lobster cages, seaweed, tires, etc.) (Fig. \ref{fig: marine_env}). To represent the semantic identities of the objects, DINO features~\cite{caron2021emerging}, extracted via unsupervised segmentation, are used instead of one-hot feature vectors, yielding real-valued descriptors for each detected object. Since ground-truth trajectory data is not available for this dataset, we evaluate the results through qualitative visualization of the estimated maps.
\begin{figure}[t]
\vspace{1.0em} % adjust margin
\centering
    \begin{subfigure}{1.0\linewidth}
        \centering
        \includegraphics[width=1.0\linewidth]{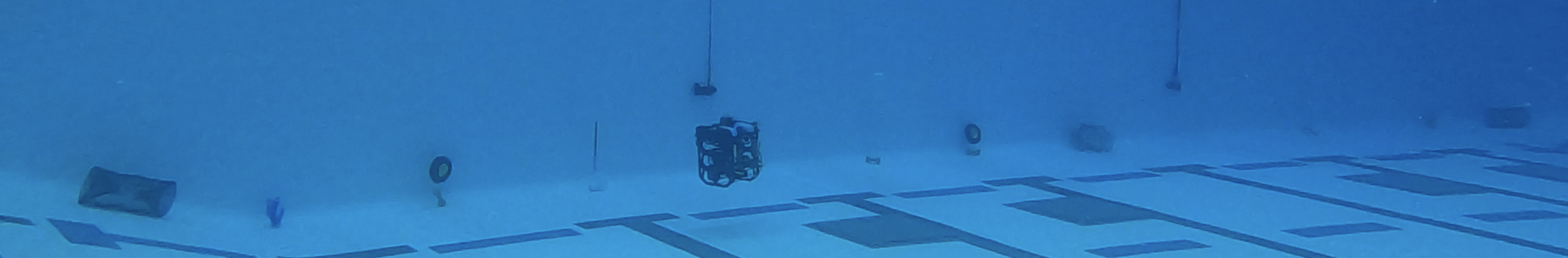}
        \caption{Pool dataset setup}
        \label{fig: marine_env}
    \end{subfigure}
    \begin{subfigure}{1.0\linewidth}
        \centering
        \scalebox{0.48}{\input{includes/figures/marine/paper_marine_init_annotated}}
        \caption{Raw data of odometry and landmark measurements}
        \label{fig: marine_init}
    \end{subfigure}
    \par\smallskip
    \begin{subfigure}{1.0\linewidth}
        \centering
        \includegraphics[width=1.0\linewidth]{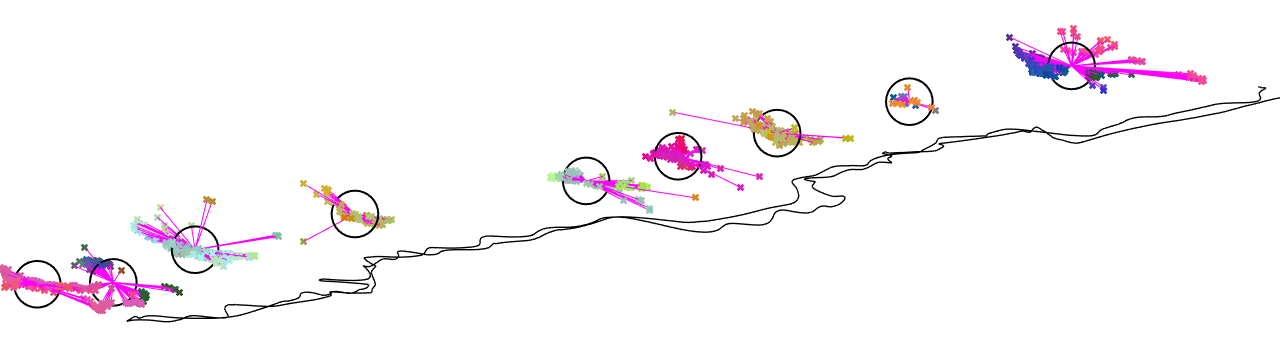}
        \caption{\algName\ \cite{zhang2023data}}
        \label{fig: marine_kslam}
    \end{subfigure}
    \par\smallskip
    \begin{subfigure}{1.0\linewidth}
        \centering
        \includegraphics[width=1.0\linewidth]{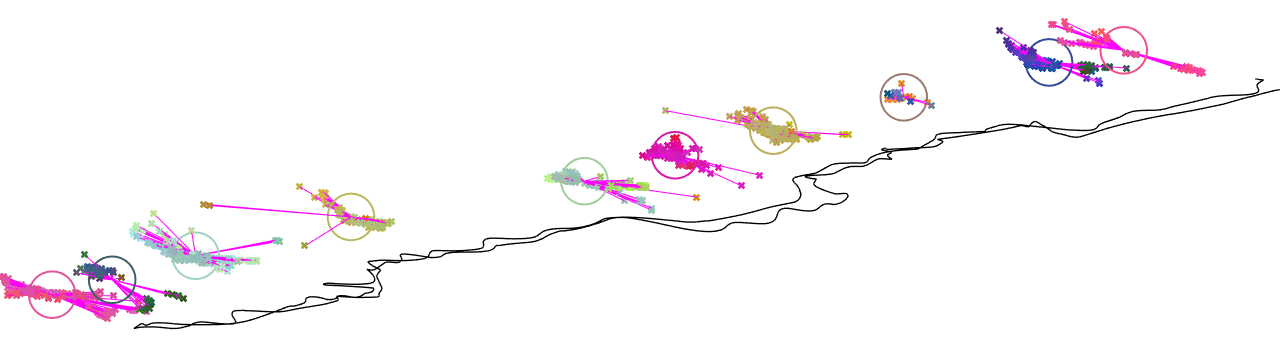}
        \caption{\algNameLL}
        \label{fig: marine_illskslam}
    \end{subfigure}
\caption{Pool dataset results. The high-dimensional feature-vector measurements are reduced to three-dimensional vectors using PCA and mapped to RGB colors for visualization. The raw data has a few serious errors indicated by the red circles in Fig. \ref{fig: marine_init}. In the \algName\ and \algNameLL\ results, the magenta lines represent the estimated data associations, and the colored circles indicate the estimated landmark positions and semantics.}
\label{fig: marine}
\vspace{-0.7em} % adjust margin
\end{figure}

In Fig. \ref{fig: marine}, our proposed \algNameLL\ is compared with the batch-processing \algName\ \cite{zhang2023data}, but readers can also refer to \cite{singh2024opti} for additional comparison with their results. The heuristic in Section \ref{sec: optimcond} is used to set $\beta$ since the landmarks in this dataset are well-separated.

In Fig. \ref{fig: marine_init}, the raw data contains several errors, particularly for small and hard-to-see objects such as the hand net and the pipes. Both \algName\ and \algNameLL\ merge these erroneous measurements into nearby landmarks. We hypothesize that the outer algorithm \eqref{eq: outer} decides not to introduce more landmarks for these sporadic outliers because the objective value reduction in $\fslamstar{K}$ for doing so does not outweigh the penalty $\beta K$ in \eqref{eq: outer}; thus, the inner algorithm \eqref{eq: seminner} is forced to assign them incorrectly to nearby landmarks.

Owing to the noisy depth measurements from the sonar and accumulated odometry drift over time, when the robot loops back and re-observes a previous landmark (e.g., the seaweed and cage on the right side in Fig. \ref{fig: marine_init}), its measurements projected from the raw odometry can be quite far away from its measurements in the first pass, making them difficult to associate correctly. In Fig. \ref{fig: marine_kslam}, \algName\ erroneously merges the cage and seaweed together and makes similar mistakes elsewhere (e.g., between the cage and seaweed on the left side). However, \algNameLL\ correctly associates the measurements of these objects by leveraging the semi-incremental scheme (Section \ref{sec: inc}) to reduce drift and semantic feature vectors to better distinguish between objects. We observe that the semantic information is critical here for data association, as the positional measurements are too noisy. % A relatively large semantic weight ($\frac{\sigma}{\sigma_s}w_s$) is used in this case, and otherwise, the result would degrade back to the \algName\ result.

\section{Conclusion}
We extended the DAF-SLAM framework \cite{zhang2023data} and developed the \algNameDNC\ and \algNameLL\ algorithms, with \algNameLL\ performing slightly better than \algNameDNC. The extended framework and its algorithms additionally process and estimate semantic vectors of object landmarks, operate in a semi-incremental manner to reduce drift, and provide principled guidelines and heuristics for setting a key parameter. They outperform the original \algName\ algorithm within the DAF-SLAM framework as well as several strong baselines on both synthetic and real-world datasets. An important future direction is handling both positional (e.g., spurious depths) and semantic (e.g., false positives) outlier measurements, potentially by modeling outliers in the generative models and using robust optimization objectives. Another future direction is incorporating shape estimation to enable a synergy between shape estimation and data association, and move beyond point-landmark representations.
% Outlier handling (outlier positional measurements), pattern for low ldmk inter-spacing (done), shape estimation to improve da

\bibliographystyle{IEEEtran}
\bibliography{isDAF-SLAM}

\newpage
\clearpage
\alt{\appendix}{\section*{\textbf{Technical Report on Additional Technical Details}}}
\section*{Assumptions in Deriving the Upper Bound for $\beta_\text{lb}$}
Several assumptions used to reach \eqref{eq: betalbub} are briefly touched on in Section \ref{sec: optimcond}. We make them more explicit in this section.

First, $K^\star$ is defined to be the ground-truth number of landmarks. However, the optimal configuration $\fslamstar{K^\star}$ may not contain all ground-truth associations if our objective function \eqref{eq: inner} does not model the real world accurately. Assuming $\fslamstar{K^\star}$ contains all correct associations makes our reasoning much easier because we no longer need to consider the case where $\fslamstar{K^\star}$ already includes mis-associations. The validity of this assumption lies in the soundness of the modeling choices (i.e., Gaussian noise and the max-mixture model for data association) and is beyond the scope of this paper.

Second, the validity of \eqref{eq: betalbub} assumes that $\fslamstar{K^\star + 1}$ also does not contain incorrect associations. In other words, the extra landmark would only split the measurements associated with a landmark in $\fslamstar{K^\star}$. This assumption is consistent with our working regime of large inter-landmark spacing relative to the landmark measurement noise, because any mis-association in this regime would introduce large residuals, which are avoided by the optimization. As illustrated in Fig. \ref{fig: betalb_good}, the added landmark splits the two measurements associated with the left landmark in Fig. \ref{fig: betalb_orig}, which is possible under our assumption. However, the added landmark instead merges two measurements previously associated with two different landmarks in Fig. \ref{fig: betalb_bad}, which is not possible under our assumption. When the inter-landmark spacing is small relative to the landmark measurement noise, this assumption starts to break down (Fig. \ref{fig: betalb_except}), but this also falls outside of the working regime of our heuristic. 
\begin{figure}[h]
    \centering
    \begin{subfigure}{0.4\linewidth}
        \centering
        \includegraphics[width=1.0\linewidth]{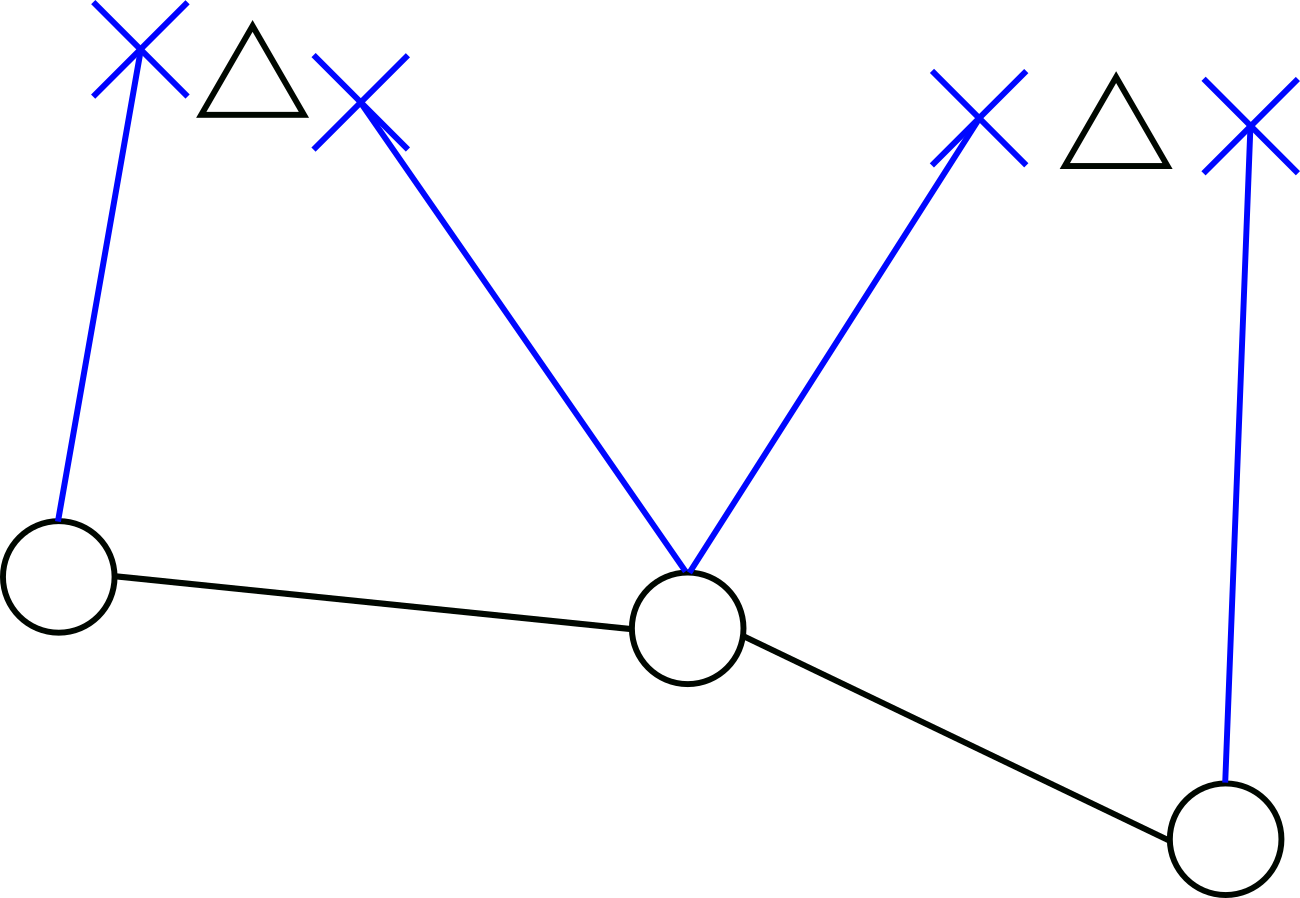}
        \caption{Configuration for $\fslamstar{K^\star}$.}
        \label{fig: betalb_orig}
    \end{subfigure}
    \par\smallskip
    \begin{subfigure}{0.4\linewidth}
        % \centering
        \includegraphics[width=1.0\linewidth]{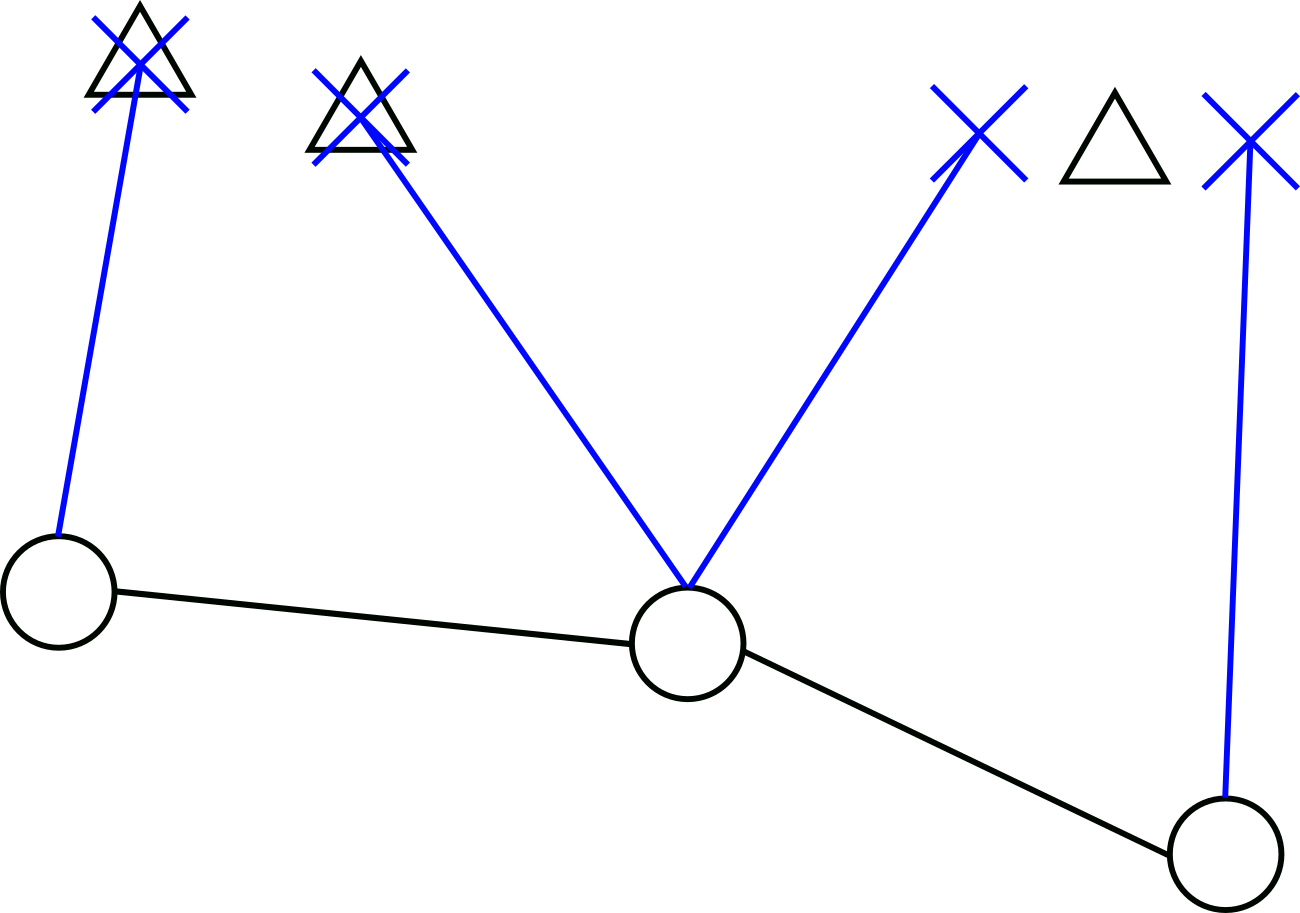}
        \caption{Possible $\fslamstar{K^\star + 1}$.}
        \label{fig: betalb_good}
    \end{subfigure}
    % \par\smallskip
    \hfill
    \begin{subfigure}{0.4\linewidth}
        % \centering
        \includegraphics[width=1.0\linewidth]{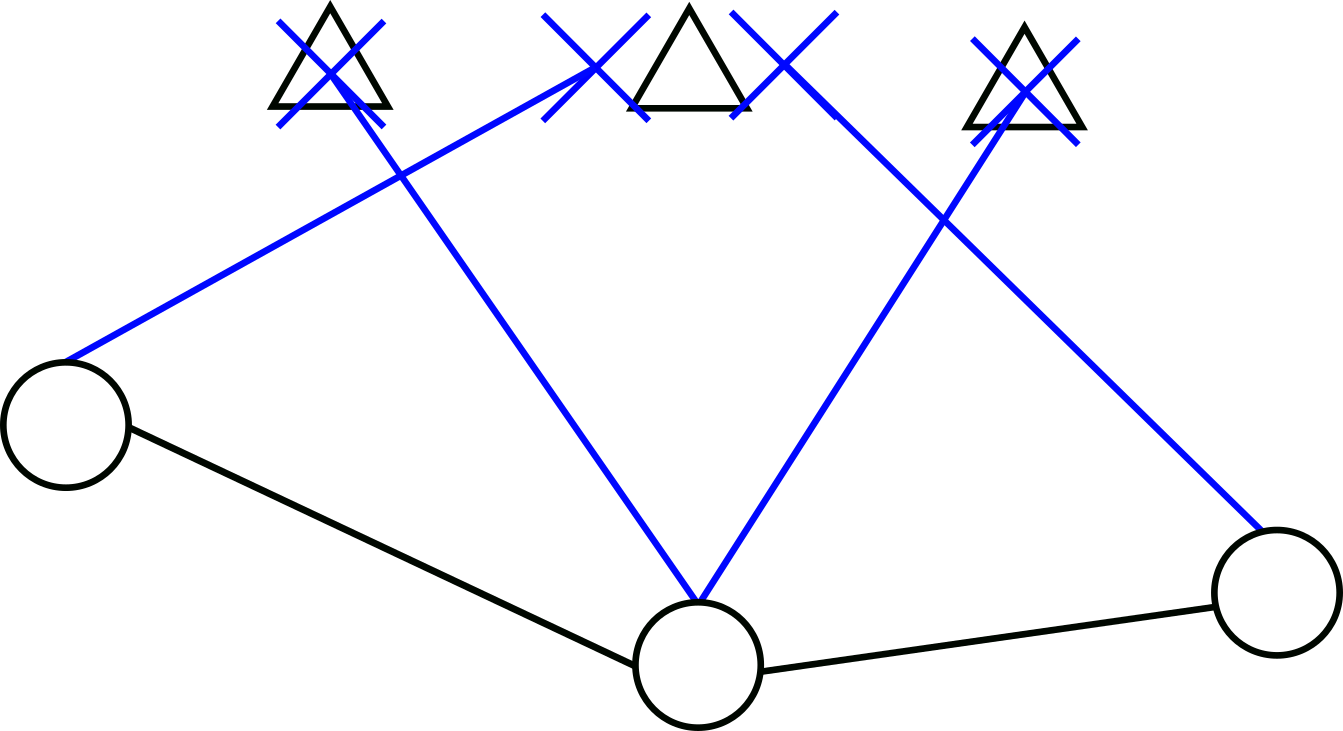}
        \caption{Impossible $\fslamstar{K^\star + 1}$.}
        \label{fig: betalb_bad}
    \end{subfigure}
\caption{Our assumption about the possible configuration for $\fslamstar{K^\star + 1}$. The circles are robot poses. The triangles are landmarks. The crosses are measurements of landmarks.}
\label{fig: betalb}
\end{figure}

\begin{figure}[h]
    \centering
    \begin{subfigure}{0.48\linewidth}
        \centering
        \includegraphics[width=1.0\linewidth]{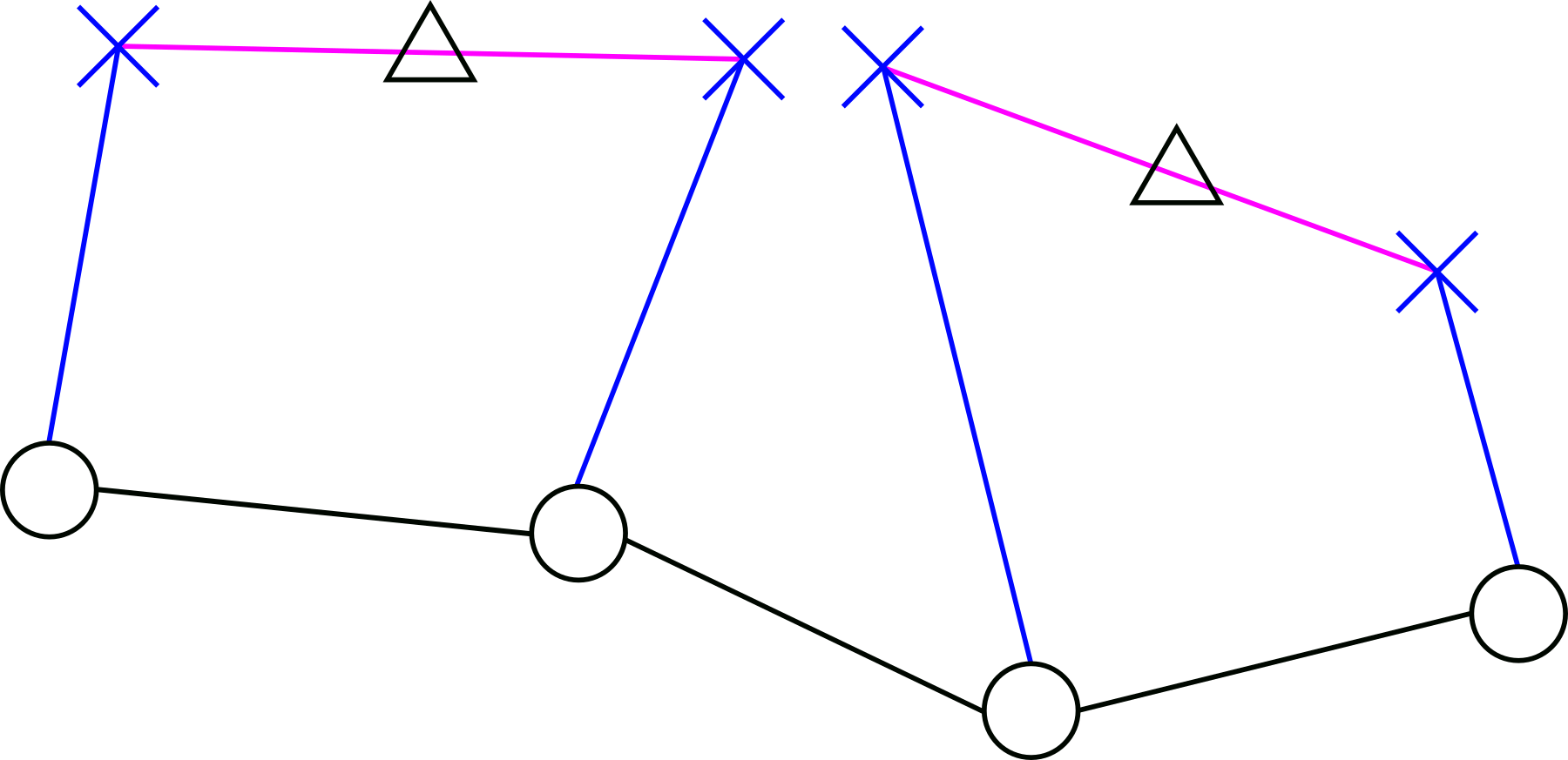}
        \caption{$\fslamstar{K^\star}$.}
        \label{fig: betalb_except1}
    \end{subfigure}
    \hfill
    % \par\smallskip
    \begin{subfigure}{0.48\linewidth}
        \centering
        \includegraphics[width=1.0\linewidth]{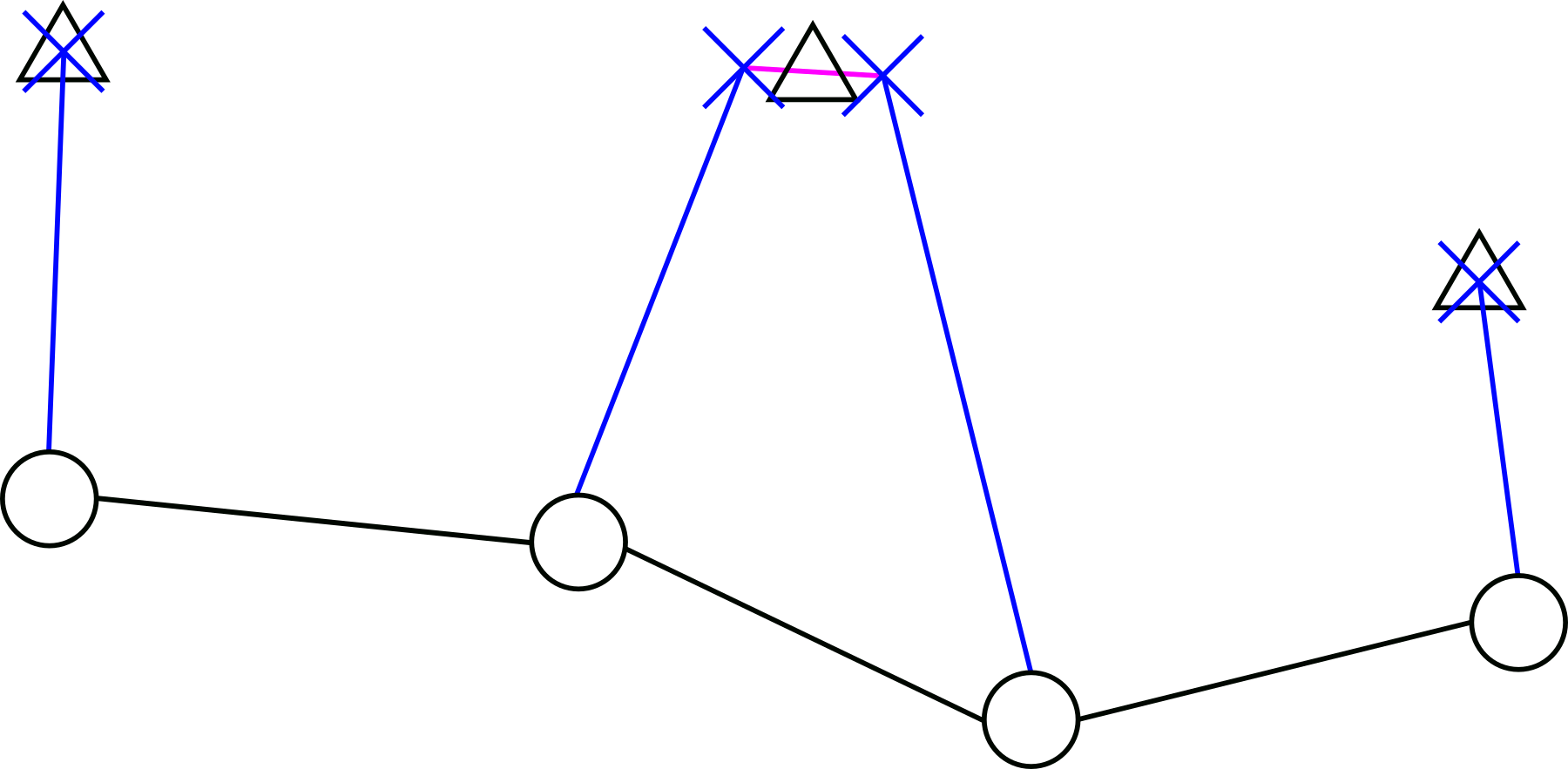}
        \caption{$\fslamstar{K^\star + 1}$.}
        \label{fig: betalb_except2}
    \end{subfigure}
\caption{An imaginary example where our assumption breaks. The magenta lines indicate data associations. The added landmark in \ref{fig: betalb_except2} merges two measurements previously associated with two different landmarks.}
\label{fig: betalb_except}
\end{figure}

Third, while not emphasized in the main paper, the odometry noise has to be moderate. In general, we find that increasing odometry noise has the opposite effect of increasing landmark measurement noise. In other words, it tends to increase $\beta_\text{lb}$ and $\beta_\text{ub}$ (Fig. \ref{fig: ksweep}) so $\beta$ is more likely to fall below $\beta_\text{lb}$, causing overestimation of $K$. We hypothesize two possible mechanisms. (i) The SLAM optimization only cares about the relative weighting (or relative noise magnitude) between the odometry term and the landmark measurement term, because scaling the objective function \eqref{eq: inner} up or down does not change the result. Therefore, increasing odometry noise can be viewed as increasing the weighting of the landmark measurement term. The non-Gaussian residual increase and reduction for $\fslamstar{K^\star - 1}$ and $\fslamstar{K^\star + 1}$ relative to $\fslamstar{K^\star}$ are thus magnified, and $\beta_\text{lb}$ and $\beta_\text{ub}$ increase. (ii) Large odometry noise creates more trajectory drift, and landmark measurements projected from the trajectory (i.e., the robot poses) have greater dispersion than that implied by the nominal landmark measurement noise. Without known data associations, the larger spread shifts the residuals at each $K$ in $\fslamstar{K}$ upward.

Fourth, when adding a landmark, besides the direct residual reduction from the landmark measurement terms, there is a nonlinear relaxation effect on the residuals of the odometry measurements and other landmark measurements, as the trajectory becomes less constrained. We assume this effect is negligible compared to the drop in the landmark measurement residuals. In practice, we use a large probability value for our $\chi^2$ heuristic so that $\beta$ is considerably larger than the right-hand side of \eqref{eq: betalbub} to compensate for the unaccounted-for relaxation effect.

Under the assumptions made above, the best possible outcome of splitting the measurements associated with one landmark across two landmarks is to drive the residuals for all the measurements associated with that landmark to zero. The exception in Fig. \ref{fig: betalb_except} is a violation, since the residuals for all the measurements associated with the two landmarks are nearly driven to zero in Fig. \ref{fig: betalb_except2}, which is even better than the outcome predicted by our upper bound \eqref{eq: betalbub}.

\section*{Low Inter-Landmark Spacing Scenario}
In Section \ref{sec: optimcond}, we mention that when the inter-landmark spacing is small relative to the landmark measurement noise, $\beta$ is set empirically because several assumptions underlying the $\beta$ heuristic no longer hold. However, useful guidelines and a reasonable starting point still exist. We discuss them in this section.

$\beta_\text{ub}$ is, in general, difficult to quantify since it depends on the environment-dependent mis-association errors. We thus focus on $\beta_\text{lb}$. It is easy to show, without additional assumptions, that
\begin{equation}
    \beta_\text{lb} \ge \max_k{\|{R^\star_{i_k}}\tran(y^\star_{j^\star_k} - t^\star_{i_k}) - \zbar_k\|_{\Sigma}^2}
    \label{eq: betalblb}
\end{equation}
because the one extra landmark can always be associated with a single measurement to drive its residual to zero. The right-hand side of \eqref{eq: betalblb} follows a $\chi^2$ distribution with degrees of freedom $d$. We can use the inverse $\chi^2$ distribution and a large probability value to compute a conservative value for the right-hand side of \eqref{eq: betalblb}, and $\beta$ needs to be greater than this value according to \eqref{eq: betacond}.

As discussed in Section \ref{sec: optimcond}, $\beta_\text{ub}$ is small in the regime of low inter-landmark spacing, and thus, we would like a value smaller than the right-hand side of \eqref{eq: betalbub} to set $\beta$. In this regime, $\fslamstar{K^\star + 1}$ may contain mis-associations such as the scenario in Fig. \ref{fig: betalb_except}. However, it is still useful to consider the case where there are no mis-associations and the added landmark only splits the measurements associated with a landmark in $\fslamstar{K^\star}$, which is essentially a restriction of the solution space of $\min{\fslam(K^\star + 1)}$, and the resulting objective value will be larger than $\fslamstar{K^\star + 1}$. While \eqref{eq: betalbub} gives an upper bound on the residual reduction for splitting the measurements, the goal here is to quantify the reduction more exactly to set a smaller $\beta$. We thus model splitting the measurements of a landmark as splitting Gaussian samples. When the number of samples is large (i.e., $|\mathcal{M}(j^\star_k)| \to \infty$), the maximum reduction in the residuals (i.e., analogous to $\fslamstar{K^\star} - \fslamstar{K^\star + 1}$) is approximately $\frac{2|\mathcal{M}(j^\star_k)|}{\pi}$, which can be a starting point for tuning $\beta$. We leave the derivations in the final section.

We summarize the findings here: (i) $\beta$ should be chosen to be greater than the right-hand side of \eqref{eq: betalblb}. (ii) In the modeling above, we have restricted the solution space, and a lower $\fslamstar{K^\star + 1}$ may actually be achievable. In other words, $\beta_\text{lb} = \fslamstar{K^\star} - \fslamstar{K^\star + 1}$ may actually be greater than $\frac{2|\mathcal{M}(j^\star_k)|}{\pi}$, and setting $\beta$ to this value may cause $\beta$ to fall below $\beta_\text{lb}$. (iii) The relaxation effect of adding a landmark on the odometry terms and other measurement terms is not considered. In principle, the relaxation effect would make $\fslamstar{K^\star + 1}$ even smaller, which may cause $\beta = \frac{2|\mathcal{M}(j^\star_k)|}{\pi}$ to fall below $\beta_\text{lb}$. (iv) The result assumes a large number of samples (measurements of a landmark). This is typically true with a high-frame-rate sensor. When this assumption is violated, a lower $\fslamstar{K^\star + 1}$ may be achievable, and $\beta = \frac{2|\mathcal{M}(j^\star_k)|}{\pi}$ may fall below $\beta_\text{lb}$.

\section*{Applicable Ranges of the $\beta$ Heuristics}
It is interesting to examine what the inverse $\chi^2$ heuristic in Section \ref{sec: optimcond} and setting $\beta = \frac{2|\mathcal{M}(j^\star_k)|}{\pi}$ lead to under different inter-landmark spacing (relative to the landmark measurement noise). To this end, we first summarize the properties of our datasets (Section \ref{sec: exp}) in Tables \ref{tab: datasets} and \ref{tab: datasets_more}. The kitchen datasets (i.e., Diff Class and Same Class) clearly have their measurement standard deviation $\sigma$ on the same order of magnitude as the mean and minimum inter-landmark spacing, whereas the other datasets have $\sigma$ orders of magnitude smaller than the spacing. Accordingly, we use $\beta = \frac{2|\mathcal{M}(j^\star_k)|}{\pi}$ for the kitchen datasets (see the actual parameters used in Table \ref{tab: algparams}).
\begin{table*}[h]
\begin{center}
\begin{threeparttable}[b]
% \small % caption font size
\caption{Dataset properties to the best of our estimates.}
\label{tab: datasets}
\begin{tabular}{|l| c c c c c c|}
    \hline
    Dataset & $N$ & $K_\text{gt}$ & $M$ & $|\mathcal{M}(j)|_\text{avg}$ & $\sigma$ & Spacing\tnote{1} (mean, min, max) \\
    \hline
    \hline
    Diff Class & 3719 & 21 & 6552 & 312 & 0.1 & 0.38, 0.10, 1.44 \\
    \hline
    Same Class & 4009 & 21 & 6872 & 327 & 0.1 & 0.31, 0.11, 1.88 \\
    \hline
    Pool & 2335 & 12 (10)\tnote{2} & 1104 & 92 (110)\tnote{2} & 0.3, 0.3, 0.6\tnote{3} & 1.74, 1.17, 2.94 \\
    \hline
    Grid 3D & 216 & 43 & 430 & 10 & 0.05 & 1.06, 0.49, 2.00\tnote{4} \\
    \hline
    Grid 2D & 500 & 100 & 1000 & 10 & 0.05 & 1.33, 0.18, 3.51\tnote{4} \\
    \hline
\end{tabular}
\begin{tablenotes}
    \item [1] Defined as the distance to the closest landmark for every landmark. \\
    \item [2] Two landmarks are almost entirely missing in the measurements (Fig. \ref{fig: marine_init}). \\
    \item [3] In the order of x, y and z axes in the robot camera frame. \\
    \item [4] Example numbers from one simulation instance.
\end{tablenotes}
\end{threeparttable}
\end{center}
\end{table*}

\begin{table*}[h]
\begin{center}
\begin{threeparttable}[b]
% \small % caption font size
\caption{Additional dataset properties to the best of our estimates.}
\label{tab: datasets_more}
\begin{tabular}{|l| c c c c|}
    \hline
    Dataset & Avg. Odom. Std. x y z r p y ($\times 10^{-3}$) & Trajectory Size\tnote{1} & Mislabeling Rate\tnote{2} & Semantic Std. \\
    \hline
    \hline
    Diff Class & 2.5, 1.6, 1.8, 3.6, 4.7, 2.3 & 3.6 $\times$ 4.6 & - & - \\
    \hline
    Same Class & 2.6, 1.7, 1.8, 3.6, 4.7, 2.2 & 3.4 $\times$ 4.6 & - & - \\
    \hline
    Pool & 1.0, 1.0, 1.0, 1.0, 1.0, 1.0 & 20.0 $\times$ 1.9 & - & - \\
    \hline
    Grid 3D & 50, 50, 50, 5.0, 5.0, 5.0 & 6 $\times$ 6 $\times$ 6 & 0.1 & 0.1 \\
    \hline
    Grid 2D & 50, 50, 5.0 & 20 $\times$ 25 & 0.1 & 0.1 \\
    \hline
\end{tabular}
\begin{tablenotes}
    \item [1] If the trajectory is mostly on a principal plane, this is the span of the trajectory on the plane.
    \item [2] Nominal mislabeling rate to simulate an object detector.
\end{tablenotes}
\end{threeparttable}
\end{center}
\end{table*}

We also perform a set of small-scale experiments for simple validation in Table \ref{tab: workingrange}. In Experiments 1 -- 3, we see that $\beta$ has to scale with the number of measurements per landmark ($|\mathcal{M}(j)|$) to estimate $K$ correctly. The $\frac{2 |\mathcal{M}(j)|}{\pi}$ heuristic consistently sets $\beta$ too small, leading to overestimation of $K$ in Experiments 4 -- 6, which is expected because our assumptions used to derive the heuristic tend to underestimate what $\beta$ should ideally be (as explained in the section ``Low Inter-Landmark Spacing Scenario").

In Experiments 7 -- 14, as $\sigma$ approaches the same order of magnitude as the mean inter-landmark spacing, the inverse $\chi^2$ heuristic starts to underestimate $K$ because of the drop in $\beta_\text{ub}$ (Section \ref{sec: optimcond}), while the $\frac{2 |\mathcal{M}(j)|}{\pi}$ heuristic begins to estimate $K$ correctly. In Experiments 10 and 14, $\sigma = 1.5$ is considered extremely large, but the $\frac{2 |\mathcal{M}(j)|}{\pi}$ heuristic manages to correctly estimate $K$.
\begin{table*}[h]
\begin{center}
\begin{threeparttable}[b]
% \small % caption font size
\caption{Results using the two $\beta$ heuristics on a small-scale 3 $\times$ 2 $\times$ 3 simulated 3D Grid dataset (18 poses).}
\label{tab: workingrange}
\begin{tabular}{|l|c c c c c c c|}
    \hline
    Exp. & $K_\text{gt}$ & $|\mathcal{M}(j)|$ & $\sigma$ & Spacing\tnote{1} (mean, min, max) & Heuristic & $\beta$ & $K_\text{est}$ \\
    \hline
    \hline
    1 & 3 & 100 & 0.05 & 1.26, 1.16, 1.47 & $(\chi^2)^{-1}(0.999, 300)$ & 381.4 & 3 (0) \\
    \hline
    2 & 3 & 500 & 0.05 & 0.94, 0.43, 1.96 & $(\chi^2)^{-1}(0.999, 1500)$ & 1675 & 3 (0) \\
    \hline
    3 & 3 & 800 & 0.05 & 1.03, 0.91, 1.26 & $(\chi^2)^{-1}(0.999, 2400)$ & 2620 & 3 (0) \\
    \hline
    4 & 3 & 100 & 0.05 & 1.26, 1.16, 1.47 & $\frac{2 \times 100}{\pi}$ & 63.7 & 6 (+3) \\
    \hline
    5 & 3 & 500 & 0.05 & 0.94, 0.43, 1.96 & $\frac{2 \times 500}{\pi}$ & 318.3 & 4 (+1) \\
    \hline
    6 & 3 & 800 & 0.05 & 1.03, 0.91, 1.26 & $\frac{2 \times 800}{\pi}$ & 509.3 &  4 (+1) \\
    \hline
    \hline
    7 & 10 & 500 & 0.1 & 0.57, 0.22, 1.18 & $(\chi^2)^{-1}(0.999, 1500)$ & 1675 & 9 (-1) \\
    \hline
    8 & 10 & 500 & 0.5 & 0.57, 0.24, 1.22 & $(\chi^2)^{-1}(0.999, 1500)$ & 1675 & 4 (-6) \\
    \hline
    9 & 10 & 500 & 0.8 & 0.57, 0.25, 1.24 & $(\chi^2)^{-1}(0.999, 1500)$ & 1675 & 4 (-6) \\
    \hline
    10 & 10 & 500 & 1.5 & 0.48, 0.26, 1.11 & $(\chi^2)^{-1}(0.999, 1500)$ & 1675 & 3 (-7) \\
    \hline
    11 & 10 & 500 & 0.1 & 0.57, 0.22, 1.18 & $\frac{2 \times 500}{\pi}$ & 318.3 & 18 (+8) \\
    \hline
    12 & 10 & 500 & 0.5 & 0.57, 0.24, 1.22 & $\frac{2 \times 500}{\pi}$ & 318.3 & 14 (+4) \\
    \hline
    13 & 10 & 500 & 0.8 & 0.57, 0.25, 1.24 & $\frac{2 \times 500}{\pi}$ & 318.3 & 11 (+1) \\
    \hline
    14 & 10 & 500 & 1.5 & 0.48, 0.26, 1.11 & $\frac{2 \times 500}{\pi}$ & 318.3 & 10 (0) \\
    \hline
\end{tabular}
\begin{tablenotes}
    \item [1] Distance to the closest landmark for every landmark, calculated from the estimated landmark positions given the ground truth data associations, to reflect the configuration of $\fslamstar{K^\star}$.
\end{tablenotes}
\end{threeparttable}
\end{center}
\end{table*}

\section*{Estimating the Number of Measurements per Landmark}
In Section \ref{sec: optimcond}, to compute a heuristic for setting $\beta$, we need to estimate the number of measurements per landmark. We consider a forward-moving model (Fig. \ref{fig: forward_motion}) for the kitchen dataset (Section \ref{sec: kitchen}) and a sideward-moving model (Fig. \ref{fig: sideward_motion}) for the pool dataset (Section \ref{sec: pool}). These are simple 2D models used to simplify the calculations. %but they can be easily extended to 3D conic or pyramid-shaped sensor frustum models.

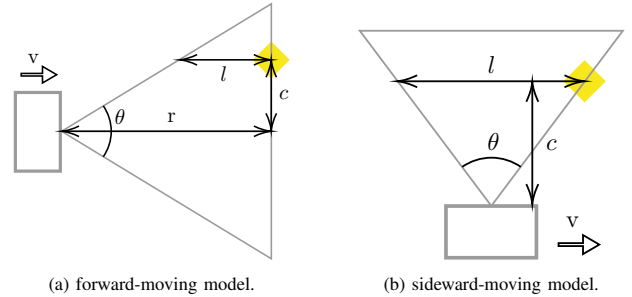
\begin{figure}[h]
    \centering
    \begin{subfigure}{0.48\linewidth}
        \centering
        \scalebox{0.8}{\tikzset{every picture/.style={line width=0.75pt}} %set default line width to 0.75pt        

\begin{tikzpicture}[x=0.75pt,y=0.75pt,yscale=-1,xscale=1]
%uncomment if require: \path (0,241); %set diagram left start at 0, and has height of 241

%Straight Lines [id:da3577954250780693] 
\draw [color={rgb, 255:red, 155; green, 155; blue, 155 }  ,draw opacity=1 ]   (63,150.36) -- (194.57,150.36) ;
%Shape: Diamond [id:dp6001602912040225] 
\draw  [color={rgb, 255:red, 248; green, 231; blue, 28 }  ,draw opacity=1 ][fill={rgb, 255:red, 248; green, 231; blue, 28 }  ,fill opacity=1 ] (194.53,94.67) -- (205.39,105.53) -- (194.53,116.39) -- (183.67,105.53) -- cycle ;
%Shape: Rectangle [id:dp23588132914443571] 
\draw  [color={rgb, 255:red, 155; green, 155; blue, 155 }  ,draw opacity=1 ][line width=1.5]  (62.2,126) -- (62.2,175.35) -- (34,175.35) -- (34,126) -- cycle ;
%Shape: Triangle [id:dp7241282762785266] 
\draw  [color={rgb, 255:red, 155; green, 155; blue, 155 }  ,draw opacity=1 ] (63,150.36) -- (194.58,70.39) -- (194.58,230.33) -- cycle ;
%Right Arrow [id:dp009770301075520793] 
\draw  [color={rgb, 255:red, 0; green, 0; blue, 0 }  ,draw opacity=1 ][fill={rgb, 255:red, 255; green, 255; blue, 255 }  ,fill opacity=1 ][line width=0.75]  (37.67,114.06) -- (51.02,114.06) -- (51.02,109.67) -- (59.91,114.86) -- (51.02,120.06) -- (51.02,115.67) -- (37.67,115.67) -- cycle ;
%Straight Lines [id:da7406571291239548] 
\draw [color={rgb, 255:red, 0; green, 0; blue, 0 }  ,draw opacity=1 ]   (139.08,105.53) -- (192.53,105.53) ;
\draw [shift={(194.53,105.53)}, rotate = 180] [color={rgb, 255:red, 0; green, 0; blue, 0 }  ,draw opacity=1 ][line width=0.75]    (10.93,-3.29) .. controls (6.95,-1.4) and (3.31,-0.3) .. (0,0) .. controls (3.31,0.3) and (6.95,1.4) .. (10.93,3.29)   ;
\draw [shift={(137.08,105.53)}, rotate = 0] [color={rgb, 255:red, 0; green, 0; blue, 0 }  ,draw opacity=1 ][line width=0.75]    (10.93,-3.29) .. controls (6.95,-1.4) and (3.31,-0.3) .. (0,0) .. controls (3.31,0.3) and (6.95,1.4) .. (10.93,3.29)   ;
%Shape: Arc [id:dp07164211643495355] 
\draw  [draw opacity=0] (89.15,134.4) .. controls (92.32,139.02) and (94.15,144.49) .. (94.15,150.36) .. controls (94.15,156.23) and (92.32,161.7) .. (89.15,166.32) -- (61.5,150.36) -- cycle ; \draw   (89.15,134.4) .. controls (92.32,139.02) and (94.15,144.49) .. (94.15,150.36) .. controls (94.15,156.23) and (92.32,161.7) .. (89.15,166.32) ;  
%Straight Lines [id:da14124474848009094] 
\draw [color={rgb, 255:red, 0; green, 0; blue, 0 }  ,draw opacity=1 ]   (194.53,107.53) -- (194.53,148.03) ;
\draw [shift={(194.53,150.03)}, rotate = 270] [color={rgb, 255:red, 0; green, 0; blue, 0 }  ,draw opacity=1 ][line width=0.75]    (10.93,-3.29) .. controls (6.95,-1.4) and (3.31,-0.3) .. (0,0) .. controls (3.31,0.3) and (6.95,1.4) .. (10.93,3.29)   ;
\draw [shift={(194.53,105.53)}, rotate = 90] [color={rgb, 255:red, 0; green, 0; blue, 0 }  ,draw opacity=1 ][line width=0.75]    (10.93,-3.29) .. controls (6.95,-1.4) and (3.31,-0.3) .. (0,0) .. controls (3.31,0.3) and (6.95,1.4) .. (10.93,3.29)   ;
%Straight Lines [id:da0197994515098594] 
\draw [color={rgb, 255:red, 0; green, 0; blue, 0 }  ,draw opacity=1 ]   (65,150.36) -- (192.57,150.36) ;
\draw [shift={(194.57,150.36)}, rotate = 180] [color={rgb, 255:red, 0; green, 0; blue, 0 }  ,draw opacity=1 ][line width=0.75]    (10.93,-3.29) .. controls (6.95,-1.4) and (3.31,-0.3) .. (0,0) .. controls (3.31,0.3) and (6.95,1.4) .. (10.93,3.29)   ;
\draw [shift={(63,150.36)}, rotate = 0] [color={rgb, 255:red, 0; green, 0; blue, 0 }  ,draw opacity=1 ][line width=0.75]    (10.93,-3.29) .. controls (6.95,-1.4) and (3.31,-0.3) .. (0,0) .. controls (3.31,0.3) and (6.95,1.4) .. (10.93,3.29)   ;

% Text Node
\draw (95,135) node [anchor=north west][inner sep=0.75pt]    {$\mathrm{\theta }$};
% Text Node
\draw (128,137) node [anchor=north west][inner sep=0.75pt]    {$\mathrm{r}$};
% Text Node
\draw (40.25,100) node [anchor=north west][inner sep=0.75pt]  [color={rgb, 255:red, 0; green, 0; blue, 0 }  ,opacity=1 ]  {$\mathrm{v}$};
% Text Node
\draw (161.25,108) node [anchor=north west][inner sep=0.75pt]    {$l$};
% Text Node
\draw (198,123) node [anchor=north west][inner sep=0.75pt]    {$c$};

\end{tikzpicture}}
        \caption{forward-moving model.}
        \label{fig: forward_motion}
    \end{subfigure}
    \hfill
    % \par\smallskip
    \begin{subfigure}{0.48\linewidth}
        \centering
        \scalebox{0.92}{\tikzset{every picture/.style={line width=0.75pt}} %set default line width to 0.75pt        

\begin{tikzpicture}[x=0.75pt,y=0.75pt,yscale=-1,xscale=1]
%uncomment if require: \path (0,300); %set diagram left start at 0, and has height of 300

%Shape: Diamond [id:dp1775737232151594] 
\draw  [color={rgb, 255:red, 248; green, 231; blue, 28 }  ,draw opacity=1 ][fill={rgb, 255:red, 248; green, 231; blue, 28 }  ,fill opacity=1 ] (165.53,106.67) -- (176.39,117.53) -- (165.53,128.39) -- (154.67,117.53) -- cycle ;
%Shape: Rectangle [id:dp1076914589446788] 
\draw  [color={rgb, 255:red, 155; green, 155; blue, 155 }  ,draw opacity=1 ][line width=1.5]  (89.92,185.58) -- (139.28,185.58) -- (139.28,213.78) -- (89.92,213.78) -- cycle ;
%Shape: Triangle [id:dp272588449099198] 
\draw  [color={rgb, 255:red, 155; green, 155; blue, 155 }  ,draw opacity=1 ] (114.67,185.5) -- (42.92,89.97) -- (186.42,89.97) -- cycle ;
%Right Arrow [id:dp22738716554960725] 
\draw  [color={rgb, 255:red, 0; green, 0; blue, 0 }  ,draw opacity=1 ][fill={rgb, 255:red, 255; green, 255; blue, 255 }  ,fill opacity=1 ][line width=0.75]  (151.34,205.81) -- (164.69,205.81) -- (164.69,201.42) -- (173.58,206.61) -- (164.69,211.81) -- (164.69,207.42) -- (151.34,207.42) -- cycle ;
%Straight Lines [id:da29286790387701] 
\draw [color={rgb, 255:red, 0; green, 0; blue, 0 }  ,draw opacity=1 ]   (65.27,117.53) -- (163.53,117.53) ;
\draw [shift={(165.53,117.53)}, rotate = 180] [color={rgb, 255:red, 0; green, 0; blue, 0 }  ,draw opacity=1 ][line width=0.75]    (10.93,-3.29) .. controls (6.95,-1.4) and (3.31,-0.3) .. (0,0) .. controls (3.31,0.3) and (6.95,1.4) .. (10.93,3.29)   ;
\draw [shift={(63.27,117.53)}, rotate = 0] [color={rgb, 255:red, 0; green, 0; blue, 0 }  ,draw opacity=1 ][line width=0.75]    (10.93,-3.29) .. controls (6.95,-1.4) and (3.31,-0.3) .. (0,0) .. controls (3.31,0.3) and (6.95,1.4) .. (10.93,3.29)   ;
%Shape: Arc [id:dp35115706436058425] 
\draw  [draw opacity=0] (98.8,164.21) .. controls (103.42,161.04) and (108.89,159.21) .. (114.76,159.21) .. controls (120.63,159.21) and (126.1,161.04) .. (130.72,164.21) -- (114.76,191.86) -- cycle ; \draw   (98.8,164.21) .. controls (103.42,161.04) and (108.89,159.21) .. (114.76,159.21) .. controls (120.63,159.21) and (126.1,161.04) .. (130.72,164.21) ;  
%Straight Lines [id:da018725019773750184] 
\draw [color={rgb, 255:red, 0; green, 0; blue, 0 }  ,draw opacity=1 ]   (137.03,119.78) -- (137.03,182.55) ;
\draw [shift={(137.03,184.55)}, rotate = 270] [color={rgb, 255:red, 0; green, 0; blue, 0 }  ,draw opacity=1 ][line width=0.75]    (10.93,-3.29) .. controls (6.95,-1.4) and (3.31,-0.3) .. (0,0) .. controls (3.31,0.3) and (6.95,1.4) .. (10.93,3.29)   ;
\draw [shift={(137.03,117.78)}, rotate = 90] [color={rgb, 255:red, 0; green, 0; blue, 0 }  ,draw opacity=1 ][line width=0.75]    (10.93,-3.29) .. controls (6.95,-1.4) and (3.31,-0.3) .. (0,0) .. controls (3.31,0.3) and (6.95,1.4) .. (10.93,3.29)   ;

% Text Node
\draw (153.92,190) node [anchor=north west][inner sep=0.75pt]  [color={rgb, 255:red, 0; green, 0; blue, 0 }  ,opacity=1 ]  {$\mathrm{v}$};
% Text Node
\draw (111,144) node [anchor=north west][inner sep=0.75pt]    {$\mathrm{\theta }$};
% Text Node
\draw (111.18,102) node [anchor=north west][inner sep=0.75pt]    {$l$};
% Text Node
\draw (143,148) node [anchor=north west][inner sep=0.75pt]    {$c$};

\end{tikzpicture}}
        \caption{sideward-moving model.}
        \label{fig: sideward_motion}
    \end{subfigure}
\caption{Simple 2D models to estimate the number of measurements per landmark. The sensor is represented as a gray rectangle. The yellow diamond represents a landmark.}
\label{fig: motion_model}
\end{figure}

In Fig. \ref{fig: forward_motion}, the distance ($l$) while the landmark is in view, is 
\begin{equation}
    l = r - c\cot(\frac{\theta}{2})
\end{equation}
where $r$ is the cutoff depth, $\theta$ is the field of view, and $c$ may be computed as an average over all landmark measurements. In Fig. \ref{fig: sideward_motion}, the in-view distance ($l$) is
\begin{equation}
    l = 2c\tan(\frac{\theta}{2})
\end{equation}
where $c$ is also computed as an average over all landmark measurements, with $c$ corresponding to the depth component here. The estimated number of observations is then computed as $\text{fps}\times l / v$ where $\text{fps}$ is the sensor fps and $v$ is the robot's average speed computed from odometry measurements. These models assume that the robot passes by each landmark only once, which is generally true for the segment processing stage (Section \ref{sec: inc}). If there is a strong belief that the robot passes by each landmark multiple times, we can set a different $\beta$ for the whole-trajectory global optimization. Otherwise, we rely on the robustness of formulation \eqref{eq: outer} to correctly estimate $K$. For the segment processing, the number of observations should also not exceed the number of poses in the segment, assuming each pose can produce at most one observation of a landmark.

\section*{Algorithm Parameters}
We provide our parameter settings in Table \ref{tab: algparams}.
\begin{table*}[h]
\begin{center}
\begin{threeparttable}[b]
% \small % caption font size
\caption{Motion model parameters (calculated based on prior information or from measurements) and algorithm parameters.}
\label{tab: algparams}
\begin{tabular}{|l| c c c c c c| c c c c c c c c|}
    \hline
    Dataset & $v$ & $\theta$ & $r$ & $c$ & fps & $|\mathcal{M}(j)|_\text{est}$ & Heuristic & $\beta_\text{glb}$ & $\frac{\sigma}{\sigma_s}w_s$\tnote{1} & $N_n$ & $\beta_\text{seg}$\tnote{2} & $N_\text{inner}$\tnote{3} & $N_\text{LM}$\tnote{4} & k-means\tnote{5} \\
    \hline
    \hline
    Grid 3D & - & - & - & - & - & 10 & $(\chi^2)^{-1}(0.999, 30)$ & 60 & 0.1 & 100 & 60 & 30 & 100 & 3, 15, 0.1 \\
    \hline
    Grid 2D & - & - & - & - & - & 10 & $(\chi^2)^{-1}(0.999, 20)$ & 45 & 0.1 & 100 & 45 & 30 & 100 & 3, 15, 0.1 \\
    \hline
    Diff Cls & 0.247 & 120 & 3 & 0.642 & 20 & 213 & $\frac{2 \times 213}{\pi}$ & 136 & 0.5 & 100 & 64 & 15 & 100 & 3, 15, 0.1 \\
    \hline
    Same Cls & 0.246 & 120 & 3 & 0.698 & 20 & 211 & $\frac{2 \times 211}{\pi}$ & 134 & 0.5 & 100 & 64 & 15 & 100 & 3, 15, 0.1 \\
    \hline
    Pool & 0.186 & 80 & - & 2.62 & 3.81 & 90 & $(\chi^2)^{-1}(0.999, 270)$ & 347 & 100 & 100 & 347 & 15 & 100 & 3, 15, 0.1\\
    \hline
\end{tabular}
\begin{tablenotes}
    \item [1] Real-valued feature vectors are additionally normalized by length. 
    \item [2] Calculated using $N_n$ instead of $|\mathcal{M}(j)|_\text{est}$ if $N_n < |\mathcal{M}(j)|_\text{est}$. \\
    \item [3] Number of inner solver iterations. The $\arg\!\min$ solution across iterations (not necessarily the final iteration due to k-means++ randomness) is selected. \\
    \item [4] Number of Levenberg-Marquardt optimizer iterations for SLAM. \\
    \item [5] Number of k-means attempts (the $\arg\!\min$ result is kept), termination iteration and termination tolerance.
\end{tablenotes}
\end{threeparttable}
\end{center}
\end{table*}

\section*{Landmark Evaluation}
We provide additional details about the evaluation of landmark estimation on the synthetic datasets. The data association is unknown to our algorithms, and therefore, we need to match the estimated landmarks to the reference landmarks. The reference landmarks are obtained along with the reference trajectory, by optimizing the measurements given the ground-truth data associations. We first compute the $\mathrm{SIM}(3)$ alignment between the estimated trajectory and the reference trajectory, and apply it to the estimated landmark positions. A linear assignment problem is then solved to match the estimated landmark positions to reference landmark positions. Only $\min(K_\text{ref}, K_\text{est})$ pairs are matched, and any extra landmarks are not included in the evaluation. The estimation error is computed over all matched pairs.

\section*{Visualization of $\fslamstar{K}$}
It is interesting to plot $\fslamstar{K}$ as a function of $K$. Computing the true $\fslamstar{K}$ is combinatorial, therefore we use our inner algorithm \eqref{eq: inner} to approximate its value. We show three cases under different noise levels in Fig. \ref{fig: ksweep}. In the case of the normal noise level, we can clearly observe the elbow point at the true $K$, which is consistent with our discussions in Section \ref{sec: optimcond}. The case of high odometry noise exhibits the opposite trend compared to the case of high landmark measurement noise. We provide two possible explanations in the previous section ``Assumptions in Deriving the Upper Bound for $\beta_\text{lb}$". Additionally, the case of high odometry noise appears to be ``noisier", which reveals that the algorithm more easily returns a suboptimal solution, potentially due to the drift accumulation effect of odometry noise that does not occur with landmark measurement noise.
\begin{figure}[h]
\centering
\includegraphics[width=1.0\columnwidth]{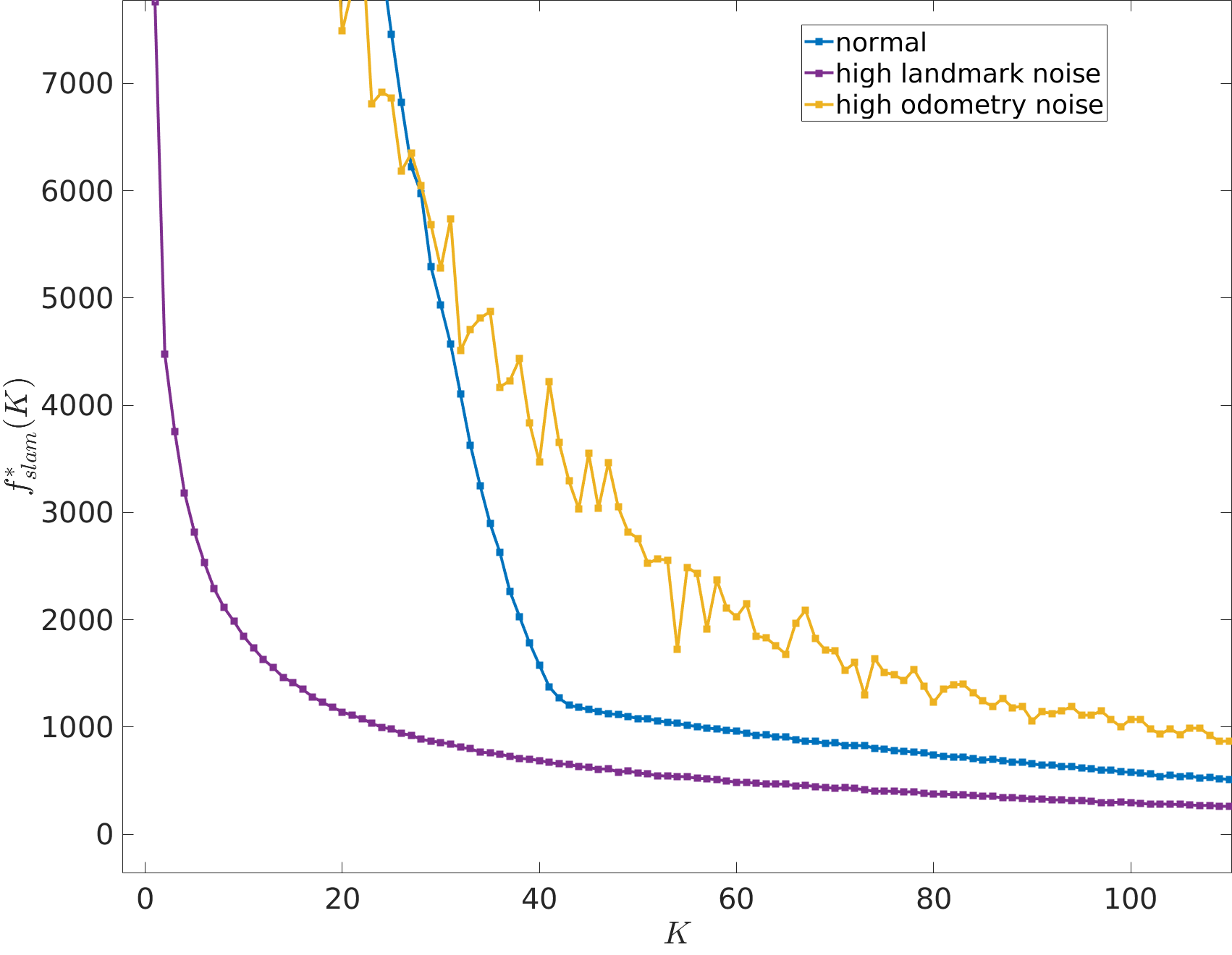}
\caption{Evaluating $\fslamstar{K}$ using the inner algorithm \eqref{eq: inner} on the synthetic 3D grid dataset of size $6 \times 6 \times 6$. The normal case uses a landmark measurement standard deviation 0.1 and odometry translation and rotation standard deviations of 0.05 and 0.005, respectively. The high landmark noise case uses a landmark measurement standard deviation 0.5, which is fairly large with respect to the map scale. The high odometry noise case uses translation and rotation standard deviations of 0.15 and 0.015, respectively. The true $K$ is 43.}
\label{fig: ksweep}
\end{figure}

There are two effects that lead to the differences between the normal case and the case of high landmark measurement noise. First, the nonlinear optimization will result in a lower SLAM objective value as the landmark measurement noise increases. We show a simple experiment in Table \ref{tab: grapherror} for a standard landmark-based SLAM problem (given data associations). The data are generated synthetically from the Gaussian noise models that exactly match the probabilistic models assumed by the SLAM objective function. Since the covariance matrices and the probabilistic models are exact, the covariance matrices perfectly normalize the measurement residuals \eqref{eq: posgen}, and thus, given that the ground truth is used as the initial values for the optimization variables, the initial errors (i.e., objective function values) are all the same. However, after optimization, due to the nonlinear interplay between the odometry term and the landmark measurement term, the optimal objective values are smaller at higher landmark measurement noise. This explains why the case of high landmark measurement noise has lower $\fslamstar{K}$ than the normal case at the true $K$ in Fig. \ref{fig: ksweep}.
\begin{table}[h]
\begin{center}
% \small % caption font size
\caption{Factor graph errors (i.e. objective function values) of a standard landmark-based SLAM problem (given data associations) at different landmark measurement noise levels. The synthetic 3D grid dataset is used. Landmark measurements are generated from the Gaussian generative model \eqref{eq: posgen}, so the covariance matrices are exact. Ground truth is used as the initial guess.}
\label{tab: grapherror}
{\setlength{\tabcolsep}{3.5pt}
\begin{tabular}{|l| c c c c c c c c|}
    \hline
    LM Std & 0.001 & 0.01 & 0.05 & 0.1 & 0.2 & 0.3 & 0.4 & 0.5 \\
    \hline
    Initial Error & \multicolumn{8}{c|}{2607} \\
    \hline
    Final Error & 1230 & 1222 & 1219 & 1203 & 1183 & 1174 & 1169 & 1167 \\
    \hline
\end{tabular}}
\end{center}
\end{table}

Second, when the landmark measurement noise is high (relative to the inter-landmark spacing), the unnormalized (i.e., not normalized by the covariance matrix) residual $||R_{i_k}\tran(y_{j_k} - t_{i_k}) - \zbar_k||^2$ of an association error may have a similar magnitude as the residual due to pure sensor noise. Therefore, even when $K$ is below the true $K$ and more association errors are introduced in computing $\fslamstar{K}$, we do not observe as sharp a transition as in the normal case in Fig. \ref{fig: ksweep}. This is where the gap between $\beta_{lb}$ and $\beta_{ub}$ becomes narrow, and our algorithm \eqref{eq: outer}, using $\beta$ set by the inverse $\chi^2$ heuristic in \eqref{eq: betalbub}, begins to underestimate $K$.

\section*{Kitchen Dataset Panoramas}
For direct visualization of the kitchen datasets, we provide panoramas of the setups in Fig. \ref{fig: kitchen_pano}.
\begin{figure*}[t]
\centering
    \begin{subfigure}{0.95\linewidth}
        \centering
        \includegraphics[width=1.0\linewidth]{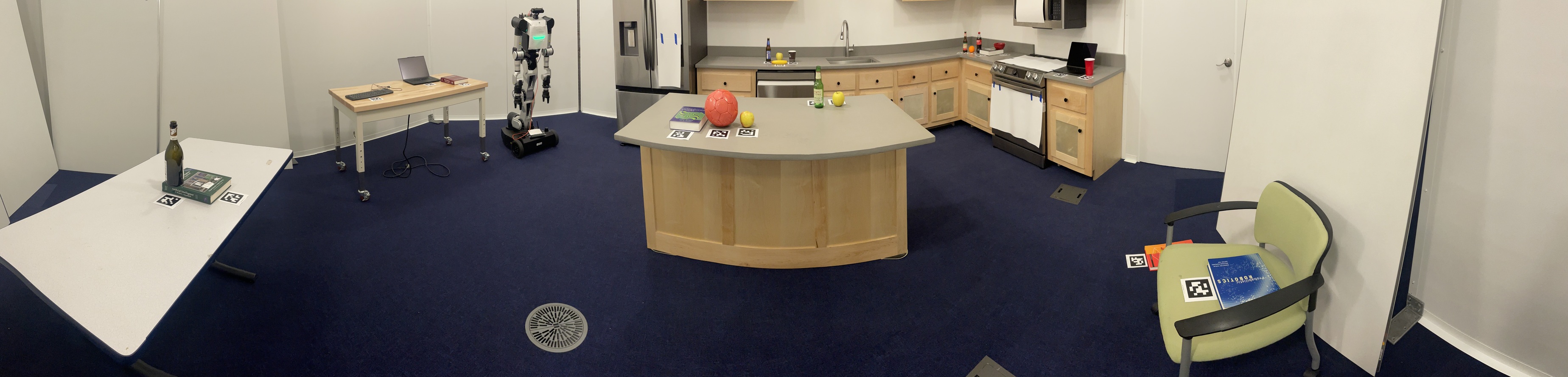}
        \caption{Kitchen - Diff Class.}
        \label{fig: kitchen_diff_pano}
    \end{subfigure}
    \par\smallskip
    \begin{subfigure}{0.95\linewidth}
        \centering
        \includegraphics[width=1.0\linewidth]{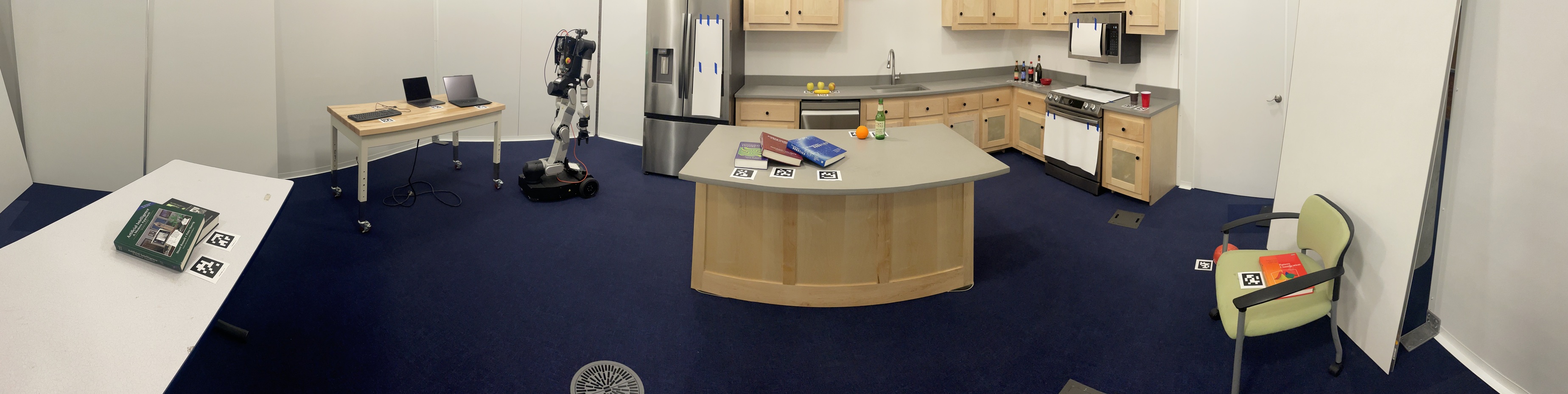}
        \caption{Kitchen - Same Class.}
        \label{fig: kitchen_same_pano}
    \end{subfigure}
\caption{Panoramas of the kitchen environment and the experiment setups. In the datasets, we have removed the laptop detections, because the laptop screens caused erroneous depth measurements.}
\label{fig: kitchen_pano}
\end{figure*}

\section*{Splitting Gaussian Samples}
\label{app: split_reduction}
Following up on the previous section ``Low Inter-Landmark Spacing Scenario", we model the addition of a landmark as splitting Gaussian samples to approximate the objective value change from $\fslamstar{K^\star}$ to the pre-optimization (non-optimal) value of $\fslam(K^\star + 1)$. Fig. \ref{fig: splitsamples} shows a schematic for illustration. For clarity, we hereafter use notation that is independent of the rest of the paper.
\begin{figure}[h]
\centering
    \scalebox{1.0}{\tikzset{every picture/.style={line width=0.75pt}} %set default line width to 0.75pt        

\begin{tikzpicture}[x=0.75pt,y=0.75pt,yscale=-1,xscale=1]
%uncomment if require: \path (0,106); %set diagram left start at 0, and has height of 106

\draw  [color={rgb, 255:red, 245; green, 166; blue, 35 }  ,draw opacity=1 ] (145.12,37.39) -- (155.73,48)(155.73,37.39) -- (145.12,48) ;
\draw  [color={rgb, 255:red, 245; green, 166; blue, 35 }  ,draw opacity=1 ] (200.39,53.45) -- (211,64.06)(211,53.45) -- (200.39,64.06) ;
\draw  [color={rgb, 255:red, 245; green, 166; blue, 35 }  ,draw opacity=1 ] (178.67,69.39) -- (189.27,80)(189.27,69.39) -- (178.67,80) ;
\draw  [color={rgb, 255:red, 245; green, 166; blue, 35 }  ,draw opacity=1 ] (138,53.06) -- (148.61,63.67)(148.61,53.06) -- (138,63.67) ;
\draw  [color={rgb, 255:red, 245; green, 166; blue, 35 }  ,draw opacity=1 ] (160.39,23.67) -- (171,34.27)(171,23.67) -- (160.39,34.27) ;
\draw  [color={rgb, 255:red, 245; green, 166; blue, 35 }  ,draw opacity=1 ] (156,58.73) -- (166.61,69.33)(166.61,58.73) -- (156,69.33) ;
\draw  [color={rgb, 255:red, 245; green, 166; blue, 35 }  ,draw opacity=1 ] (187,46.06) -- (197.61,56.67)(197.61,46.06) -- (187,56.67) ;
\draw  [color={rgb, 255:red, 245; green, 166; blue, 35 }  ,draw opacity=1 ][fill={rgb, 255:red, 74; green, 144; blue, 226 }  ,fill opacity=1 ] (180.39,24.67) -- (191,35.27)(191,24.67) -- (180.39,35.27) ;
\draw  [color={rgb, 255:red, 245; green, 166; blue, 35 }  ,draw opacity=1 ][fill={rgb, 255:red, 74; green, 144; blue, 226 }  ,fill opacity=1 ] (167.33,65.39) -- (177.94,76)(177.94,65.39) -- (167.33,76) ;
\draw  [color={rgb, 255:red, 245; green, 166; blue, 35 }  ,draw opacity=1 ][fill={rgb, 255:red, 74; green, 144; blue, 226 }  ,fill opacity=1 ] (191,62.06) -- (201.61,72.67)(201.61,62.06) -- (191,72.67) ;
%Shape: Triangle [id:dp3271464016873371] 
\draw  [draw opacity=0][fill={rgb, 255:red, 245; green, 166; blue, 35 }  ,fill opacity=1 ] (175.37,46.48) -- (182.8,58.48) -- (167.94,58.48) -- cycle ;
\draw  [color={rgb, 255:red, 245; green, 166; blue, 35 }  ,draw opacity=1 ] (193.06,33.67) -- (203.67,44.27)(203.67,33.67) -- (193.06,44.27) ;
\draw  [color={rgb, 255:red, 245; green, 166; blue, 35 }  ,draw opacity=1 ][fill={rgb, 255:red, 74; green, 144; blue, 226 }  ,fill opacity=1 ] (160.67,38.39) -- (171.27,49)(171.27,38.39) -- (160.67,49) ;
\draw  [color={rgb, 255:red, 245; green, 166; blue, 35 }  ,draw opacity=1 ] (363.12,37.39) -- (373.73,48)(373.73,37.39) -- (363.12,48) ;
\draw  [color={rgb, 255:red, 74; green, 144; blue, 226 }  ,draw opacity=1 ] (418.39,53.45) -- (429,64.06)(429,53.45) -- (418.39,64.06) ;
\draw  [color={rgb, 255:red, 74; green, 144; blue, 226 }  ,draw opacity=1 ] (396.67,69.39) -- (407.27,80)(407.27,69.39) -- (396.67,80) ;
\draw  [color={rgb, 255:red, 245; green, 166; blue, 35 }  ,draw opacity=1 ] (356,53.06) -- (366.61,63.67)(366.61,53.06) -- (356,63.67) ;
\draw  [color={rgb, 255:red, 245; green, 166; blue, 35 }  ,draw opacity=1 ] (378.39,23.67) -- (389,34.27)(389,23.67) -- (378.39,34.27) ;
\draw  [color={rgb, 255:red, 245; green, 166; blue, 35 }  ,draw opacity=1 ] (374,58.73) -- (384.61,69.33)(384.61,58.73) -- (374,69.33) ;
\draw  [color={rgb, 255:red, 74; green, 144; blue, 226 }  ,draw opacity=1 ] (405,46.06) -- (415.61,56.67)(415.61,46.06) -- (405,56.67) ;
\draw  [color={rgb, 255:red, 74; green, 144; blue, 226 }  ,draw opacity=1 ][fill={rgb, 255:red, 74; green, 144; blue, 226 }  ,fill opacity=1 ] (398.39,24.67) -- (409,35.27)(409,24.67) -- (398.39,35.27) ;
\draw  [color={rgb, 255:red, 245; green, 166; blue, 35 }  ,draw opacity=1 ][fill={rgb, 255:red, 74; green, 144; blue, 226 }  ,fill opacity=1 ] (385.33,65.39) -- (395.94,76)(395.94,65.39) -- (385.33,76) ;
\draw  [color={rgb, 255:red, 74; green, 144; blue, 226 }  ,draw opacity=1 ][fill={rgb, 255:red, 74; green, 144; blue, 226 }  ,fill opacity=1 ] (409,62.06) -- (419.61,72.67)(419.61,62.06) -- (409,72.67) ;
%Shape: Triangle [id:dp5018921661909306] 
\draw  [draw opacity=0][fill={rgb, 255:red, 245; green, 166; blue, 35 }  ,fill opacity=1 ] (378.23,46.48) -- (385.66,58.48) -- (370.8,58.48) -- cycle ;
\draw  [color={rgb, 255:red, 74; green, 144; blue, 226 }  ,draw opacity=1 ] (411.06,33.67) -- (421.67,44.27)(421.67,33.67) -- (411.06,44.27) ;
\draw  [color={rgb, 255:red, 245; green, 166; blue, 35 }  ,draw opacity=1 ][fill={rgb, 255:red, 74; green, 144; blue, 226 }  ,fill opacity=1 ] (378.67,38.39) -- (389.27,49)(389.27,38.39) -- (378.67,49) ;
%Shape: Triangle [id:dp04373179759451529] 
\draw  [draw opacity=0][fill={rgb, 255:red, 74; green, 144; blue, 226 }  ,fill opacity=1 ] (409.23,46.48) -- (416.66,58.48) -- (401.8,58.48) -- cycle ;
\draw  [color={rgb, 255:red, 245; green, 166; blue, 35 }  ,draw opacity=1 ] (255.12,37.39) -- (265.73,48)(265.73,37.39) -- (255.12,48) ;
\draw  [color={rgb, 255:red, 245; green, 166; blue, 35 }  ,draw opacity=1 ] (310.39,53.45) -- (321,64.06)(321,53.45) -- (310.39,64.06) ;
\draw  [color={rgb, 255:red, 245; green, 166; blue, 35 }  ,draw opacity=1 ] (288.67,69.39) -- (299.27,80)(299.27,69.39) -- (288.67,80) ;
\draw  [color={rgb, 255:red, 74; green, 144; blue, 226 }  ,draw opacity=1 ] (248,53.06) -- (258.61,63.67)(258.61,53.06) -- (248,63.67) ;
\draw  [color={rgb, 255:red, 245; green, 166; blue, 35 }  ,draw opacity=1 ] (270.39,23.67) -- (281,34.27)(281,23.67) -- (270.39,34.27) ;
\draw  [color={rgb, 255:red, 245; green, 166; blue, 35 }  ,draw opacity=1 ] (266,58.73) -- (276.61,69.33)(276.61,58.73) -- (266,69.33) ;
\draw  [color={rgb, 255:red, 74; green, 144; blue, 226 }  ,draw opacity=1 ] (297,46.06) -- (307.61,56.67)(307.61,46.06) -- (297,56.67) ;
\draw  [color={rgb, 255:red, 74; green, 144; blue, 226 }  ,draw opacity=1 ][fill={rgb, 255:red, 74; green, 144; blue, 226 }  ,fill opacity=1 ] (290.39,24.67) -- (301,35.27)(301,24.67) -- (290.39,35.27) ;
\draw  [color={rgb, 255:red, 74; green, 144; blue, 226 }  ,draw opacity=1 ][fill={rgb, 255:red, 74; green, 144; blue, 226 }  ,fill opacity=1 ] (277.33,65.39) -- (287.94,76)(287.94,65.39) -- (277.33,76) ;
\draw  [color={rgb, 255:red, 74; green, 144; blue, 226 }  ,draw opacity=1 ][fill={rgb, 255:red, 74; green, 144; blue, 226 }  ,fill opacity=1 ] (301,62.06) -- (311.61,72.67)(311.61,62.06) -- (301,72.67) ;
%Shape: Triangle [id:dp9350614231497729] 
\draw  [draw opacity=0][fill={rgb, 255:red, 245; green, 166; blue, 35 }  ,fill opacity=1 ] (287.37,44.48) -- (294.8,56.48) -- (279.94,56.48) -- cycle ;
\draw  [color={rgb, 255:red, 245; green, 166; blue, 35 }  ,draw opacity=1 ] (303.06,33.67) -- (313.67,44.27)(313.67,33.67) -- (303.06,44.27) ;
\draw  [color={rgb, 255:red, 74; green, 144; blue, 226 }  ,draw opacity=1 ][fill={rgb, 255:red, 74; green, 144; blue, 226 }  ,fill opacity=1 ] (270.67,38.39) -- (281.27,49)(281.27,38.39) -- (270.67,49) ;
%Shape: Triangle [id:dp7141321851969222] 
\draw  [draw opacity=0][fill={rgb, 255:red, 74; green, 144; blue, 226 }  ,fill opacity=1 ] (284.37,46.48) -- (291.8,58.48) -- (276.94,58.48) -- cycle ;

\end{tikzpicture}}
\caption{Splitting Gaussian samples to model $\fslam(K^\star + 1)$. We show two different ways of splitting.}
\label{fig: splitsamples}
\end{figure}
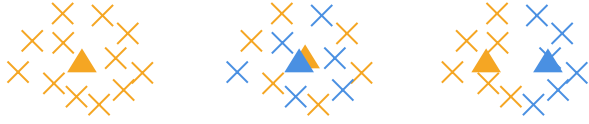

\subsection{Problem Setup}

Consider a landmark at position $\mu \in \realspace^d$ observed \(n\) times. Let $u_1,\ldots,u_n\in\realspace^d$ denote the measurements for this landmark. We assume an isotropic noise covariance. The anisotropic case can be derived from the isotropic case with additional tools such as PCA. Further, without loss of generality, we assume that the noise standard deviation is unity; otherwise, the residuals in our derivations below would be scaled by the inverse variance throughout (since SLAM optimization assumes a fixed covariance), but the conclusion in Subsection \ref{sec: halfsplit} remains unchanged due to cancellation. In short,
\[
u_1,\ldots,u_n \overset{\text{i.i.d.}}{\sim} \gauss(\mu, I_d).
\]
In this model, we use the average as our estimate,
\[
\bar u = \frac{1}{n}\sum_{i=1}^{n}u_i .
\]
The total residual (after normalizing by the variance) for all the measurements is 
\begin{equation}
\label{eq: r0}
    r_0 = \sum_{i=1}^{n}\|u_i-\bar u\|^2 ,
\end{equation}
which is analogous to the total residual for all measurements of a landmark in SLAM.

\subsection{Reduction Formula for Any Split}
Suppose the measurements are split into two nonempty groups $G_1$ and $G_2$ with sizes $n_1=|G_1|$ and $n_2=|G_2|$ ($n_1+n_2=n$). Each has its own average estimate, 
\[
\bar{u}_1 = \frac{1}{n_1}\sum_{i\in G_1}u_i,
\qquad
\bar{u}_2 = \frac{1}{n_2}\sum_{i\in G_2}u_i.
\]
The total residual after splitting is now
\begin{equation}
\label{eq: rsplit}
    r_{\rm split} = \sum_{i\in G_1}\|u_i-\bar{u}_1\|^2 + \sum_{i\in G_2}\|u_i-\bar{u}_2\|^2 .    
\end{equation}

\begin{lemma} [Residual Reduction] \label{lem: resred}
\leavevmode \\
\hspace*{\parindent}
For any split $G_1$ and $G_2$,
\[
r_0-r_{\rm split} = n_1\|\bar{u}_1-\bar u\|^2 + n_2\|\bar{u}_2-\bar u\|^2 .
\]
Equivalently,
\begin{equation}
\label{eq: rred}
    r_0-r_{\rm split} = \frac{n_1n_2}{n}\|\bar{u}_1-\bar{u}_2\|^2 .    
\end{equation}
\end{lemma}

\begin{proof}
\leavevmode \\
\hspace*{\parindent}
Expanding $\sum_{i\in G_1}\|u_i-\bar{u}_1\|^2$, we obtain
\begin{align*}
&\sum_{i\in G_1}{\|u_i\|^2} + \sum_{i\in G_1}{\|\bar{u}_1\|^2} - 2\sum_{i\in G_1}{u_i\tran \bar{u}_1}
\\
&= \sum_{i\in G_1}{\|u_i\|^2} + n_1\|\bar{u}_1\|^2 - 2n_1\|\bar{u}_1\|^2
\\
&= \sum_{i\in G_1}{\|u_i\|^2} - n_1\|\bar{u}_1\|^2 ,
\end{align*}
and similarly for $\sum_{i\in G_2}\|u_i-\bar{u}_2\|^2$. After merging terms,
\begin{align*}
r_{\rm split} &= \sum_{i=1}^{n}{\|u_i\|^2} - n_1\|\bar{u}_1\|^2 - n_2\|\bar{u}_2\|^2
\\
&= \sum_{i=1}^{n}{\|u_i-\bar{u}\|^2} - n_1\|\bar{u}_1\|^2 - n_2\|\bar{u}_2\|^2
\\
&\ \ \ - n\|\bar{u}\|^2 + 2n\|\bar{u}\|^2 % 2\sum_{i=1}^{n}{u_i\tran \bar{u}}
\\
&= r_0 + n_1(\|\bar{u}\|^2 - \|\bar{u}_1\|^2) + n_2(\|\bar{u}\|^2 - \|\bar{u}_2\|^2) .
\end{align*}
Rearranging terms,
\begin{align}
\label{eq: rred1}
    r_0 - r_{\rm split} &= n_1(\|\bar{u}_1\|^2 - \|\bar{u}\|^2) + n_2(\|\bar{u}_2\|^2 - \|\bar{u}\|^2) .
\end{align}
For the two terms in Lemma \ref{lem: resred},
\begin{align*}
    n_1\|\bar{u}_1-\bar u\|^2 &= n_1\|\bar{u}_1\|^2 + n_1\|\bar{u}\|^2 - 2n_1\bar{u}_1\tran\bar{u} ,
    \\
    n_2\|\bar{u}_2-\bar u\|^2 &= n_2\|\bar{u}_2\|^2 + n_2\|\bar{u}\|^2 - 2n_2\bar{u}_2\tran\bar{u} .
\end{align*}
Note the following equivalence,
\[
2n_1\bar{u}_1\tran\bar{u} + 2n_2\bar{u}_2\tran\bar{u} = 2\sum_{i=1}^n{u_i\tran \bar{u}} = 2(n_1 + n_2)\|\bar{u}\|^2,
\]
which makes $r_0 - r_{\rm split}$ in Lemma \ref{lem: resred} the same as \eqref{eq: rred1}, 
\[
n_1(\|\bar{u}_1\|^2 - \|\bar{u}\|^2) + n_2(\|\bar{u}_2\|^2 - \|\bar{u}\|^2)
\]
This completes half of the proof.

For the equivalence part in Lemma \ref{lem: resred}, note the following,
\[
\bar{u} = \frac{n_1\bar{u}_1+n_2\bar{u}_2}{n}, 
\]
plugging which into \eqref{eq: rred1} leads to
\begin{align*}
    &r_0 - r_{\rm split} = n_1\|\bar{u}_1\|^2 + n_2\|\bar{u}_2\|^2 - \frac{\|n_1\bar{u}_1+n_2\bar{u}_2\|^2}{n}
    \\
    &= \frac{(nn_1 - {n_1}^2)}{n}\|\bar{u}_1\|^2 + \frac{(nn_2 - {n_2}^2)}{n}\|\bar{u}_2\|^2 
    - 2\frac{n_1n_2}{n}{\bar{u}_1}\tran\bar{u}_2
\end{align*}
Noting that $nn_1 - {n_1}^2 = n_1(n-n_1) = n_1n_2$ and similarly for $n_2$, we can easily show that
\begin{align*}
    r_0 - r_{\rm split} &= \frac{n_1n_2}{n}(\|\bar{u}_1\|^2 + \|\bar{u}_2\|^2 - 2{\bar{u}_1}\tran\bar{u}_2)
    \\
    &= \frac{n_1n_2}{n}\|\bar{u}_1-\bar{u}_2\|^2 ,
\end{align*}
which completes the proof. 
\end{proof}

We remark that $r_0 - r_{\rm split}$ is analogous to $\fslamstar{K^\star} - \fslam(K^\star + 1)$, where the second term is not necessarily optimal since SLAM optimization would distribute some residuals into the odometry term and other measurement terms so that the projected landmark positions could shift. Our splitting model is analogous to the pre-optimized value $\fslam(K^\star + 1)$.

\subsection{Geometric Half-Space Split} \label{sec: halfsplit}
We now consider a geometric split that divides the measurements into two halves, which is similar to the right schematic in Fig. \ref{fig: splitsamples}. We start by choosing a unit vector
\[
v\in\R^d,
\qquad
\|v\|=1.
\]
Define the scalar projection
\[
s_i
=
v\tran u_i.
\]
Without loss of generality, we center our measurements
\[
u_i\sim \gauss(0,I_d),
\]
which leads to
\[
s_i\sim \gauss(0,1).
\]
Split the samples according to the sign of this projection:
\[
G_+ = \{i:s_i \ge 0\},
\qquad
G_- = \{i:s_i<0\}.
\]
An important assumption is that the number of measurements is large, and therefore, $\bar{u} \approx \mu = 0$. Because of this assumption, additionally we also have $(I - vv\tran)\bar{u}_+ \approx (I - vv\tran)\bar{u}_- \approx (I - vv\tran)\bar{u} \approx 0$. This implies that the residual reduction $r_0 - r_{\rm split}$ is only contributed by the component along the $v$ direction, since the sum of the residuals in all orthogonal directions is unchanged from $r_0$ \eqref{eq: r0} to $r_{\rm split}$ \eqref{eq: rsplit}. Furthermore, due to this assumption, we have $\bar{s}_+ = \frac{1}{n_+}\sum_{i \in G_+}{s_i} \approx \E[s\mid s \ge 0]$ and $\bar{s}_- = \frac{1}{n_-}\sum_{i \in G_-}{s_i} \approx \E[s\mid s<0]$, where $s$ follows $\gauss(0,1)$ and $s = v\tran u$ for $u$ following $\gauss(0,I_d)$.

More formally, consider $u$, which follows $\gauss(0, I_d)$, decomposed into the component along $v$ and the component orthogonal to $v$:
\begin{equation*}
u=sv+z,
\qquad
z\perp v.
% \label{eq: zdef}
\end{equation*}
%Because the Gaussian distribution is isotropic, the orthogonal component has zero conditional mean after conditioning on the sign of $s$:
Note that $z = (I - vv\tran)u$ and $s = v\tran u$ are jointly Gaussian. Ffor this zero-mean isotropic Gaussian case, it can be shown that the covariance between $z$ and $s$ is zero,
\begin{align*}
\text{COV}(s, z) &= \E[sz] - \E[s]\E[z] = \E[v\tran u (I-vv\tran)u]
\\
&= \E[(I-vv\tran)uu\tran v] = (I-vv\tran)\E[uu\tran]v
\\
&= (I-vv\tran)(\text{COV}(u) + \E[u]\E[u]\tran)v
\\
&= v - vv\tran v = v - v \|v\|^2 = 0 .
\end{align*}
For jointly Gaussian variables, $z$ is then independent of $s$, which makes
\begin{equation*}
    \E[z\mid s \ge 0] = \E[z\mid s<0] = \E[z] = 0 .
\end{equation*}
Therefore,
\begin{equation*}
\E[u\mid s  \ge  0]
=
\E[sv+z\mid s \ge 0]
=
\E[s\mid s \ge 0]v ,
% \label{eq: expuexps}
\end{equation*}
and similarly,
\[
\E[u\mid s < 0] = \E[s\mid s < 0]v .
\]
By definition,
\[
\E[s\mid s \ge 0] =
\frac{
\int_0^\infty s\frac{1}{\sqrt{2\pi}}
\exp\left(-\frac{s^2}{2}\right)ds
}{
\mathbb P(s \ge 0)
}.
\]
Since
\[
\mathbb P(s \ge 0)=\frac12,
\]
we get
\[
\E[s\mid s \ge 0] =
2\int_0^\infty s\frac{1}{\sqrt{2\pi}}
\exp\left(-\frac{s^2}{2}\right)ds.
\]
Let
\[
q=\frac{s^2}{2},
\qquad
dq=s\,ds.
\]
Then
\[
\E[s\mid s \ge 0] =
\frac{2}{\sqrt{2\pi}}
\int_0^\infty e^{-q}dq.
\]
Since
\[
\int_0^\infty e^{-q}dq=1,
\]
we obtain
\[
\E[s\mid s \ge 0] = \sqrt{\frac{2}{\pi}}.
\]
By symmetry,
\[
\E[s\mid s<0] = -\sqrt{\frac{2}{\pi}}.
\]
Therefore,
\[
\E[u\mid s \ge 0] = \sqrt{\frac{2}{\pi}}v,
\]
and
\[
\E[u\mid s < 0] = -\sqrt{\frac{2}{\pi}}v.
\]
Due to the assumption of a large number of measurements,
\[
\bar{u}_+ \approx \E[u\mid s \ge 0] = \sqrt{\frac{2}{\pi}}v,
\]
and
\[
\bar{u}_- \approx \E[u\mid s < 0] = -\sqrt{\frac{2}{\pi}}v,
\]
Also due to the large-sample assumption and the symmetry of an isotropic Gaussian, approximately half of the samples fall into each group:
\[
n_+\approx \frac{n}{2},
\qquad
n_-\approx \frac{n}{2}.
\]
We substitute these results into Lemma \ref{lem: resred} and obtain
\begin{align*}
    r_0-r_{\rm split} &= \frac{n_+n_-}{n}\|\bar{u}_+-\bar{u}_-\|^2
    \\
    &= \frac{n}{4}\|\sqrt{\frac{2}{\pi}}v + \sqrt{\frac{2}{\pi}}v\|^2
    \\
    &= \frac{2n}{\pi}\|v\|^2
    \\
    &= \frac{2n}{\pi}.
\end{align*}

\subsection{Optimality of the Half-Space Split}
When the number of measurements $n$ is large ($n \to \infty$), the half-space split maximally reduces the residual (i.e., maximizes $r_0 - r_{\rm split}$). Intuitively, in \eqref{eq: rred}, $n_1n_2 = n_1(n - n_1)$ is maximized (simply by setting the derivative to zero) when $n_1 = n_2 = \frac{n}{2}$, which is achieved by the half-space split when $n$ is large. When moving away from the center split, the decrease in $n_1n_2$ dominates the growth in $\|\bar{u}_1-\bar{u}_2\|^2$ (one can reason about $n_1$ and $n_2$ through Gaussian CDF and about $\bar{u}_1$ and $\bar{u}_2$ using the previous derivations but with a non-zero threshold).

Under non-ideal conditions such as $n \not\to \infty$ and an anisotropic data covariance (Fig. \ref{fig: splitsamplesnidl} for illustration), $v$ may not be an arbitrary direction, but the principal direction of maximum data variance. Splitting at $s = 0$ may not be optimal since the best split may shift toward outliers to maximize $\|\bar{u}_1-\bar{u}_2\|^2$. The directions orthogonal to $v$ may also contribute to the residual reduction under non-ideal conditions.
\begin{figure}[h]
\centering
    \scalebox{1.0}{\tikzset{every picture/.style={line width=0.75pt}} %set default line width to 0.75pt        

\begin{tikzpicture}[x=0.75pt,y=0.75pt,yscale=-1,xscale=1]
%uncomment if require: \path (0,110); %set diagram left start at 0, and has height of 110

\draw  [color={rgb, 255:red, 245; green, 166; blue, 35 }  ,draw opacity=1 ] (140,33.73) -- (150.61,44.33)(150.61,33.73) -- (140,44.33) ;
\draw  [color={rgb, 255:red, 245; green, 166; blue, 35 }  ,draw opacity=1 ] (211.94,54.79) -- (222.55,65.39)(222.55,54.79) -- (211.94,65.39) ;
\draw  [color={rgb, 255:red, 245; green, 166; blue, 35 }  ,draw opacity=1 ] (173.55,65.73) -- (184.15,76.33)(184.15,65.73) -- (173.55,76.33) ;
\draw  [color={rgb, 255:red, 245; green, 166; blue, 35 }  ,draw opacity=1 ] (141.88,46.06) -- (152.49,56.67)(152.49,46.06) -- (141.88,56.67) ;
\draw  [color={rgb, 255:red, 245; green, 166; blue, 35 }  ,draw opacity=1 ] (155.27,20) -- (165.88,30.61)(165.88,20) -- (155.27,30.61) ;
\draw  [color={rgb, 255:red, 245; green, 166; blue, 35 }  ,draw opacity=1 ] (150.88,55.06) -- (161.49,65.67)(161.49,55.06) -- (150.88,65.67) ;
\draw  [color={rgb, 255:red, 245; green, 166; blue, 35 }  ,draw opacity=1 ] (184.88,42.39) -- (195.49,53)(195.49,42.39) -- (184.88,53) ;
\draw  [color={rgb, 255:red, 245; green, 166; blue, 35 }  ,draw opacity=1 ][fill={rgb, 255:red, 74; green, 144; blue, 226 }  ,fill opacity=1 ] (175.27,21) -- (185.88,31.61)(185.88,21) -- (175.27,31.61) ;
\draw  [color={rgb, 255:red, 245; green, 166; blue, 35 }  ,draw opacity=1 ][fill={rgb, 255:red, 74; green, 144; blue, 226 }  ,fill opacity=1 ] (162.21,61.73) -- (172.82,72.33)(172.82,61.73) -- (162.21,72.33) ;
\draw  [color={rgb, 255:red, 245; green, 166; blue, 35 }  ,draw opacity=1 ][fill={rgb, 255:red, 74; green, 144; blue, 226 }  ,fill opacity=1 ] (185.88,58.39) -- (196.49,69)(196.49,58.39) -- (185.88,69) ;
%Shape: Triangle [id:dp81421549920045] 
\draw  [draw opacity=0][fill={rgb, 255:red, 245; green, 166; blue, 35 }  ,fill opacity=1 ] (180.25,42.81) -- (187.68,54.81) -- (172.82,54.81) -- cycle ;
\draw  [color={rgb, 255:red, 245; green, 166; blue, 35 }  ,draw opacity=1 ] (201.94,39) -- (212.55,49.61)(212.55,39) -- (201.94,49.61) ;
\draw  [color={rgb, 255:red, 245; green, 166; blue, 35 }  ,draw opacity=1 ][fill={rgb, 255:red, 74; green, 144; blue, 226 }  ,fill opacity=1 ] (155.55,34.73) -- (166.15,45.33)(166.15,34.73) -- (155.55,45.33) ;
\draw  [color={rgb, 255:red, 245; green, 166; blue, 35 }  ,draw opacity=1 ] (250,33.73) -- (260.61,44.33)(260.61,33.73) -- (250,44.33) ;
\draw  [color={rgb, 255:red, 74; green, 144; blue, 226 }  ,draw opacity=1 ] (321.94,54.79) -- (332.55,65.39)(332.55,54.79) -- (321.94,65.39) ;
\draw  [color={rgb, 255:red, 74; green, 144; blue, 226 }  ,draw opacity=1 ] (283.55,65.73) -- (294.15,76.33)(294.15,65.73) -- (283.55,76.33) ;
\draw  [color={rgb, 255:red, 245; green, 166; blue, 35 }  ,draw opacity=1 ] (251.88,46.06) -- (262.49,56.67)(262.49,46.06) -- (251.88,56.67) ;
\draw  [color={rgb, 255:red, 245; green, 166; blue, 35 }  ,draw opacity=1 ] (265.27,20) -- (275.88,30.61)(275.88,20) -- (265.27,30.61) ;
\draw  [color={rgb, 255:red, 245; green, 166; blue, 35 }  ,draw opacity=1 ] (260.88,55.06) -- (271.49,65.67)(271.49,55.06) -- (260.88,65.67) ;
\draw  [color={rgb, 255:red, 74; green, 144; blue, 226 }  ,draw opacity=1 ] (294.88,42.39) -- (305.49,53)(305.49,42.39) -- (294.88,53) ;
\draw  [color={rgb, 255:red, 245; green, 166; blue, 35 }  ,draw opacity=1 ][fill={rgb, 255:red, 74; green, 144; blue, 226 }  ,fill opacity=1 ] (285.27,21) -- (295.88,31.61)(295.88,21) -- (285.27,31.61) ;
\draw  [color={rgb, 255:red, 245; green, 166; blue, 35 }  ,draw opacity=1 ][fill={rgb, 255:red, 74; green, 144; blue, 226 }  ,fill opacity=1 ] (272.21,61.73) -- (282.82,72.33)(282.82,61.73) -- (272.21,72.33) ;
\draw  [color={rgb, 255:red, 74; green, 144; blue, 226 }  ,draw opacity=1 ][fill={rgb, 255:red, 74; green, 144; blue, 226 }  ,fill opacity=1 ] (295.88,58.39) -- (306.49,69)(306.49,58.39) -- (295.88,69) ;
\draw  [color={rgb, 255:red, 74; green, 144; blue, 226 }  ,draw opacity=1 ] (311.94,39) -- (322.55,49.61)(322.55,39) -- (311.94,49.61) ;
\draw  [color={rgb, 255:red, 245; green, 166; blue, 35 }  ,draw opacity=1 ][fill={rgb, 255:red, 74; green, 144; blue, 226 }  ,fill opacity=1 ] (265.55,34.73) -- (276.15,45.33)(276.15,34.73) -- (265.55,45.33) ;
%Shape: Triangle [id:dp14117196806462462] 
\draw  [draw opacity=0][fill={rgb, 255:red, 74; green, 144; blue, 226 }  ,fill opacity=1 ] (307.57,57) -- (315,69) -- (300.14,69) -- cycle ;
%Shape: Triangle [id:dp5649318901600253] 
\draw  [draw opacity=0][fill={rgb, 255:red, 245; green, 166; blue, 35 }  ,fill opacity=1 ] (273.25,37.81) -- (280.68,49.81) -- (265.82,49.81) -- cycle ;

\end{tikzpicture}}
\caption{Splitting Gaussian samples under non-ideal conditions.}
\label{fig: splitsamplesnidl}
\end{figure}
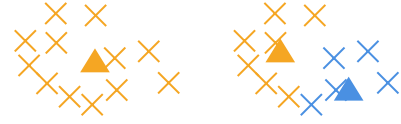

We now provide a more thorough derivation of the optimality of the half-space split under ideal conditions. Our objective is to maximize \eqref{eq: rred}, which is valid for all splits. For convenience, we repeat it below:
\[
    r_0-r_{\rm split} = \frac{n_1n_2}{n}\|\bar{u}_1-\bar{u}_2\|^2 .
\]
There are two parts: (i) showing that splitting at $s = 0$ is optimal among all the half-space splits, and (ii) showing that the half-space splits are optimal among all splits. For the first part, we consider a general half-space split at $s = t$ instead of at $s = 0$:
\[
G_1 = \{i:s_i < t\},
\qquad
G_2 = \{i:s_i \ge t\}.
\]
The objective is to show that $t = 0$ maximizes $r_0-r_{\rm split}$. Before proceeding with the derivation, we introduce several definitions and state a few properties that will be useful later in the derivation. Let $\phi(s)$ denote the standard Gaussian probability density function, and let $\Phi(s)$ denote the cumulative distribution function of a standard Gaussian. It is easy to show that
\begin{equation}
    \phi'(s) = -s\phi(s) ,
    \label{eq: gausspdfderiv}
\end{equation}
\begin{equation}
    \Phi'(s) = \phi(s) .
    \label{eq: gausscdfderiv}
\end{equation}
Our objective function $\max_t{r_0-r_{\rm split}}$ can be written as
\begin{align*}
    \max_t{\frac{n_1n_2}{n}\|\bar{u}_1-\bar{u}_2\|^2} = n\left(\max_t{\frac{n_1n_2}{n^2}\|\bar{u}_1-\bar{u}_2\|^2}\right) .
\end{align*}
Assuming a large number of measurements,
\begin{align*}
    \frac{n_1}{n} &\approx \int^{t}_{-\infty}{\phi(s)\,ds} = \Phi(t) ,
    \\
    \frac{n_2}{n} &\approx 1 - \Phi(t) ,
\end{align*}
\begin{align*}
    \bar{u}_1 &\approx \E[u\mid s < t]
    = \frac{\int^{t}_{-\infty}{s\phi(s)\,ds}}{\int^{t}_{-\infty}{\phi(s)\,ds}}v
    \overset{\text{\eqref{eq: gausspdfderiv}\eqref{eq: gausscdfderiv}}}{=} \frac{-\phi(t)}{\Phi(t)}v ,
    \\
    \bar{u}_2 &\approx \E[u \mid s \ge t] = \frac{\int^{\infty}_{t}{s\phi(s)\,ds}}{\int^{\infty}_{t}{\phi(s)\,ds}}v \overset{\text{\eqref{eq: gausspdfderiv}\eqref{eq: gausscdfderiv}}}{=} \frac{\phi(t)}{1 - \Phi(t)}v .
\end{align*}
We substitute the above equalities into the objective function:
\begin{equation}
\begin{aligned}
    &\max_t{\frac{n_1n_2}{n^2}\|\bar{u}_1-\bar{u}_2\|^2}
    \\
    &\approx \max_t{\Phi(t)\left[1 - \Phi(t)\right]\phi^2(t)\left[-\frac{1}{\Phi(t)} - \frac{1}{1 - \Phi(t)}\right]^2\|v\|^2}
    \\
    &= \max_t{\frac{\phi^2(t)}{\Phi(t)\left[1 - \Phi(t)\right]}} = \max_t{h(t)} .
\end{aligned}
\end{equation}
Using the following properties:
\begin{equation}
    \phi(-t) = \phi(t) ,
    \label{eq: gausspdfsymm}
\end{equation}
\begin{equation}
    \Phi(-t) = 1 - \Phi(t) ,
    \label{eq: gausscdfsymm}
\end{equation}
$h(t)$ is symmetric about $t = 0$. To show that $\arg\max_t{h(t)} = 0$, it is sufficient to show (i) $\left.\frac{\mathrm{d}\log{h(t)}}{\mathrm{d}t} \right|_{t=0} = 0$, and (ii) $\left.\frac{\mathrm{d}\log{h(t)}}{\mathrm{d}t} \right|_{t > 0} < 0$. Taking the derivative:
\[
\log{h(t)} \overset{\text{\eqref{eq: gausscdfsymm}}}{=} 2\log{\phi(t)} - \log{\Phi(t)} - \log{\Phi(-t)} ,
\]
\begin{equation}
\begin{aligned}
    \frac{\mathrm{d}\log{h(t)}}{\mathrm{d}t} &\overset{\text{\eqref{eq: gausspdfderiv}\eqref{eq: gausscdfderiv}}}{=} -2\frac{t\phi(t)}{\phi(t)} - \frac{\phi(t)}{\Phi(t)} + \frac{\phi(t)}{\Phi(-t)}
    \\
    &= -2t + \phi(t)\left[\frac{1}{\Phi(-t)} - \frac{1}{\Phi(t)}\right] .
    \label{eq: dlogh}
\end{aligned}
\end{equation}
From \eqref{eq: dlogh}, it is easy to check that $\left.\frac{\mathrm{d}\log{h(t)}}{\mathrm{d}t} \right|_{t=0} = 0$. Next, we show $\left.\frac{\mathrm{d}\log{h(t)}}{\mathrm{d}t} \right|_{t > 0} < 0$, i.e.,
\begin{equation}
    \phi(t)\left[\frac{1}{\Phi(-t)}-\frac{1}{\Phi(t)}\right]<2t \qquad \text{for}\ t>0.
    \label{eq: dloghless}
\end{equation}
We define the inverse Mills ratio,
\[
\lambda(t)=\frac{\phi(t)}{\Phi(-t)}, \qquad
\lambda(-t) \overset{\text{\eqref{eq: gausspdfsymm}}}{=} \frac{\phi(t)}{\Phi(t)}.
\]
Therefore,
\begin{equation}
\begin{aligned}
    \phi(t)\left[\frac{1}{\Phi(-t)}-\frac{1}{\Phi(t)}\right] &= \lambda(t)-\lambda(-t)
    \\
    &= \int^{t}_{-t}{\lambda'(x)\,dx} \qquad \text{for}\ t > 0.
    \label{eq: intdmills}
\end{aligned}
\end{equation}
From \eqref{eq: intdmills}, to show that $\left.\frac{\mathrm{d}\log{h(t)}}{\mathrm{d}t} \right|_{t > 0} < 0$, or equivalently \eqref{eq: dloghless}, we require $\lambda'(t) < 1$ and $t > 0$ (so that the integral from $-t$ to $t$ is valid). We therefore take the derivative of $\lambda(t)$:
\begin{equation}
\begin{aligned}
    \lambda'(t) = &\frac{\phi'(t)\Phi(-t)+\phi(t)\Phi'(-t)}{\Phi(-t)^2}
    \\
    \overset{\text{\eqref{eq: gausspdfderiv}\eqref{eq: gausscdfderiv}}}{=} &\frac{-t\phi(t)\Phi(-t)+\phi(t)\phi(t)}{\Phi(-t)^2}
    \\
    = &\frac{\phi(t)}{\Phi(-t)}\left[\frac{\phi(t)}{\Phi(-t)}-t\right]
    \\
    = &\lambda(t)\left[\lambda(t) - t\right] .
    \label{eq: lambdaderiv}
\end{aligned}
\end{equation}
Note the following equalities:
\[
\E[s \mid s \ge t] = \frac{\int^{\infty}_{t}{s\phi(s)\,ds}}{\int^{\infty}_{t}{\phi(s)\,ds}} \overset{\text{\eqref{eq: gausspdfderiv}\eqref{eq: gausscdfderiv}\eqref{eq: gausscdfsymm}}}{=} \frac{\phi(t)}{\Phi(-t)} = \lambda(t) .
\]
\begin{align*}
\E[s^2 \mid s \ge t] &= \frac{\int^{\infty}_{t}{s^2\phi(s)\,ds}}{\int^{\infty}_{t}{\phi(s)\,ds}} = \frac{\int^{\infty}_{t}{s^2\phi(s)\,ds}}{\Phi(-t)}
\\
&= \frac{\left. -s\phi(s) \right|^{\infty}_t + \int^{\infty}_t{\phi(s)\,ds}}{\Phi(-t)}
\\
&= \frac{t\phi(t) + \Phi(-t)}{\Phi(-t)}
\\
&= t\lambda(t) + 1 ,
\end{align*}
where we have used integration by parts. Further,
\begin{align*}
    \text{Var}(s \mid s \ge t) &= \E[s^2 \mid s \ge t] - \E[s \mid s \ge t]^2
    \\
    &= t\lambda(t) + 1 - \lambda(t)^2 > 0 ,
\end{align*}
since $\text{Var}(s \mid s \ge t) > 0$. The above inequality leads to
\[
\lambda(t)^2 - t\lambda(t) \overset{\text{\eqref{eq: lambdaderiv}}}{=} \lambda'(t) < 1 .
\]
Therefore, using \eqref{eq: intdmills},
\begin{align*}
    &\phi(t)\left[\frac{1}{\Phi(-t)}-\frac{1}{\Phi(t)}\right] = \lambda(t)-\lambda(-t)
    \\
    &= \int^{t}_{-t}{\lambda'(x)\,dx} < \int^{t}_{-t}{1} = 2t \qquad \text{for}\ t>0.
    \label{eq: intdmillsineq}
\end{align*}
Consequently, \eqref{eq: dloghless} holds and $\left.\frac{\mathrm{d}\log{h(t)}}{\mathrm{d}t} \right|_{t > 0} < 0$ holds. We conclude that $\arg\max_t{h(t)} = 0$.

Next, we show that the half-space splits are optimal among all splits. We continue to assume a large number of measurements and generalize the previous derivation to arbitrary splits:
\begin{align*}
    \frac{n_1}{n} &\approx \int_{G_1}{\phi(u)\,\dd^{d}u} = \Prob(G_1) = p ,
    \\
    \frac{n_2}{n} &\approx 1 - p ,
\end{align*}
\begin{align*}
    \bar{u}_1 &\approx \E[u\mid u \in G_1]
    = \frac{\int_{G_1}{u\phi(u)\,\dd^du}}{\int_{G_1}{\phi(u)\,\dd^du}}
    = \frac{\int_{G_1}{u\phi(u)\,\dd^du}}{p} ,
    \\
    \bar{u}_2 &\approx \E[u \mid u \in G_2] = \frac{\int_{G_2}{u\phi(u)\,\dd^du}}{\int_{G_2}{\phi(u)\,\dd^du}}
    = \frac{\int_{G_2}{u\phi(u)\,\dd^du}}{1 - p} .
\end{align*}
By the law of total expectation,
\begin{align*}
    &p\E[u\mid u \in G_1] + (1 - p)\E[u \mid u \in G_2]
    \\
    &= \int_{\realspace^d}{u\phi(u)\,\dd^du} = \E[u] = 0 .
\end{align*}
Using the equalities above, the objective function can be written as
\begin{equation*}
\begin{aligned}
    &\max_{G_1}{\frac{n_1n_2}{n^2}\|\bar{u}_1-\bar{u}_2\|^2}
    \\
    \approx &\max_{G_1}{p(1 - p)\left\|\E[u\mid u \in G_1] - \frac{p}{p - 1}\E[u\mid u \in G_1]\right\|^2}
    \\
    = &\max_{G_1}{\frac{p}{1 - p}\left\|\E[u\mid u \in G_1]\right\|^2}
    % \\
    % \overset{\text{\eqref{eq: expuexps}}}{=} &\max_{G_1}{\frac{p}{1 - p}\E[s\mid s \in G_1]^2}
    \\
    = &\max_p{\max_{G_1: \Prob(G_1) = p}{\frac{p}{1 - p}\left\|\E[u \mid u \in G_1]\right\|^2}} ,
\end{aligned}
\end{equation*}
where the last equality decomposes the maximization into two steps and makes the following observation clearer: for a fixed $p$, we need to maximize $\left\|\E[u \mid u \in G_1]\right\|$ in order to maximize the objective function. Let $v$ be:
\begin{equation}
v = \frac{\E[u \mid u \in G_1]}{\left\|\E[u \mid u \in G_1]\right\|} .
% v depends on G_1 but I think it fine (1) This is finding a convenient representation to write the norm. An alternative representation should not change the result. (2) We can start with some G_1 and find v for this G_1. Now for this v, we can maximize by moving the probability mass such that s >= t. This process makes sense. (3) The underlying trickiness does not come from defining v this way but the optimal solution is not unique - any half-space split (along any unit vector direction) can be the optimal solution. Defining v this way makes the derivation easier, though we are restricted to a particular direction.
\label{eq: vdef}
\end{equation}
Substituting in $v$ yields:
\begin{equation*}
\begin{aligned}
    &\max_{G_1: \Prob(G_1) = p}{p\left\|\E[u \mid u \in G_1]\right\|}
    \\
    &= \max_{G_1: \Prob(G_1) = p}{p\left(v\tran \E[u \mid u \in G_1]\right)}
    \\
    &= \max_{G_1: \Prob(G_1) = p}{\int_{G_1}{v\tran u\phi(u)\,\dd^du}}
    \\
    &\overset{s = v\tran u}{=} \max_{G_1: \Prob(G_1) = p}{\int_{G_1}{s\phi(u)\,\dd^du}} % bathtub principle (layer-cake)
    \\
    % &\overset{}{=} \max_{G_1: \Prob(G_1) = p}{\int_{z: (s,z) \in G_1}{\left[\int_{s:(s,z) \in G_1}s\phi(s)\,\dd s\right]}{\phi(z)\,\dd^{d-1}z}}
    &\overset{}{=} \max_{G_1: \Prob(G_1) = p}{\int_{s:(s,z) \in G_1}\left[\int_{z: (s,z) \in G_1}{\phi(z)\,\dd^{d-1}z}\right]s\phi(s)\,\dd s}
    \\
    % &= \int_{\realspace^{d-1}}{\left[\max_{G_1: \Prob(G_1) = p}{\int_{s:(s,z) \in G_1}s\phi(s)\,\dd s}\right]\phi(z)\,\dd^{d-1}z}
    &\overset{}{=} \max_{G_1: \Prob(G_1) = p}{\int_{s:(s,z) \in G_1}s\left[\int_{\realspace^{d - 1}}{\phi(z)\,\dd^{d-1}z}\right]\phi(s)\,\dd s} ,
\end{aligned}
\end{equation*}
where the penultimate equality decomposes the integral into the coordinate along $v$ and the remaining $d - 1$ orthogonal coordinates, using the isotropic Gaussian assumption for $u$ and the resulting independence between coordinates. The last equality considers the following maximization process: starting from the largest possible value of $s$, we include all points in the orthogonal coordinates to assign as much probability mass as possible to each value of $s$ before including points with smaller values of $s$. The process continues until the total probability mass equals $p$.
% In the last equality, we essentially split into two steps, (i) maximizing the inner integral in the $v$ coordinate, and (ii) placing all possible probability mass from all the other coordinates on the maximized inner integral. For a scalar standard Gaussian random variable $s$, the inner integral is maximized by taking the largest values of $s$.
Therefore, 
\begin{align}
    &\underset{G_1: \Prob(G_1) = p}{\arg\max}{\left\|\E[u \mid u \in G_1]\right\|} \nonumber
    \\
    = &\underset{G_1: \Prob(G_1) = p}{\arg\max}{\int_{\{s \mid \forall z \in \realspace^{d - 1},\,(s,z) \in G_1\}}s\phi(s)\,\dd s} \label{eq: scalaroptimg}
    \\
    = &\{u : v\tran u = s \ge t\} , \label{eq: halfspacesplit}
\end{align}
where $t$ is selected such that $\Prob(G_1) = p$. Note that \eqref{eq: scalaroptimg} is essentially a one-dimensional problem, so it is obvious that we need to include the largest possible values of $s$ to maximize the objective, leading to $s \ge t$ as the solution. It is also worthwhile to mention that $v$ is defined in \eqref{eq: vdef} as a way to rewrite $\left\|\E[u \mid u \in G_1]\right\|$ for the derivation. Because of the isotropic Gaussian assumption, $v$ may point in any direction. In other words, the solution is not unique. Any half-space split satisfying \eqref{eq: halfspacesplit} and $\Prob(G_1) = p$ is optimal.

\end{document}